\documentclass[a4paper,fleqn]{cas-sc}
\usepackage[utf8]{inputenc} % allow utf-8 input
\usepackage[T1]{fontenc}    % use 8-bit T1 fonts

\usepackage{url}            % simple URL typesetting
\usepackage{booktabs}       % professional-quality tables
\usepackage{amsfonts}       % blackboard math symbols
\usepackage{amsmath}        % amsmath
\usepackage{amsthm}
\usepackage{caption}        % caption
\usepackage{subcaption}     % subcaption
\usepackage{nicefrac}       % compact symbols for 1/2, etc.
\usepackage{multirow}       % multirow in tablet
\usepackage{microtype}      % microtypography
\usepackage{xcolor}         % colors
\usepackage{tikz-cd}        % commutation graphs
\usepackage{threeparttable}

\newtheorem{theorem}{Theorem} %[section]

\definecolor{deblue}{RGB}{11,132,147}
\definecolor{ocra}{RGB}{204, 119, 34}
\usepackage{enumitem}
\usepackage{tikz}
\newcommand{\fcircle}[2][red,fill=red]{\tikz[baseline=-0.5ex]\draw[#1,radius=#2] (0,0.03) circle ;}

\definecolor{neuripsblue}{rgb}{0.21,0.49,0.74}
\newtheorem{definition}{Definition}

\newtheorem{proposition}{Proposition}[theorem]
\newtheorem{corollary}{Corollary}[theorem]

\newtheorem{remark}{Remark}
\usepackage[table,xcdraw]{xcolor}
\setlist[enumerate,1]{leftmargin=2.0em}
\setlist[enumerate,2]{leftmargin=2.5em}
\setlist[enumerate,3]{leftmargin=3.0em}
\setlist[enumerate,4]{leftmargin=3.5em}
\setlist[enumerate,5]{leftmargin=4.0em}
\setlist[enumerate,6]{leftmargin=4.5em}
\makeatletter

\newcounter{enumv}
\newcounter{enumvi}
\makeatother
\usepackage[authoryear,longnamesfirst]{natbib}
\usepackage{float}
\usepackage[ruled,vlined,linesnumbered]{algorithm2e}


\makeatletter
\ExplSyntaxOn
\cs_set:Npn \__first_footerline:
{
	\group_begin:
	\small
	\sffamily
	\ifnum\theblind>0\relax
	\else 
	\csuse{__short_authors:} ~
	\fi
	\group_end:
}
\ExplSyntaxOff
\makeatother

\usepackage[normalem]{ulem}

\def\tsc#1{\csdef{#1}{\textsc{\lowercase{#1}}\xspace}}
\tsc{WGM}
\tsc{QE}
\tsc{EP}
\tsc{PMS}
\tsc{BEC}
\tsc{DE}
\begin{document}
\let\WriteBookmarks\relax
\def\floatpagepagefraction{1}
\def\textpagefraction{.001}
\shorttitle{ABLE}
\shortauthors{X. Zhao \& A.I. Aviles-rivero} 

\title [mode = title]{Don't Fix the Basis -- Learn It: Spectral Representation with Adaptive Basis Learning for PDEs}
\tnotemark[1]

\tnotetext[1]{This document is the results of the research project funded by Tsinghua University Dushi Program and Tsinghua University Initiative Scientific Research Program }

\author[1]{Xuxiang Zhao}
\ead{zhaoxx24@mails.tsinghua.edu.cn}

\credit{Conceptualization of this study, Methodology, Software, Visualization, Formal analysis, Writing - original
draft preparation}
\affiliation[1]{organization={Qiuzhen College, Tsinghua University},
                postcode={100084}, 
                postcodesep={},
                city={Beijing},
                country={China}}
\author[2]{Angelica I. Aviles-rivero}
\cormark[1]
\ead{aviles-rivero@tsinghua.edu.cn}
\credit{Investigation, Data Curation, Visualization, Writing - review and editing, Supervision, Funding Acquisition}

\affiliation[2]{organization={Yau Mathematical Sciences Center, Tsinghua University}, 
                postcode={100084}, 
                postcodesep={}, 
                city={Beijing},
                country={China}}

\cortext[cor1]{Corresponding author}

\begin{abstract}
Spectral neural operators achieve strong performance for PDE learning, but rely on fixed global bases that limit their ability to represent spatially heterogeneous and multiscale dynamics. We propose Adaptive Basis Learning (ABLE), a framework that learns data-dependent spectral representations instead of relying on predefined bases. ABLE constructs a spatially adaptive Parseval frame via a learned ancillary density, enabling the operator to act in a lifted spectral space while preserving invertibility and maintaining $O(N\log N)$ complexity through FFT-based implementation. This shifts the source of expressivity from spectral coefficients to the representation itself, allowing the model to capture localized structures and non-translation-invariant interactions more efficiently.
ABLE integrates seamlessly into existing neural operator architectures as a drop-in replacement for spectral layers. Across a range of benchmarks ABLE  improves accuracy over strong baselines, with the largest gains in regimes characterized by sharp gradients and multiscale behavior. Moreover, augmenting existing models (e.g., U-FNO, HPM, SAOT) with ABLE further enhances their performance, demonstrating its role as a general and complementary spectral refinement.
Our results highlight that the data-driven choice of representation, rather than operator complexity alone, is a key bottleneck in neural operator design. By learning the basis itself, ABLE provides a principled and efficient framework for improving spectral methods in PDE learning.
\end{abstract}


\maketitle

\section{Introduction}
Partial differential equations (PDEs)~\citep{evans2022partial,leveque2002finite,quarteroni1994numerical} underpin a vast range of physical and engineering systems, from fluid dynamics to porous media flow. Yet solving nonlinear PDEs in practice remains computationally demanding. High-fidelity numerical solvers are often required, and their cost quickly becomes prohibitive when solutions must be recomputed across different parameter settings. This challenge is further compounded by the sensitivity of many PDE systems to initial and boundary conditions: even small perturbations can lead to qualitatively different behaviors, necessitating recomputation for each new configuration.

Neural operators~\citep{DBLP:journals/corr/abs-2010-08895,bryutkin2024hamlet,lu2019deeponet,wu2024transolver,cheng2025mamba} offer a different paradigm by learning mappings between function spaces directly from the data, enabling the approximation of solution operators rather than individual solutions. 
This approach has been successfully applied across a range of domains, including fluid dynamics, porous media flow, and climate modeling ~\citep{pathak2022fourcastnet, wang2025fourier}. 
These developments have led to a family of models, among which spectral neural operators have emerged as particularly effective due to their ability to exploit structure in the frequency domain. In particular, the Fourier Neural Operator~\citep{DBLP:journals/corr/abs-2010-08895} (FNO) and its variants~\citep{QI2024106239,Guibas:2021,Wen2021UFNOA,Li2022FourierNO,Fanaskov2022SpectralNO,Wang2025LaplacianEN} approximate solution operators in the spectral domain, enabling efficient inference once trained. When inputs are drawn from a common distribution, neural operators amortize the cost of repeated PDE solves and provide a scalable alternative to classical simulation pipelines.
A key ingredient behind the success of spectral neural operators is the use of a fixed basis. The functions are projected onto predetermined Fourier modes, and the learning is restricted to modifying spectral coefficients within this representation. While this structure enables efficiency and stability, it also imposes a fundamental constraint: the representation is fixed independently of the data. As a result, operator expressivity is limited not only by the model in a certain given basis but also by the choice of basis itself. This limitation becomes particularly severe in regimes characterized by strong spatial heterogeneity, sharp transitions, or multiscale interactions, revealing a mismatch between fixed representations and complex PDE dynamics. Rather than adapting the operator within a fixed representation, we ask: can the representation itself be learned?

We propose \textbf{Adaptive Basis Learning (ABLE)}, a framework in which the representation basis is no longer fixed but learned from the input. This yields a spatially adaptive, overcomplete representation that couples the spectral structure with local variability. Instead of synthesizing localized features through increasingly complex spectral coefficients, ABLE encodes them directly in the learned basis.
From a theoretical perspective, the learned basis forms a Parseval frame, and the associated transformation is invertible and isometric, ensuring that the information is preserved. Moreover, classical Fourier-based operators are recovered as a special case, establishing ABLE as a strict generalization of existing spectral neural operators. By allowing the representation to adapt to the data, the model effectively redistributes complexity from the operator to the basis, enabling a richer class of operators while retaining the computational advantages of spectral methods.
This shift in representation has several consequences: it enables modeling of spatially heterogeneous, non-translation-invariant operators, improves the efficiency of representing localized structures, and introduces an implicit form of nonlinearity via the interaction between the adaptive basis and the operator.  Our contributions are as follows: 

\fcircle[fill=ocra]{2pt} We introduce Adaptive Basis Learning (ABLE), a framework in which the representation basis is learned from data rather than fixed \emph{a priori}, shifting the focus of learning from spectral coefficients to the underlying representation of the function space.

\fcircle[fill=ocra]{2pt} We formulate the adaptive basis as a Parseval tight frame and define the associated transformation, which is invertible and isometric. This ensures information preservation and provides a principled extension of classical Fourier analysis to data-dependent bases.

\fcircle[fill=ocra]{2pt} We propose the Adaptive Basis Learning Neural Operator (ABLE-NO) and show that it induces a strictly larger class of operators than Fourier Neural Operators, characterized by spatially adaptive and not necessarily translation-invariant kernels with inherently endowed nonlinearity.

\fcircle[fill=ocra]{2pt}  We demonstrate the effectiveness of ABLE on several challenging PDE benchmarks, including Burgers', Darcy, and Navier--Stokes equations, with consistent improvements in regimes exhibiting strong multiscale and localized phenomena.

\section{Related Work}
\textbf{Fixed‑basis Spectral Neural Operators.} The Fourier Neural Operator (FNO) \citep{DBLP:journals/corr/abs-2010-08895} represents the current state of practice for learning PDE solution operators. By assuming translation invariance, FNO diagonalises convolution kernels in the Fourier domain, achieving \(O(N\log N)\) complexity via the Fast Fourier Transform. The Fourier basis is orthonormal, self‑dual, and satisfies Parseval’s identity, which guarantees energy preservation. However, the basis is global and fixed a priori, limiting the representation of localized or multiscale features. Extensions to other orthonormal bases have been proposed~\citep{QI2024106239,Guibas:2021,Wen2021UFNOA,Li2022FourierNO,Fanaskov2022SpectralNO,Wang2025LaplacianEN}: the Spectral Neural Operator (SNO) \citep{Fanaskov2022SpectralNO} uses Chebyshev polynomials, the Laplacian Neural Operator (LE‑NO) \citep{Wang2025LaplacianEN} employs Laplacian eigenfunctions, and the Wavelet Neural Operator (WNO) \citep{Tripura2022WaveletNO} uses wavelets to capture both spatial and spectral information. \textit{These methods share the same structural limitation: the basis is predetermined and data‑agnostic.} Consequently, localized phenomena must be approximated through high‑frequency modes, which are often truncated for stability, discarding fine‑scale information.

\textbf{Spectral Filtering and Learnable Spectral Mixing Methods.}
Several works attempt to improve FNO \textit{without changing the underlying basis.} AFNO\citep{Guibas:2021} and GFNO \citep{QI2024106239} introduce learnable spectral masks to adapt the diagonal multiplier \(R(k)\), but they do not address the global nature of the basis. UFNO \citep{Wen2021UFNOA} and FreqMoE \citep{Chen2025FreqMoEDF} combine multiple spectral branches or mix spatial and spectral pathways. GeoFNO \citep{Li2022FourierNO} adds encoding/decoding maps to handle irregular geometries. Although these modifications yield incremental gains, they leave the fixed Fourier framework intact and do not solve the fundamental lack of spatial adaptivity. Attention mechanisms \citep{Vaswani2017AttentionIA} have been applied to operator learning in models such as Galerkin Attention \citep{Cao2021ChooseAT}, Latent Spectral Models (LSM) \citep{Wu2023SolvingHP} and Spectral Attention Operator Transformer (SAOT)\citep{Zhou2025SAOTAE}. These methods typically incur \(O(N^2)\) cost. A notable recent approach is the Holistic Physics Mixer (HPM) \citep{Yue2024HolisticPS}, which uses a Laplacian eigenbasis together with a learnable softmax selection per spatial point. HPM achieves competitive accuracy but at the cost of \(18\times\) longer training and inference due to the lack of FFT compatibility; moreover, despite their learnable spectral selection mechanism, its basis remains pre‑computed and fixed.

\paragraph{Positioning of this work.}
In contrast to all prior approaches — whether fixed‑basis spectral operators, adaptive filtering, hybrid architectures, or attention‑based mixers — our method does not operate within a predetermined coordinate system. Instead, we learn the basis itself as a data‑dependent object. This shifts the fundamental locus of learning from the coefficients of a fixed expansion to the representation of the function space itself (see Figure~\ref{fig:FIG_TEASER}). ABLE introduces a learnable spectral representation via a Parseval tight frame, while preserving invertibility, isometry, and FFT efficiency. This yields a strictly more expressive class of operators, including spatially heterogeneous, non-linear and non-translation-invariant kernels, which cannot be captured by fixed-basis spectral methods.
\begin{figure*} %t!
  \centering  \includegraphics[width=1\textwidth]{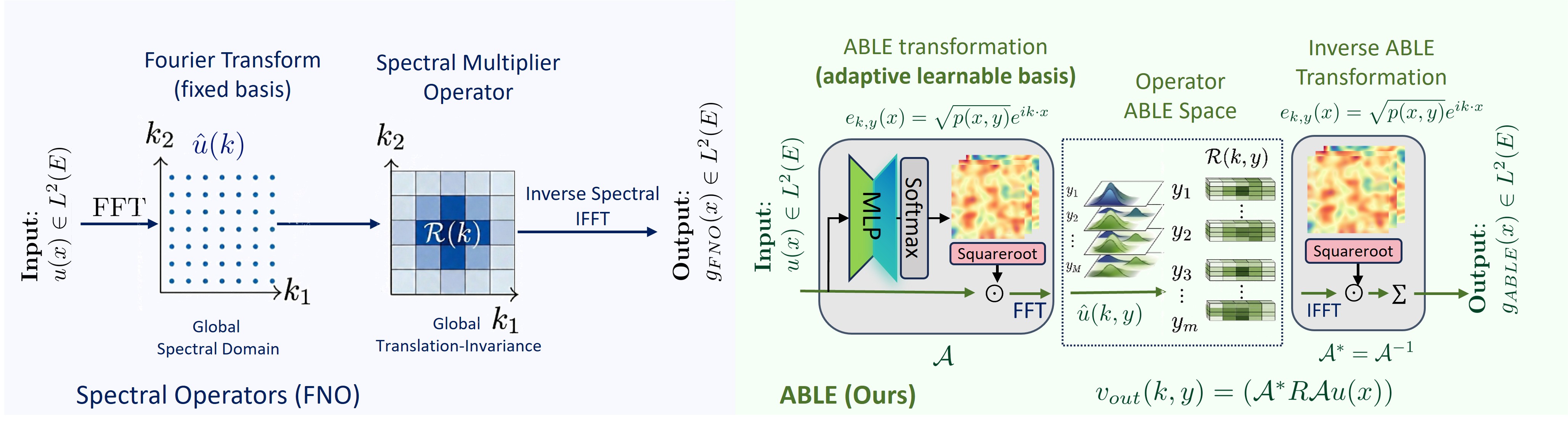}%{images/teaser.pdf} 
  \caption{
  \textbf{From global to adaptive spectral representations.} 
FNO operates on fixed Fourier modes $e^{ik\cdot x}$ in $k$-space, 
while ABLE lifts the representation to $(k,y)$ via learned basis functions 
$e_{k,y}(x)=\sqrt{p(x,y)}e^{ik\cdot x}$, enabling localized, data-dependent operator learning.
  }
\label{fig:FIG_TEASER}
\end{figure*}

\section{Methodology}
\textbf{Problem Statement and Preliminaries.} Consider a parametric partial differential equation of the form
$
F\big(u(x,t), x, t, a, \nabla u, \nabla^2 u, \partial_t u, \ldots \big) = 0,
\quad (x,t)\in \Omega := E \times [0,T],
$
where $E \subset \mathbb{R}^d$ denotes the spatial domain, the parametric function
$
a : \partial \Omega \to \mathbb{R}^{d_a}
$
represents boundary or initial conditions, and the solution is denoted by
$
u : \Omega \to \mathbb{R}^{d_u}.
$
Under suitable well-posedness assumptions, such a PDE induces a nonlinear solution operator
\begin{equation}
\mathcal{G}^\dagger : \mathcal{A} \subset L^2(\partial \Omega; \mathbb{R}^{d_a}) 
\to \mathcal{U} \subset L^2(\Omega; \mathbb{R}^{d_u}), 
\quad a \mapsto u,
\end{equation}
which maps the parametric input function to the corresponding solution field. In general, $\mathcal{G}^\dagger$ is a nonlinear operator between infinite-dimensional function spaces.
Neural operator methods aim to construct a parametric approximation $\mathcal{G}_\theta$ of $\mathcal{G}^\dagger$, typically through compositions of nonlinear integral operators acting on function spaces. More precisely, a neural operator layer is defined by
\begin{equation}
v_{l+1}(x) = G_l(v_l)(x) 
= \sigma\!\left( W_l v_l(x) + \int_E K_l(x,z)\, v_l(z)\, dz \right),
\end{equation}
where $W_l \in \mathbb{R}^{d_{l+1}\times d_l}$ is a linear operator acting pointwise, $K_l$ is an integral kernel, and $\sigma$ is a nonlinear activation function.
By composing sufficiently many such layers, one obtains a parameterized operator $\mathcal{G}_\theta$, which, under appropriate assumptions, satisfies a universal approximation property \citep{kovachki2021universal} for nonlinear operators. In particular, for any $\varepsilon > 0$, there exists a parameter set $\theta$ such that
$
\|\mathcal{G}_\theta - \mathcal{G}^\dagger\|_{L^2(\mu;\mathcal{U})} \le \varepsilon,
$
where $\mu$ is a probability measure on the input function space $\mathcal{A}$. 
% parameter space $\Theta$.
%
A fundamental class of neural operators arises by imposing translation invariance on the kernel, i.e.,
$
K_l(x,z) = K_l(x - z),
$
so that the integral operator reduces to a convolution. By the convolution theorem, such operators admit a spectral representation via the Fourier transform $\mathcal{F}$, namely,
\begin{equation}
\int_E K_l(x-z) v_l(z)\,dz 
= \mathcal{F}^{-1}\!\big( R_l(k)\, \mathcal{F}(v_l)(k) \big)(x),
\end{equation}
where $R_l(k) = \mathcal{F}[K_l](k)$ is a diagonal operator in the spectral domain.
This yields the Fourier Neural Operator (FNO), which can be written in operator form as
$
\mathcal{T}_{FNO} = \mathcal{F}^{-1} \circ R \circ \mathcal{F}.
$
Since the Fourier transform is an isometry on $L^2(E)$, satisfying Parseval’s identity
$
\|\mathcal{F}(f)\|_{L^2(E)} = \|f\|_{L^2(E)},
$
the above construction preserves the energy of the signal and enables efficient computation through the Fast Fourier Transform with complexity $O(N \log N)$.
More generally, one may replace the Fourier basis with an arbitrary orthonormal basis $\{e_k\}$ of $L^2(E)$, leading to a generalized spectral transform
$
\mathcal{F}_{gen}(f)_k = \langle f, e_k \rangle,
$
and corresponding operator representations of the form
$
\mathcal{T} = \mathcal{F}_{gen}^{-1} \circ R \circ \mathcal{F}_{gen}.
$
However, all such constructions rely on a fixed basis $\{e_k\}$ that is independent of data. As a consequence, the representation remains globally defined, and localized structures must be captured indirectly through high-frequency modes, often requiring truncation and leading to loss of fine-scale information.
Therefore, the central problem is to construct operators acting on $L^2(E)$ whose underlying representation is data-dependent, while preserving invertibility, energy conservation, and computational efficiency comparable to Fourier-based methods.

\subsection{Adaptive Learnable Basis with Data-Dependent Representations}
Unlike existing works, \textit{we propose a fundamentally different representation for neural operators in which the underlying basis is no longer fixed, but learned and spatially dependent} (see Figure~\ref{fig:FIG_TEASER}).  In contrast to Fourier-based neural operators that rely on global, translation-invariant bases, the proposed construction introduces a data-driven Parseval frame that adapts to the geometry and distribution of solution space, while still preserving isometry, full fidelity and computational efficiency.

Firstly, we introduce a density function
$
p:E\times \chi \rightarrow \mathbb{R}_{\ge 0}, \, (\textit{\textbf{x}},y)\mapsto p(\textit{\textbf{x}},y),
$
where $E$ is the spatial domain for the PDE, and $\chi$ is an auxiliary measurable space equipped with a probability measure $\mu$.($\chi$ is always taken as $\mathbb{R}$ or in implementation a discrete point set $[M]$.) We impose the normalization condition
$
\int_\chi p(\textit{\textbf{x}},y)\,\mu(dy)=1, \, \forall \textit{\textbf{x}}\in E.
$
We then construct \textbf{adaptive learnable basis}, an adaptive family of functions $\{e_{\textbf{k},y}\}_{\textbf{k}\in\mathbb{R}^d,\,y\in\chi}$ defined by
$
e_{\textit{\textbf{k}},y}(\textit{\textbf{x}})=\sqrt{p(\textit{\textbf{x}},y)}e^{i\textit{\textbf{k}}\cdot \textit{\textbf{x}}}.
$
Under the normalization condition on $p$, this family forms a Parseval frame for $L^2(E)$.
With respect to this adaptive frame, we define the transformation
$
\mathcal A: L^2(E)\rightarrow L^2(\mathbb{R}^d\times \chi)
$
by
\begin{equation}
\hat{f}_{k,y}
=
\mathcal{A}(f)_{k,y}
=
\langle e_{k,y},f\rangle
=
\frac{1}{(2\pi)^{d/2}} \int_E \sqrt{p(x,y)}\, e^{-ik\cdot x} f(x)\,dx.
\end{equation}
\textit{This construction departs from classical neural operator formulations by shifting the source of expressivity from learned spectral coefficients to the representation itself}: rather than just operating in a fixed global basis, the model also learns a spatially dependent frame through $p(\textbf{x},y)$. Consequently, localized and heterogeneous structures are encoded directly in the basis, instead of being approximated through increasingly high-frequency modes.
The inverse transformation $\mathcal A^{-1}: \mathrm{Im}(\mathcal A)\rightarrow L^2(E)$ is given by
\begin{equation}
f(x)
=
\mathcal{A}^{-1}(\hat f)(x)
=
\int_{\chi}
\sqrt{p(x,y)}
\left(
\frac{1}{(2\pi)^{d/2}} \int_{\mathbb{R}^d} e^{ik\cdot x}\,\hat f_{k,y}\,dk
\right)
\mu(dy).
\end{equation}

Invertibility $\mathcal A^{-1}\mathcal A = \mathrm{Id}$ is preserved, which can be verified by:
\begin{align}
(\mathcal A^{-1}\mathcal A f)(x)
&=
\int_{\chi}\sqrt{p(x,y)}
\frac{1}{(2\pi)^{d/2}} \int_{\mathbb{R}^d} e^{ik\cdot x}
\left(
\frac{1}{(2\pi)^{d/2}} \int_E \sqrt{p(z,y)} e^{-ik\cdot z}f(z)\,dz
\right)
dk\,\mu(dy) \nonumber\\
&=
\int_E
\left[
\int_{\chi}\sqrt{p(x,y)p(z,y)}
\frac{1}{(2\pi)^d} \int_{\mathbb{R}^d} e^{ik\cdot(x-z)}dk
\,
\mu(dy)
\right]
f(z)\,dz.
\end{align}

Using
$
\frac{1}{(2\pi)^d} \int_{\mathbb{R}^d} e^{i\textbf{k}\cdot(\textbf{x}-\textbf{z})}dk = \delta(\textbf{x}-\textbf{z}),
$
and by the normalization setting $\int_\chi p(\textbf{x},y)\mu(dy)=1$, we obtain $(\mathcal A^{-1}\mathcal A f)(x)=f(x)$.

Secondly, $\mathcal A$ satisfies a Parseval identity:
\begin{align}
\|\mathcal A(f)\|_{L^2(\mathbb{R}^d\times \chi)}^2
&=
\int_{\chi}\int_{\mathbb{R}^d} |\hat f_{k,y}|^2\,dk\,\mu(dy) 
=
\int_E |f(x)|^2
\left(
\int_\chi p(x,y)\mu(dy)
\right)
dx
=
\|f\|_{L^2(E)}^2.
\end{align}
Thus $\mathcal A$ is an isometry. This further shows that the transformation preserves the $L^2$ norm. 
We formalize this key property below.
\begin{theorem}[Isometry of the ABLE transformation]
Let the ABLE transformation $\mathcal{A}: L^2(E) \rightarrow \mathrm{Im}(\mathcal{A}) \subset L^2(\mathbb{Z}^d\times\chi)$ be defined by
$f \mapsto \hat f_{\textbf{k},y}=\langle f,e_{\textbf{k},y}\rangle$. Under the normalization condition $\int_\chi p(\textbf{x},y)\mu(dy)=1$, 
the transformation $\mathcal{A}$ is a bijection and an isometry onto its image, i.e.,
$
\|f\|_{L^2(E)} = \|\hat f\|_{L^2(E \times \chi)},
\quad 
\|f\|_{L^2(E)}^2 = \int_\chi \sum_{\textbf{k}\in \mathbb{Z}^d} |\hat f_{\textbf{k},y}|^2 \,\mu(dy).
$
Moreover, the inverse is given by
$
\mathcal{A}^{-1}(\hat f)(\textbf{x}) = \int_\chi \sum_{\textbf{k}\in \mathbb{Z}^d} \hat f_{\textbf{k},y} e^{i\textbf{k}\cdot \textbf{x}}\sqrt{p(\textbf{x},y)}\,\mu(dy).
$ {\normalfont (Proof in Appendix A theorem~\ref{theorem_4})}
\end{theorem} 
This result guarantees that the transformation is information-preserving and norm-stable, ensuring that operator learning in the transformed space remains faithful to the original signal. 
Further, any operator $\mathcal{G}:L^2(E)\rightarrow L^2(E)$ admits the unique representation
$
\mathcal{G}=\mathcal{A}^{-1}\circ \hat{\mathcal{G}}\circ \mathcal{A},
$
where $\hat{\mathcal{G}}$ acts on $L^2(\mathbb{R}^d\times \chi)$.
In practice, the above formulation is discretized on a finite grid and implemented via FFT.
More generally, for any spectral basis $e_{k}(x)$, one may define
$
e_{k,y}(x)=e_{k}(x)\sqrt{p(x,y)},
$
yielding a generalized transformation $\mathcal A_{gen}$. But in practice, only the Fourier basis admits efficient FFT implementation.

\begin{figure} %[t!]
\centering
\begin{subfigure}{0.45\textwidth}
	\centering
	\begin{tikzcd}
		L^2(E)\arrow[r,"\mathcal{T}_{FNO}"]\arrow[d,"\mathcal{F}"]&L^2(E)\arrow[d,"\mathcal{F}"]\\
		L^{2*}(E)\arrow[r,"\mathcal{R}"] & L^{2*}(E)
	\end{tikzcd}\caption{Standard Fixed Fourier basis representation}
    \end{subfigure}
    \hfill
    \begin{subfigure}[b]{0.45\textwidth}
    \centering
	\begin{tikzcd}
		L^2(E)\arrow[r,"\mathcal{G}"]\arrow[d,"\mathcal{A}"]&L^2(E)\arrow[d,"\mathcal{A}"]\\
		\text{Im}(\mathcal{A})\subset L^{2}(E\times\chi)\arrow[r,"\hat{G}"] & \text{Im}(\mathcal{A})
	\end{tikzcd}\caption{ABLE: Adaptive basis representation}\label{able basis}
    \end{subfigure}
    \caption{\textbf{Breaking the limitations of fixed spectral representations.}
Spectral operators, including Fourier neural operators, rely on a fixed global basis, restricting them to translation-invariant structure. ABLE learns the representation via a data-dependent lifting, enabling spatially adaptive operators.}
 \label{fig:scheme}
\end{figure}

\subsection{ABLE Neural Operator: Operator Learning under Adaptive Basis}
We introduce an \textbf{Adaptive Basis Learning} Neural Operator (ABLE-NO), \textit{a new class of neural operators} defined in a data-dependent spectral representation induced by our adaptive transformation $\mathcal A$. Unlike classical neural operators, where learning is restricted to spectral coefficients in a fixed basis, ABLE-NO jointly learns both the representation and the operator itself, thereby shifting the source of expressivity from the coefficients to the underlying basis.
We adopt the same unitary Fourier convention as in the previous subsection. In this formulation, the lifted operator $\hat{\mathcal G}$ is modeled as a learnable spectral multiplier $R(k,y)$ that acts on the adaptive basis, yielding
$
\mathcal{T}_{ABLE}
=
\mathcal A^{-1}\circ R \circ \mathcal A,
$
which maps an input function $f$ to $g=\mathcal{T}_{ABLE}(f)$ via
\begin{equation}
g(x)
=
\int_{\chi}
\int_{\mathbb{R}^d}
e_{k,y}(x)\, R(k,y)\, \langle e_{k,y},f\rangle
\,dk\,\mu(dy).
\end{equation}
Here, the adaptive basis functions are given by
$
e_{k,y}(x)=\sqrt{p(x,y)}\,\frac{1}{(2\pi)^{d/2}} e^{ik\cdot x},
$
so that the representation explicitly couples spatial modulation with spectral structure.
Substituting the definition of the coefficients
$
\langle e_{k,y},f\rangle
=
\frac{1}{(2\pi)^{d/2}}
\int_E \sqrt{p(z,y)} e^{-ik\cdot z} f(z)\,dz,
$
into \eqref{ABLE_NO}, we obtain an equivalent kernel representation:
\begin{align} \label{ABLE_NO}
g(x)
&=
\int_{\chi}
\int_{\mathbb{R}^d}
\sqrt{p(x,y)} \frac{1}{(2\pi)^{d/2}} e^{ik\cdot x} R(k,y)
\left(
\frac{1}{(2\pi)^{d/2}}
\int_E \sqrt{p(z,y)} e^{-ik\cdot z} f(z)\,dz
\right)
dk\,\mu(dy) \nonumber\\
&=
\int_E
\left[
\int_{\chi}
\sqrt{p(x,y)}
\left(
\frac{1}{(2\pi)^d}
\int_{\mathbb{R}^d}
e^{ik\cdot(x-z)} R(k,y)\,dk
\right)
\sqrt{p(z,y)}
\mu(dy)
\right]
f(z)\,dz.
\end{align}
Defining the inverse Fourier transform of the multiplier
$
R(x-z;y)
=
\frac{1}{(2\pi)^d}
\int_{\mathbb{R}^d}
e^{ik\cdot(x-z)} R(k,y)\,dk,
$
the induced kernel takes the form
$
K_{ABLE}(x,z)
=
\int_{\chi}
\sqrt{p(x,y)}\,R(x-z;y)\,\sqrt{p(z,y)}\,\mu(dy),
$
and therefore
$
g(x)=\int_E K_{ABLE}(x,z)f(z)\,dz.
$
This formulation reveals a fundamental departure from Fourier neural operators. Even when $p(x,y)$ is independent of the input $f$, the resulting kernel is generally not necessarily translation invariant, since
$
K_{ABLE}(x,z)\neq K(x-z).
$
Consequently, ABLE-NO naturally captures spatially heterogeneous and locally modulated interactions that cannot be represented within the translation-invariant framework of classical spectral operators.
Moreover, the density can be parameterized as an input-dependent function $p_\theta(x,f(x),y)$, in which case the representation itself becomes adaptive to the input. The induced kernel then takes the form
\begin{equation}
K_{ABLE}(x,z;f)
=
\int_{\chi}
\sqrt{p_\theta(x,f(x),y)}\,
R(x-z;y)\,
\sqrt{p_\theta(z,f(z),y)}
\,\mu(dy),
\end{equation}
introducing nonlinearity directly at the level of the basis rather than through successive nonlinear layers.
For efficient implementation, we restrict $\chi=[M]$ to a finite set with counting measure. Then \eqref{ABLE_NO} reduces to
\begin{align}
g(x)
&=
\sum_{m=1}^{M}
\sqrt{p(x,m)}
\mathcal F^{-1}
\Big[
R(k,m)\,
\mathcal F\big(\sqrt{p(\cdot,m)}f(\cdot)\big)
\Big](x) \nonumber=
\int_E
\left[
\sum_{m=1}^{M}
\sqrt{p(x,m)}\,R(x-z;m)\,\sqrt{p(z,m)}
\right]
f(z)\,dz.
\end{align}

The first equality exactly preserves the FFT-based computational structure of FNO, with complexity $O(MN\log N)$, while significantly enlarging the class of representable operators.
\par
Finally, FNO is recovered as a special case. If $\chi=\{a\}$ and $p(x,a)=1$, then $e_{k,a}(x)=\frac{1}{(2\pi)^{d/2}} e^{ik\cdot x}$ and
$
g(x)
=
\int_{\mathbb{R}^d}
e^{ik\cdot x} R(k)\langle e_k,f\rangle\,dk,
$
which coincides with the Fourier neural operator. Thus, ABLE-NO strictly generalizes FNO while introducing a fundamentally more expressive, spatially adaptive operator representation. This distinction is illustrated in Figure~\ref{fig:scheme}, where ABLE replaces the fixed Fourier transform with a data-dependent lifting $\mathcal{A}$ that defines an adaptive spectral representation. This statement can be formalized as follows (with proof in Appendix A theorem~\ref{theorem_5}).\begin{theorem}[Expressivity of ABLE] The ABLE operator defines a strictly larger class of operators than the Fourier Neural Operator.
FNO is just a special case of ABLE: when $p(x,y)$ takes constant, ABLE transformation reduces to Fourier transformation, thus the ABLE operator 
$\mathcal{G}_{ABLE} = \mathcal{A}^{-1} \circ \hat{\mathcal{R}}_{ABLE} \circ \mathcal{A}$ 
reduces to the Fourier Neural Operator
$\mathcal{G}_{FNO} = \mathcal{F}^{-1} \circ \hat{\mathcal{R}} \circ \mathcal{F}$.
Furthermore, the ABLE operator induces kernels not necessarily translation-invariant, therefore obtains strict inclusion.
\end{theorem}

\textbf{Adaptive Spectral Decomposition.} We restrict $\chi=[M]$ and obtain the discretized ABLE layer:
\begin{equation}
f_{l+1}(x)
=
\sum_{m=1}^{M}
\sqrt{p_{\theta}(x,f_l,m)}
\mathcal{F}^{-1}
\Big[
R(k,m)\,
\mathcal{F}\big(\sqrt{p_{\theta}(\cdot,f_l,m)}f_l(\cdot)\big)
\Big](x).
\label{eq: ABLE_layer_noncrossing}
\end{equation}
This formulation admits a clear interpretation. The input $f_l$ is first decomposed into $M$ spatially adaptive components through the weights $\sqrt{p_{\theta}(x,f_l,m)}$, which act as a soft partition of unity over the domain. Each component is then processed independently in the spectral domain via $R(k,m)$, before being recombined in physical space.

ABLE can be viewed as lifting the input into multiple adaptive spectral subspaces indexed by $m$, where the interaction between subspaces is controlled by the learned density $p_{\theta}(x,f_l,m)$. This mechanism provides a flexible trade-off between global spectral structure and localized representation. Moreover, cross-interactions between components can be introduced via
\begin{equation}
f_{l+1}(x)
=
\sum_{m=1}^{M}
\sqrt{p_{\theta}(x,f_l,m)}
\mathcal{F}^{-1}
\Big[
\sum_{m'=1}^{M}
R(k;m,m')\,
\mathcal{F}\big(\sqrt{p_{\theta}(\cdot,f_l,m')}f_l(\cdot)\big)
\Big](x), \label{eq: ABLE_layer_crossing}
\end{equation}
which allows information exchange between different adaptive components. This increases parameterization to $O(M^2 k_{max} C^2)$ while preserving the same asymptotic computational complexity dominated by FFT. Moreover, with this notation (\ref{eq: ABLE_layer_crossing}), our previously proposed non-crossing variant defined in (\ref{eq: ABLE_layer_noncrossing}) can be subsumed into the crossing case by setting $R(k;m,m')=R(k,m)\delta_{m,m'}$.
A detailed description of the differences between non-crossing and crossing implementation is also shown in Fig.~\ref{fig:FIG_TEASER2} in Appendix B.

\subsection{Rank Recovery in Adaptive Basis Learning Extension}
Furthermore, whether with or without cross-interaction, the ABLE generally exhibits rank recovery, which provides another illustration of its improvement in expressivity and high-frequency fidelity from an operator-theoretic perspective.

In the implementation of the Fourier Neural Operator, the spectral representation is truncated to the first $K^d$ Fourier modes, since further increasing the mode count hardly improves accuracy and risks overfitting as well as aliasing. Hence, the kernel integral of an FNO is implemented by
\begin{align}
\int_E K_{FNO}(x-z) v(z)dz=\mathcal{F}^{-1}\bigg[R(\mathbf{k})\mathcal{F}[v](\mathbf{k})\bigg](x)=\sum_{\mathbf{k}\le \mathbf{K}} e^{i\textbf{k}\textbf{x}} R(\textbf{k})\sum_\textbf{z} e^{-i\textbf{k}\textbf{z}} v(\textbf{z})=\int_E \bigg[\sum_{\textbf{k}\le \textbf{K}} e^{i\textbf{k}\textbf{x}} R(\textbf{k}) e^{-i\textbf{k}\textbf{z}}\bigg] v(z)dz,
\end{align}
with $K_{FNO}(x-z)=\sum_{\textbf{k}\le \textbf{K}} e^{i\textbf{k}\textbf{x}} R(\textbf{k}) e^{-i\textbf{k}\textbf{z}}$.
Since $\overline{K_{FNO}(x-z)}=K_{FNO}(z-x)$, the Hermitian condition $\langle f, \mathcal{G}_{K}g\rangle = \int_{E,E}\bar{f}(z)K_{FNO}(z-x)g(x)= \int_{E,E}\overline{K_{FNO}(x-z)f(z)}g(x) = \langle \mathcal{G}_{K} f, g\rangle$ holds. And with only the first $K^d$ modes preserved, this FNO kernel operator is bounded as $\int_{E\times E} |K_{FNO}(x-z)|^2dxdz=\sum_{k=1}^{K^d}|R(k)|^2< \infty$, and further becomes a self-adjoint finite rank operator with $\text{rank}(\mathcal{G}_{FNO})\le K^d$. Admittedly, the Hermitian property aligns FNO well with real-world physical observables; yet its truncation—whether as low-pass filtering in the spectral domain or as finite-rank approximation in operator space—entails the loss of high-frequency structural details.

In comparison, by extending Fourier basis to \textbf{Adaptive Learnable Basis}, the ABLE extension adopts the same truncation $K^d$, with its kernel given by
\begin{align}
K_{ABLE}(x,z)=\sum_{m,m'}\sum_{\textbf{k}<\textbf{K}}(e^{i\textbf{k}\cdot\textbf{x}}p(\textbf{x},m)^{\frac{1}{2}})R(\textbf{k};m,m')(e^{-i\textbf{k}\cdot\textbf{z}}p(\textbf{z},m')^{\frac{1}{2}}),
\end{align} 
which retains self-adjointness while attaining greater expressive power via heterogeneous components and rank recovery. Since $\sqrt{p(x,m)}$ is real-valued and $\overline{R(\textbf{k};m,m')}=R(-\textbf{k};m,m')$ in practice, the ABLE kernel still satisfies the Hermitian condition $\overline{K_{ABLE}(x,z)}=K_{ABLE}(z,x)$: 
\begin{align}
\overline{K_{ABLE}(x,z)}=&\overline{\sum_{m,m'}\sum_{\textbf{k}<\textbf{K}}(e^{i\textbf{k}\cdot\textbf{x}}p(\textbf{x},m)^{\frac{1}{2}})R(\textbf{k};m,m')(e^{-i\textbf{k}\cdot\textbf{z}}p(\textbf{z},m)^{\frac{1}{2}})}\nonumber\\=&
\sum_{m,m'}\sum_{\textbf{k}<\textbf{K}}(e^{i\textbf{k}\cdot\textbf{z}}p(\textbf{z},m)^{\frac{1}{2}})R(\textbf{k};m,m')(e^{-i\textbf{k}\cdot\textbf{x}}p(\textbf{x},m)^{\frac{1}{2}})
=K_{ABLE}(z,x),
\end{align} and the $K^d$ truncation and normalization condition $\sum_mp(x,m)=1$ ensure its boundedness $\int_{E\times E}K_{ABLE}(x,z)dxdz\le\sum_{m,m'=1}^{M}\sum_{k=1}^{K^d}|R(k;m,m')|^2<\infty$. Thus together with the Hermitian property, this confirms that the ABLE operator is still self-adjoint. Moreover, because the $K^dM$ components $e^{i\textbf{k}\cdot\textbf{x}}\sqrt{p(\textbf{x},m)}$ generally span a $K^dM$-dimensional subspace of $L^2(E)$, this raises the maximal achievable rank of the learned operator to $K^dM$, i.e., $\text{rank}(\mathcal{G}_{ABLE})\le K^dM$.
This property is summarized as follows (with proof in Appendix A theorem~\ref{thm13:finite_rank}).
\begin{theorem}[Finite-rank self-adjoint property of FNO and ABLE]\label{thm: main_rank_recovering} Suppose the multipliers $R(k)$ or $R(k;m,m')$ are truncated with only $\textbf{k}\le\textbf{K}$ modes, then both ABLE and FNO kernel integral operators are finite-rank self-adjoint operators. 
\begin{enumerate}
\item The maximum rank of the FNO operator family $\mathcal{T}_{FNO}$ satisfies $\max \text{rank}(\mathcal{T}_{FNO})= K^d$, 
\item The maximum rank of the ABLE operator family $\mathcal{T}_{ABLE}$ satisfies $K^d \le \max\text{rank}(\mathcal{T}_{ABLE})\le K^d M$, i.e., ABLE provides an extension of spectral neural operators that preserves finite-rank self-adjoint property.
\item $\max\text{rank}(\mathcal{T}_{ABLE})= K^d M$ iff $\dim\text{Span}\{e_{k,m}\}_{k=1,m=1}^{K^d,M}=K^dM $. (A trivial example is given by a Softmax modeled basis $e_{k,m}=p(x,m)^{\frac{1}{2}}e^{ikx}=(\frac{e^{\alpha_m x}}{\sum_{m=1}^{m'}e^{\alpha_{m'}x}})^{\frac{1}{2}}e^{ikx}$ with $\alpha_m\in \mathbb{R}^d$ and $|\alpha_m|$ all distinct.)
\end{enumerate}
\end{theorem}
In particular, the construction in the third statement originates from our earlier engineering Softmax implementation, but it even provides intuitive evidence for the superiority of Softmax model design, stemming from enhancing the independence of ABLE basis functions. 
\iffalse
\begin{theorem}[Finite-rank self-adjoint property of FNO and ABLE]\label{thm: main_rank_recovering}
Both the ABLE and FNO kernel integral operators are finite-rank, self-adjoint operators. 
\begin{enumerate}
\item $\text{rank}(\mathcal{G}_{FNO})=K^d$,
\item $K^d \le \text{rank}(\mathcal{G}_{ABLE})\le K^dM$,
\end{enumerate} i.e., ABLE provides an extension of spectral neural operators that preserves finite-rank self-adjoint property.
\end{theorem}
\fi
By recovering the rank with learnability, ABLE-NO enhances the penetration of high-frequency modes through the operator layers and alleviates the truncation problem. And the rank recovery is data-adaptive; adaptive basis learning then drives convergence toward a more efficient representation under the same truncation constraint.
\subsection{When the Basis Heats Up: Temperature Controlled Heterogeneity}\label{section: temperature controlled heterogeneity}
Finally, we parameterize the density $p_\theta(x,m)$ in adaptive learnable basis $e_{k,m}(x)=\sqrt{p(x,m)}e^{ikx}$ with $m\in\chi=[M]$ via a Softmax‑MLP.  
An MLP maps local features $f(x)$ to $M$ real numbers $\epsilon_\theta(x,m)=\mathrm{MLP}_\theta(f(x))_m$, which can be interpreted as energy‑like quantities. The density is then
\begin{align}
    p_{\theta}(x,m)=\text{Softmax}(\epsilon_{\theta}({x},m))=\frac{\exp( \epsilon_{\theta}({x},m)/T)}{\sum_{m'=1}^{M}\exp(\epsilon_{\theta}({x},m')/T)}\end{align}
where $T>0$ is a temperature parameter. 
%(either fixed or learned). 
This construction admits an analogy with statistical physics. The spatial grid $E$ can be viewed as a lattice, where each point $\textbf{x}\in E$ has $M$ possible \textit{micro‑states} with energy $\epsilon_\theta(\textbf{x},m)$. The distribution $\{p_\theta(\textbf{x},m)\}_m$ then resembles a local Boltzmann distribution across these $M$ states, and it becomes a background field to encode the heterogeneity of the data. With this learnable construction of $p(\textbf{x},m)$, ABLE learns how to distribute probability mass across these states to construct a spatially adaptive representation. Under this learned representation, the spectral multipliers $R(\textbf{k},m)$ or $R(\textbf{k},m,m')$ further govern the evolution of the field on this heterogeneous lattice. 

\textbf{Basis size $M$: source of heterogeneity and scale divergence}. For each family $m\in [M]$, there are $K^d$ basis components $e_{k,m}=\sqrt{p(x,m)}e^{ikx}$, $\textbf{k}\le \textbf{K}$. When $M=1$, the model has only a single  family—the trivial Fourier modes— and reduces to FNO—a homogeneous lattice with a single state. It becomes interesting when $M>1$: each pointwise-defined $MLP_{\theta}(f(x))_m$ adopts distinct parameters and thus could behave differently, collectively manifesting as heterogeneous crossover of multiple components at the global scale. A scalar observation—or an illustrative measure of heterogeneity—is \textbf{Normalized Variance}:
\begin{equation}\text{NVar}(p) =\frac{M}{M-1}\bigg[\sum_{m=1}^{M} p(x,m)^2-\frac{1}{M}\bigg], \end{equation} which ranges from $0$ to $1$. Specifically,
$\text{NVar} = 0$ when all $p(x,m)\equiv\frac{1}{M}$ take the same value everywhere, whereupon the ABLE-NO degenerates to a homogeneous kernel action as in FNO; conversely, $\text{NVar}\sim 1$ when exactly one component becomes dominant and the others vanish. Another state-wise evaluation metric is the \textbf{correlation function} of $p(\cdot,m)$, defined as
\begin{equation}
C(\textbf{r},m)=\sum_{\mathbf{x}}p(\textbf{x},m)p(\textbf{x}+\textbf{r},m)\sim \exp(-\frac{|\textbf{r}|}{\xi_m}),
\end{equation} which could estimate the scale behavior of each basis family $m\in [M]$. Each $m$ could manifest a different correlation length $\xi_m$, indicating that scale diversity is captured in the ABLE basis with each $m$ representing a different scale (also shown in Fig.~\ref{fig:FIG_TEASER2} in Appendix B).
Consequently, the entire ABLE simulation becomes not merely a component crossover but a multiscale crossover. Actually, the Softmax-MLP modeled basis does not enforce that the learned states must manifest scale diversity, but it endows the ABLE with the ability to adapt to the heterogeneity and multi-scale behavior in the data and the target operator.  

\textbf{Temperature T: control of heterogeneity}. In this framework, the temperature $T$ is not a physical temperature, but controls the sharpness of the distribution $p$, further serving as a control of heterogeneity in the representation. The Lipschitz continuity of $p(\textbf{x},m)$ is controlled by \begin{align}\textbf{Lip}(p_m)=\textbf{Lip}\bigg(\text{Softmax}\big(\frac{\text{MLP}_{\theta}(f(x))}{T}\big)\bigg)\le \frac{1}{T}\textbf{Lip}(\text{Softmax})\textbf{Lip}(\text{MLP})=\frac{\textbf{Lip}(\text{MLP})}{T}.\label{eq:Lipshitz control}\end{align} (See also theorem~\ref{thm: temperature controlled heterogeneity} in Appendix A).

In the limit $T\to 0$, the Lipschitz continuity becomes unbounded, and the distribution becomes nearly deterministic: each spatial location selects and freezes to exactly one dominant state (and scale), equipping the representation with localized or discontinuous features. Specifically, $\forall x, \exists \text{ a dominant } i_{x}\in [M]$, $ p(x,i_x)=1$ and $p(x,j)=0$ for all $j\neq i_x$. Consequently, the normalized variance $\text{NVar}(p)$ approaches $1$, representing the strongest possible heterogeneity. 

Conversely, as $T\to\infty$, the Lipschitz constant approaches $0$, and all states become equally likely with $p(x,m)\equiv\frac{1}{M}$, recovering a fully mixed representation identical to an over-parameterized homogeneous spectral operator (FNO) of matching width $C$. The effective parameter count reduces from $O(Mk_{\max}C^2)$ (or $O(M^2k_{\max}C^2)$ for the cross-coupled case) to $O(k_{\max}C^2)$ and rank drops from $MK^d$ to $K^d$, as shown by
\begin{align}K^{T\rightarrow \infty}_{ABLE}(x,z)=&\sum_{m,m'=1}^{M}\sum_kR(k;m,m')e^{ik(x-z)}p(x,m)^{\frac{1}{2}}p(z,m')^{\frac{1}{2}}=\sum_k[\frac{1}{M}\sum_{m,m'=1}^{M}R(k;m,m')]e^{ik(x-z)}\nonumber\\=&\sum_k \tilde{R}(k)e^{ik(x-z)}=K_{FNO}(x-z).\label{eq: high-temp}\end{align} 
The normalized variance $\text{NVar}(p)=0$, and no scale differences emerge across families.

Varying $T$ therefore induces a smooth transition between these two regimes, with analogous critical behavior at both endpoints—distinct from a physical phase transition but rather manifesting as crossover dynamics of different scales. A formal and concrete analysis of the above temperature‑dependent expressivity is given in theorems~\ref{thm: temperature controlled heterogeneity},~\ref{thm: low_temp} and ~\ref{thm: high_temp} in Appendix~A.

\subsection{Efficiency through algorithm formulation}
\begin{algorithm}[t]
\caption{Forward Pass of the Adaptive Basis Learning Operator (ABLE)}
\label{alg:ABLE_forward}
\KwIn{
    $f(\textbf{x})\in \mathbb{R}^{C_\mathrm{in}} $ for all $\mathbf{x}$: either \textbf{(i)} same-time or time-independent $f(\mathbf{x})\in\mathbb{R}^{C_{\mathrm{in}}}$ (e.g., Burgers, Darcy) or \textbf{(ii)} multi-frame $\{f(\mathbf{x}, t_i)\}_{i=0}^{T-1}$ stacked as $C_{in}=c_{\mathrm{in}}\times T$ channels (e.g., NS with historical frames); \\
    Layers $L$; hidden channels $C$; mode truncation $K$; \\
    Basis size $M$; Softmax temperature $T$;\\ $\mathrm{MLP}_\theta: \mathbb{R}^{C}\!\to\!\mathbb{R}^{M}$ for \textbf{(i)} / $\mathrm{MLP}: \mathbb{R}^{C}\!\to\!\mathbb{R}^{C\times M}$ where $\mathrm{MLP}^{(c)}$ outputs $\mathbb{R}^M$ for each channel $c\in [C]$ for \textbf{(ii)};\\
    Learnable parameters $\{P,b_P, Q, \theta, W_\ell, b_\ell, R\}$
}

\KwOut{$g(\mathbf{x}) \in \mathbb{R}^{C_{\mathrm{out}}}$ for all $\mathbf{x}$}

$\mathbf{v}_0 = P f+b_P$ %\; 
\tcp*{$P: \mathbb{R}^{C_{\mathrm{in}}} \to \mathbb{R}^{C}$ lifting}

\For{$\ell = 1$ \KwTo $L$}{
    \tcp{1. Basis probabilities (auxiliary functions) computation}
    \If{multi-frame input \textbf{(ii)}}{
        $p_\ell(\mathbf{x}, m, c) = \mathrm{Softmax}_m(\mathrm{MLP}_\theta^{(c)}(\mathbf{v}_{\ell-1})/T),\; \tilde{p} = \sqrt{p}$ \tcp*{Softmax over basis index $m \in [M]$}
    }
    \Else{
        $p_\ell(\mathbf{x}, m) = \mathrm{Softmax}_m(\mathrm{MLP}_\theta(\mathbf{v}_{\ell-1})/T),\; \tilde{p} = \sqrt{p}$ \tcp*{Softmax over basis index $m\in [M]$}
    }
    
    \tcp{2. ABLE transformation $\mathcal{A}$ }
    $\hat{\mathbf{v}}_{\ell-1,k,m}^{(c)} = \mathrm{FFT}\!\left(\tilde{p}(\mathbf{x}, m, [c]) \odot \mathbf{v}_{\ell-1}^{(c)}\right)$ \tcp*{$\tilde{p}(\mathbf{x}, m)$ broadcast over $c\in [C]$ if \textbf{(i)}}
    
    \tcp{3. Spectrum multiplication of $R(k,m)$ or $R(k;m,m')$}
    \If{no crossing}{
        $\hat{\mathbf{v}}^{\mathrm{spec},(c)}_{\ell,k,m} = \sum_{c'} R_{c,c'}(k,m)\,\hat{\mathbf{v}}_{\ell-1,k,m}^{(c')}$ \;
    }
    \Else{
        $\hat{\mathbf{v}}^{\mathrm{spec},(c)}_{\ell,k,m} = \sum_{c'}\sum_{m'} R_{c,c'}(k;m,m')\,\hat{\mathbf{v}}_{\ell-1,k,m'}^{(c')}$ \;
    }
    
    \tcp{4. Inverse ABLE transformation $\mathcal{A}^{-1}$}
    $\mathbf{v}^{\mathrm{spec},(c)}_\ell(\mathbf{x}, m) = \tilde{p}(\mathbf{x}, m, [c]) \odot \mathrm{IFFT}(\hat{\mathbf{v}}^{\mathrm{spec},(c)}_{\ell,k,m})$ \tcp*{$\tilde{p}(\mathbf{x}, m)$ broadcast over $c\in [C]$ if \textbf{(i)}}
    $\mathbf{v}^{\mathrm{spec}}_\ell = \sum_{m} \mathbf{v}^{\mathrm{spec}}_\ell(\cdot, m)$ \;
    
    \tcp{5. Spatial branch \& fusion}
    $\mathbf{v}_\ell = \mathrm{GeLU}\!\left(\mathbf{v}^{\mathrm{spec}}_\ell + W_\ell\mathbf{v}_{\ell-1}+b_{\ell}\right)$ \;
}

$g = Q\,\mathbf{v}_L$ \tcp*{$Q: \mathbb{R}^{C} \to \mathbb{R}^{C_{\mathrm{out}}}$ 2-layer MLP projection}
\Return{$g$} \;
\end{algorithm}
The forward passes of ABLE is summarized in Algorithms~\ref{alg:ABLE_forward}; the entire model is trained end-to-end in a fully differentiable manner. The corresponding algorithm for FNO is presented in Alg.~\ref{alg:FNO_forward} in Appendix B for comparison. While ABLE introduces additional learnable components---specifically, the basis Softmax-MLP probability network $\mathrm{MLP}_\theta$ and the expanded spectral weight tensor $R \in \mathbb{C}^{C \times C \times K \times M}$---the overhead is highly favorable. The dominant parameter growth stems from the expanded spectral weights; however, to make these weights effective rather than redundant, the basis network $\mathrm{MLP}_\theta$ is essential and irreplaceable for adaptive representation learning. It is lightweight (a small $\mathbb{R}^{C} \to \mathbb{R}^{M}$ or $\mathbb{R}^{C \times M}$ MLP per layer), yet it injects \emph{heterogeneity} and \emph{scale diversity} into the convolution kernel by learning input-dependent, spatially varying basis functions. This substantially enhances the operator's expressivity without altering the underlying FFT/IFFT-compatible computational structure.
Crucially, because the ABLE transform $\mathcal{A}$ and its inverse $\mathcal{A}^{-1}$ are implemented via pointwise multiplication followed by standard FFT/IFFT, the additional cost is dominated by $M$ parallel FFT calls, yielding an overall complexity of $O(MN\log N)$-the same asymptotic scaling as FNO up to a factor of $M$, with $M$ remaining small in practice. 

Despite the notably growing adoption of Transformer-based architectures in scientific machine learning, ABLE attains comparable accuracy and strong generalization across diverse PDEs with this favorable computational footprint, highlighting its efficiency.

\section{Experimental Results}
\textbf{Experimental Setting.} We evaluate ABLE on three PDEs of increasing complexity: 1D Burgers', 2D Darcy flow, and 2D Navier–Stokes (NS)~\citep{DBLP:journals/corr/abs-2010-08895}. For Burgers', we consider multiple viscosities (including $\nu=10^{-3}$) to cover regimes from smooth to shock-dominated. Darcy flow follows the standard FNO setting and represents a smooth, diffusion-dominated case. For NS, we use the low-viscosity regime ($\nu=10^{-5}$), characterized by turbulent, multiscale dynamics challenging for fixed-basis methods. Experiments use an RTX 4090 (1D) and RTX 5090 (2D); training times are comparable only across identical hardware. All models are trained under identical conditions.

\textbf{Baselines.}
We compare against neural operator baselines spanning neural networks (DeepONet~\citep{lu2022comprehensive}, U-Net~\citep{DBLP:journals/corr/RonnebergerFB15}), spectral methods (FNO~\citep{DBLP:journals/corr/abs-2010-08895}, SNO~\citep{Fanaskov2022SpectralNO}, WaveletNO~\citep{Tripura2022WaveletNO}, GF-NO~\citep{QI2024106239}), attention-based latent spectral models (Galerkin Transformer~\citep{Cao2021ChooseAT}, LSM~\citep{Wu2023SolvingHP}, HPM~\citep{Yue2024HolisticPS}), and hybrid architectures (U-FNO~\citep{Wen2021UFNOA}, FreqMoE~\cite{Chen2025FreqMoEDF}, SAOT~\citep{Zhou2025SAOTAE}). To assess generality, we integrate ABLE into existing operators, including U-ABLE (from U-FNO~\citep{Wen2021UFNOA}) for Darcy flow and ABLE-SAOT and ABLE-HPM for Navier–Stokes, consistently improving performance.
ABLE acts as a general spectral refinement with minimal overhead, complementing existing architectures by making the representation adaptive rather than modifying the operator itself. For example, in HPM~\citep{Yue2024HolisticPS}, ABLE replaces the fixed Laplacian basis with a learnable one; details are provided in Appendix~\ref{Sec: Appendix B4}.

\textbf{ABLE parameterization.} We model $p(\mathbf{x},m)$ via an MLP-Softmax. Burgers uses a 4-layer MLP on finite-difference features $[f_i,f_i',f_i'',1]$, while Darcy and Navier–Stokes use a 2-layer MLP; channels share a basis except in Navier–Stokes, where they are independent. Outputs use a temperature-controlled Softmax over $M$ (details are provided in Appendix~\ref{Sec: Appendix B3}).
The spectral kernel is diagonal or includes cross-component interactions. Complexity remains $O(MN\log N)$ with negligible overhead over FNO, and the ABLE layer directly replaces the Fourier layer.

\subsection{Main Results \& Discussion.}
\textbf{Burger's equation} As shown in Tab.~1, ABLE achieves the best accuracy across all viscosity regimes, with the largest gains in the low-viscosity case ($\nu=10^{-3}$), where solutions exhibit sharp shocks and strong high-frequency components. Compared to FNO, ABLE reduces the error from $\sim24\%$, demonstrating its advantage in capturing discontinuities and localized structures. The predictions from either baselines are shown in Fig. \ref{fig: Burgers_4_sample}. While classical spectral methods (FNO, GF-NO) perform well in smoother regimes, their fixed global bases limit their ability to represent shock-dominated dynamics, whereas ABLE adapts the basis to the data, improving representation across frequencies. Notably, ABLE maintains competitive computational cost while outperforming all baselines, indicating that the gain stems from improved representation rather than increased model complexity.
\begin{table}[t!]
\centering
\caption{\textbf{Results on 1D Burgers equation ($s=256$).} ABLE attains the best accuracy across viscosity regimes, with the largest gains in the challenging low-viscosity case ($\nu=10^{-3}$). Best result (lowest error) are \textbf{\underline{underlined}}.} 
\label{tab: 1D Burgers new}
\begin{threeparttable}
\begin{tabular}{lccccc}
\toprule
\multirow{2}{*}{Baseline} & \multirow{2}{*}{Training velocity(s/epoch)\tnote{1}} & \multirow{2}{*}{Flops} & \multicolumn{3}{c}{Relative $l_2$ loss} \\
    \cmidrule(lr){4-6}
     & & & $\nu=10^{-1}$ & $\nu=10^{-2}$ & $\nu=10^{-3}$ \\
\midrule
DeepONet~\citep{lu2019deeponet} & $1.59 \pm 0.02$ & 17.01M & 2.35e-2 & 8.49e-3 & 2.02e-1  
\\
SNO~\citep{Fanaskov2022SpectralNO} &  $(8.75\pm 0.06)\times10^{-2}$ & 1.48M & 3.48e-2 & 1.96e-2 & 1.96e-1
\\
WNO~\citep{Tripura2022WaveletNO} & $4.04\pm0.21$ & 6.36M  & 1.00e-1 & 5.49e-2 & 1.95e-1
\\
Galerkin~\citep{Cao2021ChooseAT} & $2.35\pm0.13$ & 78.78M & 6.71e-4 & 7.10e-4 & 1.26e-2
\\
FNO~\citep{DBLP:journals/corr/abs-2010-08895} & $0.34\pm0.02$ & 6.36M & 6.43e-4 & 5.96e-4 & 6.12e-3
\\
GF-NO~\citep{QI2024106239} & $0.64\pm0.03$ & 6.36M & {4.80e-4} & {5.62e-4} & {5.07e-3}
\\
ABLE~(Ours) & $0.81\pm0.04$ & 12.17M & \uline{\textbf{3.90e-4}} & \uline{\textbf{3.82e-4}}& \uline{\textbf{4.65e-3}}
\\
\bottomrule
\end{tabular}
\begin{tablenotes}
\small
\item[1] Equivalent training velocity reported, full settings in Appendix~C Tab.~\ref{tab: Appendix 1D Burgers}.
\end{tablenotes}
\end{threeparttable}
\end{table}
\begin{table}[t!]
\centering
\caption{\textbf{Results on 2D Darcy flow ($s=141\times141$).} ABLE improves standard spectral operators, and further enhances hybrid architectures (U-FNO $\rightarrow$ U-ABLE), achieving the best overall accuracy. Best result (lowest error) are \textbf{\underline{underlined}}.} 
\label{tab: 2D Darcy Flow}
\begin{tabular}{lcccc}
\toprule
Baseline & Training velocity(s/epoch) & Flops & Best learning rate & Relative l2 loss \\
\midrule
UNet~\citep{DBLP:journals/corr/RonnebergerFB15} & $1.00\pm 0.14$ & 2.99G & 2e-3 & 0.00852
\\
FNO~\citep{DBLP:journals/corr/abs-2010-08895} & $1.17\pm 0.03$ & 0.18G & 6e-3 & 0.00626
\\
WNO~\citep{Tripura2022WaveletNO} & $4.20\pm0.17$ & 0.58G & 3e-3 & 0.0285
\\
GF-NO~\citep{QI2024106239} & $1.53\pm0.43$ & 0.18G & 3.6e-2(1000) & 0.00701 
\\
LSM~\citep{Wu2023SolvingHP} & $4.56\pm0.02$ & 3.76G & 1e-3 & 0.00684
\\
AFNO~\citep{Guibas:2021} & $6.85\pm 0.02$ & 1.48G & 4e-3 & 0.0336
\\
FreqMoE~\citep{Chen2025FreqMoEDF} & $5.48\pm0.07$ & 0.18G & 5e-3 (2stage) & 0.00615
\\
ABLE(Ours) & $2.22\pm0.05$ & 0.25G & 5e-3 &  \uline{\textbf{0.00552}} 
\\ 
\midrule
SAOT~\citep{Zhou2025SAOTAE} & $39.70\pm0.06$ & 28.7G & 1e-3 & 0.00469
\\
U-FNO~\citep{Wen2021UFNOA} & $2.61\pm0.03$ & 6.11G & 3e-3 & 0.00445
\\
ABLE~(Ours)-SAOT & $57.57\pm 0.22$ & 29.3G & 1e-3 & 0.00445
\\
U-ABLE~(Ours) & $4.49\pm 0.02$ & 6.19G & 2e-3 &  \uline{\textbf{0.00394}} 
\\
\bottomrule
\end{tabular}
\end{table}

\begin{table}[t!]
\centering
\caption{\textbf{Results on 2D Navier--Stokes ($\nu=10^{-5}$).}
ABLE improves spectral baselines and further enhances hybrid architectures, achieving the best accuracy. Best result (lowest error) are \textbf{\underline{underlined}}.}
\label{tab: 2D Navier Stokes}
\begin{threeparttable}
\resizebox{\textwidth}{!}{
\begin{tabular}{lcccc}
    \toprule
    Baseline & Training velocity(s/epoch) & Flops(/step) & Best learning rate & Relative $l_2$ loss \\
    \midrule
    WNO~\citep{Tripura2022WaveletNO} & $11.54\pm 0.13$ & 26.51M & 1e-4(800)\tnote{1} & 0.3371 \\
    GF-NO~\citep{QI2024106239} & $6.93\pm 0.23$ & 18.55M & 2.2e-2(1000)\tnote{1} & 0.1229 \\
    LSM~\citep{Wu2023SolvingHP} & $10.41\pm 0.31$ & 744.58M & 1e-3 & 0.1635 \\
    FreqMOE~\citep{Chen2025FreqMoEDF} & $9.13\pm 0.12$ & 18.55M & 3e-3(2-stage)\tnote{1} & 0.1399 \\
    FNO 2D~\citep{DBLP:journals/corr/abs-2010-08895} & $4.74\pm 0.19$ & 32.47M & 3e-3 & 0.1237 \\
    ABLE 2D~(Ours) & $6.08\pm 0.19$ & 60.49M & 3e-3 & \uline{\textbf{0.0985}} \\
    \midrule
    % \hline
    HPM~\citep{Chen2025FreqMoEDF} & $109.31\pm 0.06$ & 46806.34M & 1e-3 & 0.0767 \\
    SAOT~\citep{Zhou2025SAOTAE} & $295.93\pm 15.16$ & 19043.90M & 1e-3 & 0.0630 \\
    ABLE~(Ours)-HPM\tnote{2} & $194.19\pm 0.21$\tnote{3} & 49104.81M & 1e-3 & 0.0705 \\
    ABLE~(Ours)-SAOT\tnote{2} & $313.34\pm 7.51$ & 19247.14M & 1e-3 & \uline{\textbf{0.0568}} \\
    \bottomrule
\end{tabular}
}
\begin{tablenotes}
\small
\item[1] Baselines use optimal hyperparameters; epoch numbers are given in parentheses.
\item[2] SAOT: FNO $\rightarrow$ ABLE. HPM: fixed $\rightarrow$ adaptive Laplacian basis.
\item[3] Gradient checkpointing used for HPM.
\end{tablenotes}
%}
\end{threeparttable}
\end{table}

% Visualization Burgers
% 4 samples
\begin{figure*} % t!
\centering
\includegraphics[trim = 0 0 0 0,clip,width=1\textwidth]{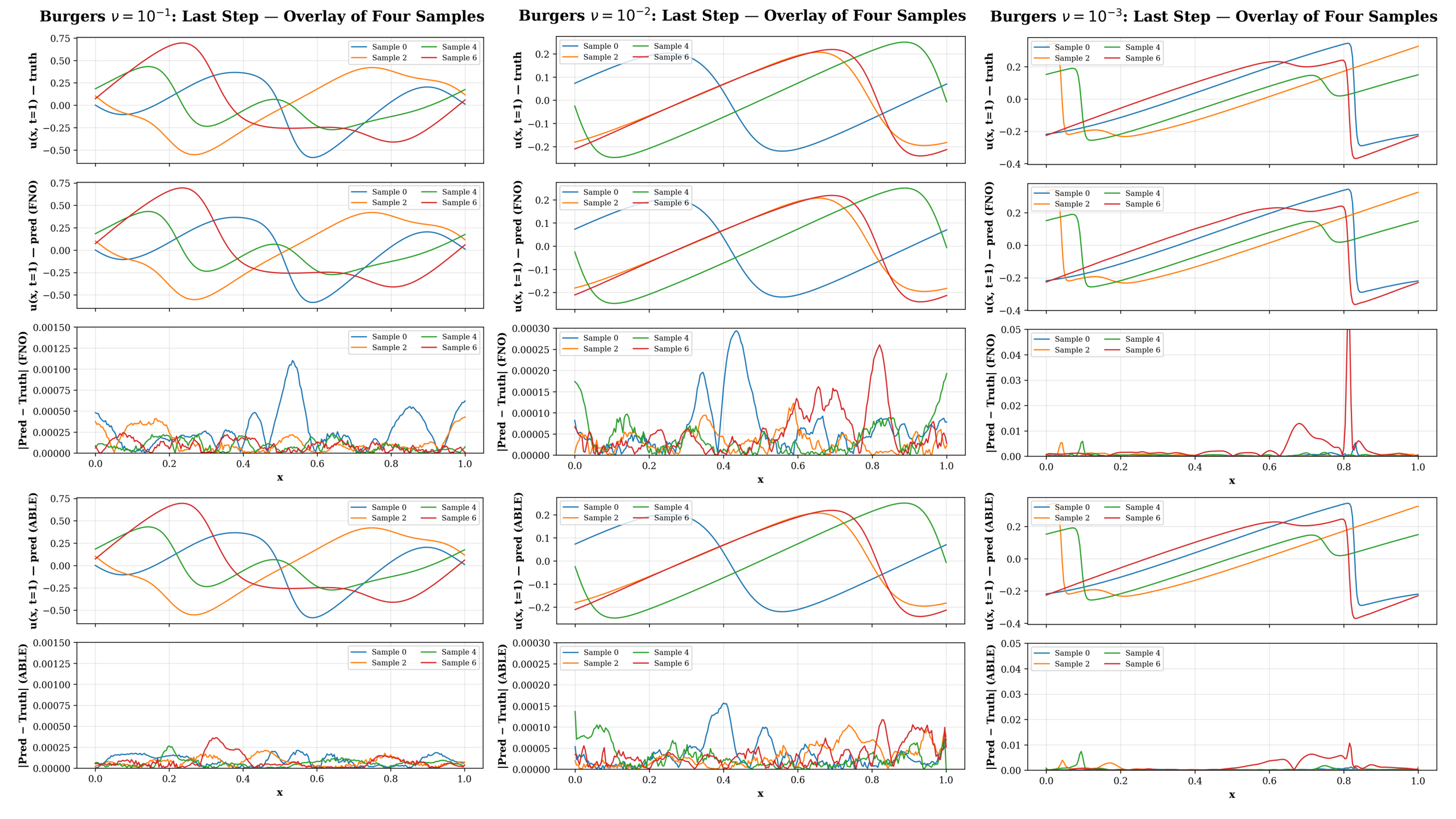}\\
\caption{
\textbf{Qualitative comparison of Burgers' predictions between FNO and ABLE.} From left to right: $\nu=10^{-1}$,$\nu=10^{-2}$,$\nu=10^{-3}$. From the above to the below: Ground Truth, FNO Prediction and FNO Absolute Error, ABLE Prediction and ABLE Absolute Error. ABLE performs much better than FNO at the steep-slope region. Especially under $\nu=10^{-3}$ regime, FNO's prediction is severely inaccurate in the sharp-gradient region, while ABLE greatly alleviated this problem.}\label{fig: Burgers_4_sample}
\end{figure*}

\textbf{Darcy flow.} As shown in Tab.~\ref{tab: 2D Darcy Flow}, ABLE consistently improves over standard spectral operators, reducing the relative $\ell_2$ error of FNO by $\sim11.8\%$ and outperforming GF-NO and FreqMoE. Despite the smooth, low-frequency nature of Darcy flow, this gain indicates that fixed global bases remain suboptimal even in elliptic regimes, and that adapting the representation improves global approximation accuracy. Importantly, the improvement is not limited to standalone models: replacing the spectral branch in U-FNO yields U-ABLE, which further reduces the error by $\sim11.5\%$, achieving the best overall performance. This demonstrates that ABLE acts as a complementary refinement of the spectral component rather than as an alternative architecture, consistently enhancing both classical spectral operators and strong hybrid models.

\textbf{Navier–Stokes.}
As shown in Tab.~\ref{tab: 2D Navier Stokes} and Fig.~\ref{fig:FIG_TEASER}, the low-viscosity regime ($\nu=10^{-5}$) presents a significantly more challenging setting, characterized by long-term prediction, strong nonlinearity, and coupled low- and high-frequency structures. In this regime, ABLE substantially improves over Fourier-based baselines, reducing the error of FNO2D by $\sim20\%$ while maintaining $O(N\log N)$ complexity. The qualitative results in Fig.~\ref{fig:visualisation1} further show that ABLE better preserves sharp gradients and coherent structures, reducing the structured errors observed in fixed-basis models.
Compared to HPM, which performs adaptive spectral selection over a fixed Laplacian basis, the standalone ABLE achieves competitive or superior accuracy with significantly lower computational cost, benefiting from FFT-compatible structure. Importantly, the improvement extends to more expressive models: integrating ABLE into HPM (ABLE-HPM) further reduces the error from $0.0767$ to $0.0705$ ($\sim8.1\%$). Under the same regime, SAOT which combines an FNO branch with MLP kernel and a wavelet branch with QKV-attention kernel, achieves the best accuracy but incurs extremely large computational cost due to the QKV-attention. But by replacing the Fourier basis with the ABLE basis in the FNO branch, the model reduces the error from $0.06295$ to $0.05683$ ($\sim 9.7\%$).
These results highlight that, in highly multiscale and turbulent regimes, the limitation lies not only in operator complexity but in the representation itself. By learning the basis, ABLE provides a systematic improvement that enhances both efficient spectral operators and more expressive latent spectral models.

\textbf{Zero-shot Resolution Extrapolation.}
Even though FNO is less accurate than attention-based models or the advanced spectral neural operator methods, it has the great advantage that it could be trained on low-resolution dataset and directly extrapolate to higher extrapolation to higher resolution in zero-shot. This property holds for the Burgers' benchmarks. In comparison, ABLE could also achieve this same resolution-extrapolation capability on Burgers when the finite difference feature enhancement module is removed, as shown in Fig. \ref{fig:Burgers_zero_shot}. More precisely, we observe that even with this finite differentiation retained, the test error for high and medium viscosity ($\nu=10^{-1},10^{-2}$) shows only modest increase. Yet, in the low-viscosity regime $\nu=10^{-3}$ where shock waves become abundant, the finite difference module is better suited for fitting single-resolution data, but for extrapolation its numerical precision is inherently limited by the discretization resolution. Even so, removing this module restores stability across resolutions and still endows ABLE with an error of $5.2\times 10^{-3}$, achieving a $15.0\%$ error reduction over FNO. Therefore, ABLE manifests favorable extrapolation capabilities, and its introduction of finite difference module is still valuable for fixed-scale training and testing.

\textbf{Long Term Stability.} Moreover, ABLE prediction on Navier Stokes manifests long-term stability. The per-step error and accumulated error for all $200$ testing samples are plotted in Fig.~\ref{fig:step_error}. And the long-term visualization for the $10$-step prediction are shown in Fig.~\ref{fig:NS_long_term_vis}. ABLE not only outperforms FNO at each single time step, but also achieves better performance on accumulated error. Endowing the representation with adaptive learnability does not compromise the long-term stability of spectral representation, and the gains are comprehensive.

\begin{figure*} % t!
  \centering
  \includegraphics[trim = 1cm 11cm 1cm 0.7cm, clip,width=1\textwidth]{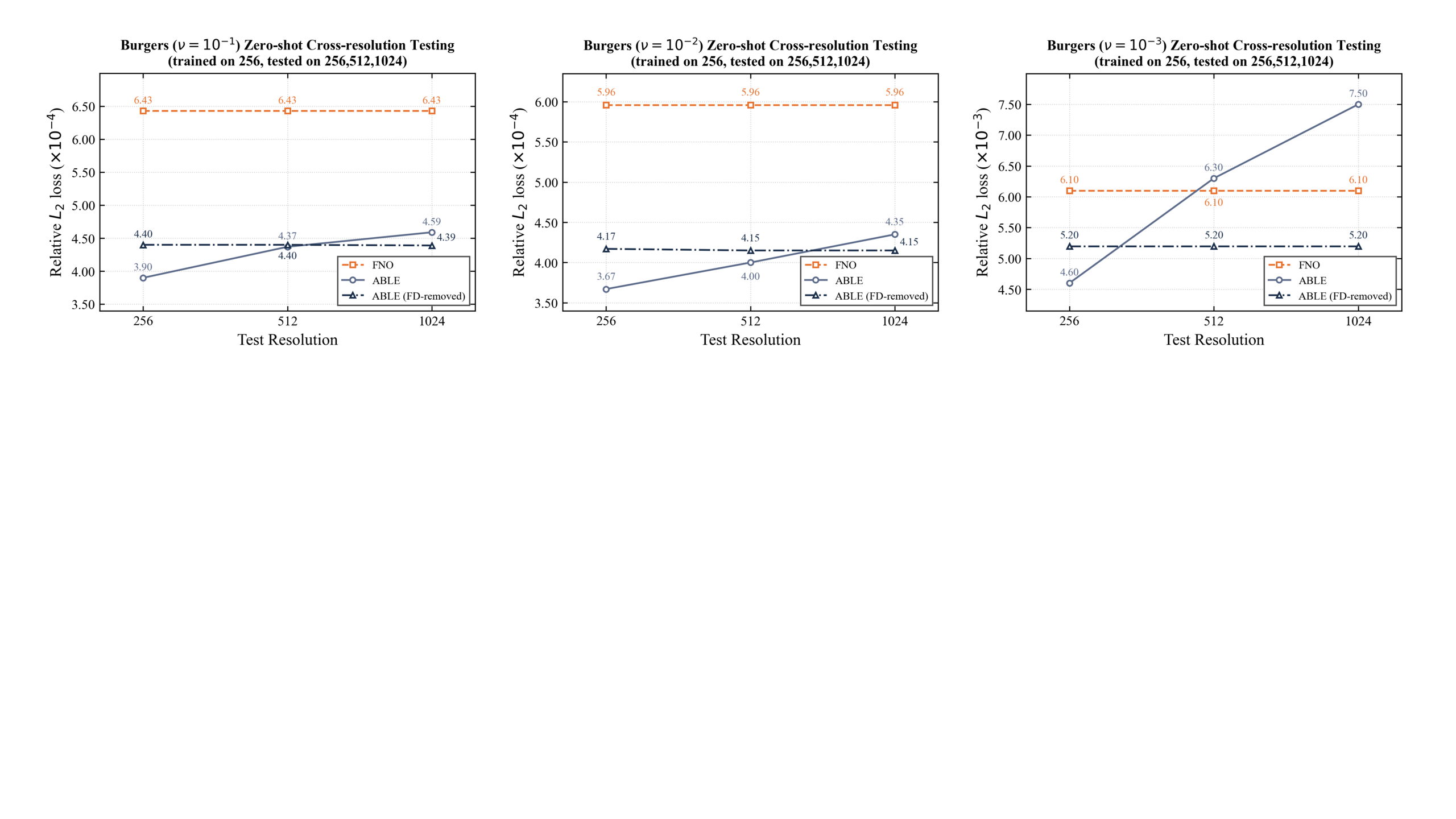}\\
  \caption{
  \textbf{Zero shot resolution extrapolation on Burgers.} From left to right: viscosity $\nu =10^{-1},10^{-2},10^{-3}$ respectively
  }
  \label{fig:Burgers_zero_shot}
\end{figure*}

\begin{figure*} % t!
  \centering
  \includegraphics[trim = 0 13.0cm 0 0, clip,width=1\textwidth]{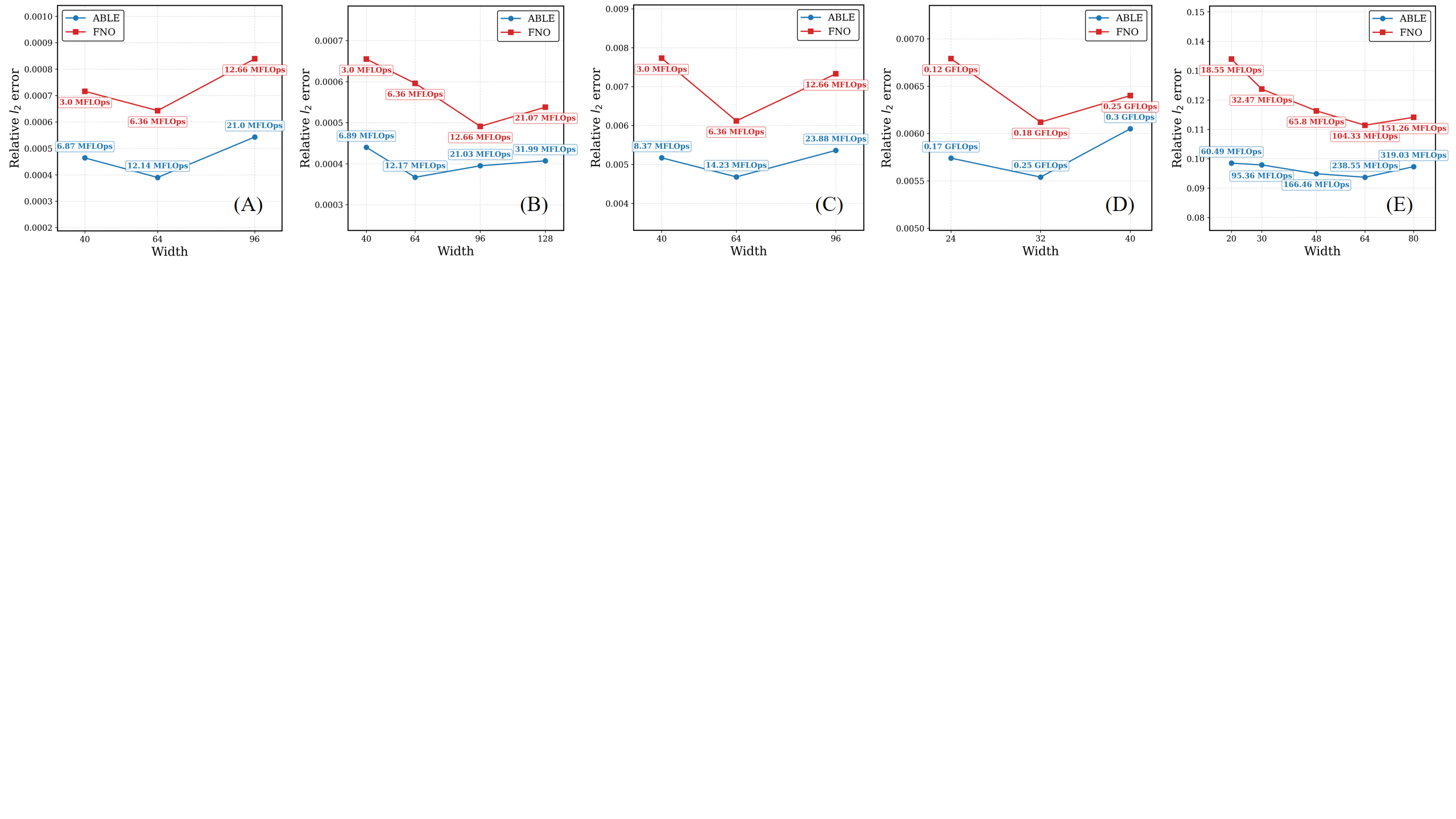}\\
  \caption{
  \textbf{Fair comparison and scaling property.} (A) Burgers $\nu=10^{-1}$ (B) Burgers $\nu=10^{-2}$ (C) Burgers $\nu=10^{-3}$ (D) Darcy (E) Navier-Stokes
  }
  \label{fig:fair_comparison}
\end{figure*}

\begin{figure*} % t!
  \centering
  \includegraphics[width=1\textwidth]{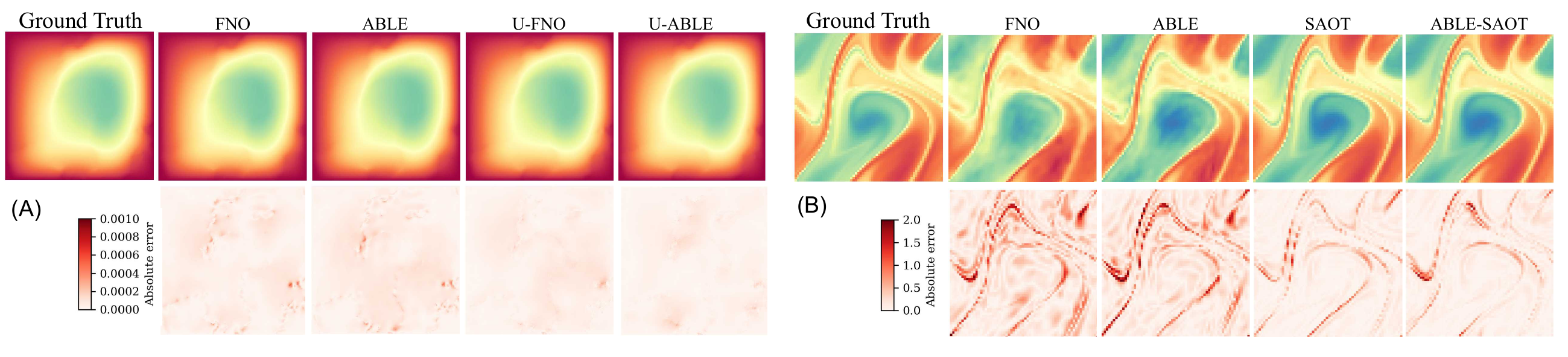}
  % out_best.pdf
  \\
  \caption{
  \textbf{Qualitative comparison.}
(A) Darcy flow: ABLE reduces error and improves U-FNO. 
(B) Navier–Stokes ($\nu=10^{-5}$): ABLE better captures sharp structures; ABLE-SAOT achieves the best reconstruction.
  }
  \label{fig:visualisation1}
\end{figure*}

\begin{figure*} % t!
  \centering
  \includegraphics[width=1\textwidth]{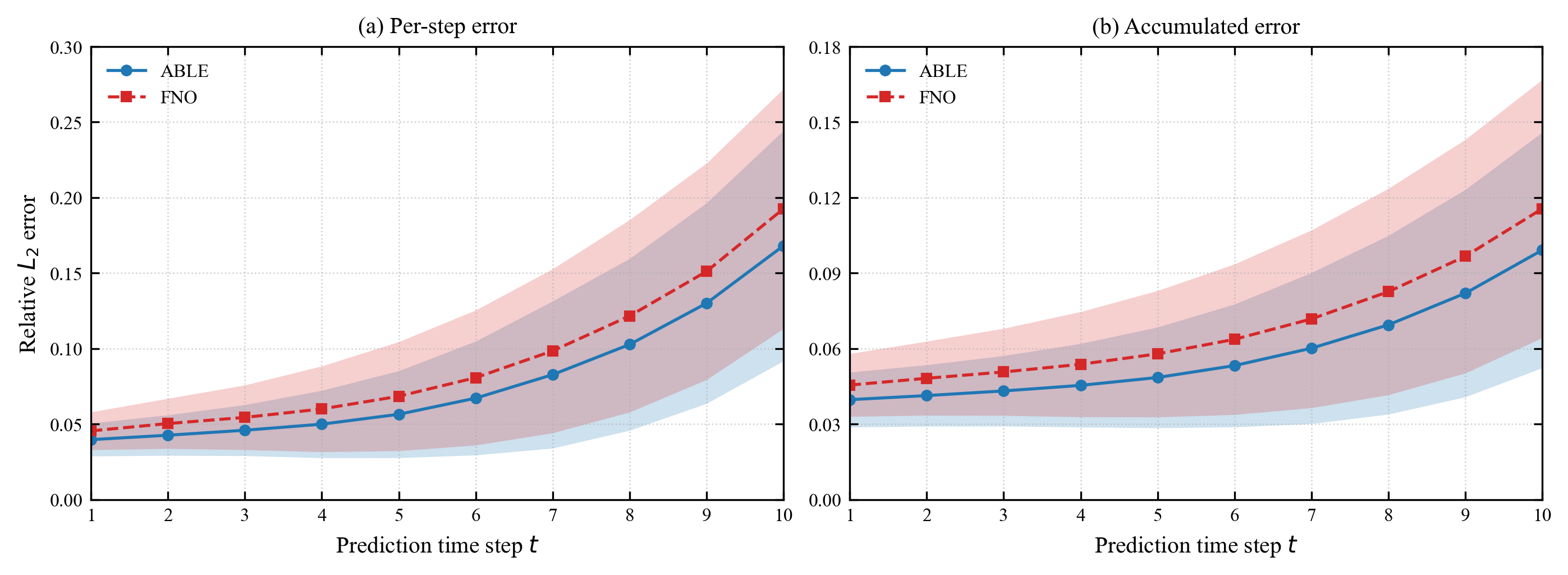}\\
  \caption{
  \textbf{Step error comparison for Navier Stokes.}
  The long-term prediction of solutions to Navier–Stokes ($\nu=10^{-5}$) is notoriously difficult, especially in the presence of turbulence. ABLE not only outperforms FNO on each single-step error, but also robustly outperforms on accumulated error, demonstrating its stability in long-term prediction.
  }
  \label{fig:step_error}
\end{figure*}
\begin{figure*} % t!
  \centering
  \includegraphics[trim= 0 0 1cm 0,clip ,width=1\textwidth]{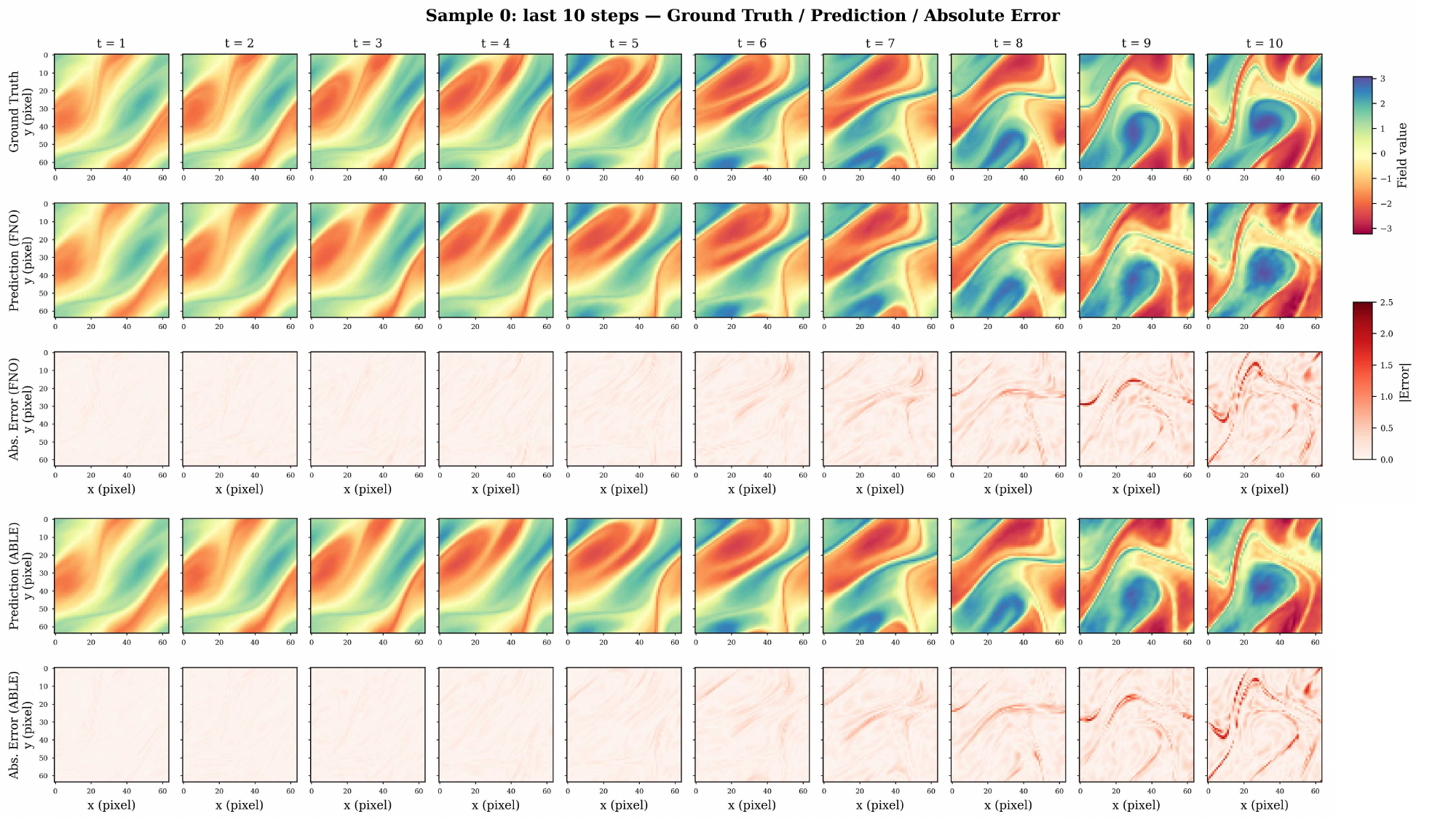}\\
  \caption{
  \textbf{Long term qualitative comparison of Navier Stokes.}  }
  \label{fig:NS_long_term_vis}
\end{figure*}

% 7. added by Xuxiang
\textbf{Fair Comparison and Performance according to Model Scale.} We provide the result of performance under various width configurations of FNO and ABLE for each benchmark (as shown in Fig.~\ref{fig:fair_comparison}). ABLE comprehensively outperform FNO not only at each same width, but the ABLE with optimal width also largely outperforms FNO with optimal width.
More precisely, since ABLE would reduce to an over-parameterized FNO with same width at infinity temperature, and directly expanding width is not a same operation as ABLE extension, it is actually fair to compare models with same width according to expressivity. But even comparing ABLE with a cost-matched but wider FNO, ABLE also achieves out-performance robustly. Simply increasing width in FNO cannot substitute ABLE's contribution in heterogeneous kernel modeling and adaptive basis learning.

For the scaling property of both FNO and ABLE, they both manifest similar diminished return at higher model scale (width), but the prediction of ABLE is much more accurate than that of FNO.

\textbf{Cost-Accuracy Analysis.} In summary, for all benchmarks, the cost-accuracy plot is shown in Fig.\ref{fig:cost_accuracy}. Pure ABLE achieves the SOTA among standalone spectral models. For hybrid structured models like SAOT and U-FNO model, it is unfair to directly compare with pure ABLE, so we induce ABLE extension in their FNO branch, and they globally achieve best performance in their cost regime. Pure ABLE model and hybrid heavy-weighted models with ABLE replacement constructs a better Pareto frontier, which largely promote the previous Pareto frontier with a considerable similar-level cost.
\begin{figure*} % t!
  \centering
  \includegraphics[trim = 1cm 4.2cm 1cm 3.4cm,clip,width=1\textwidth]{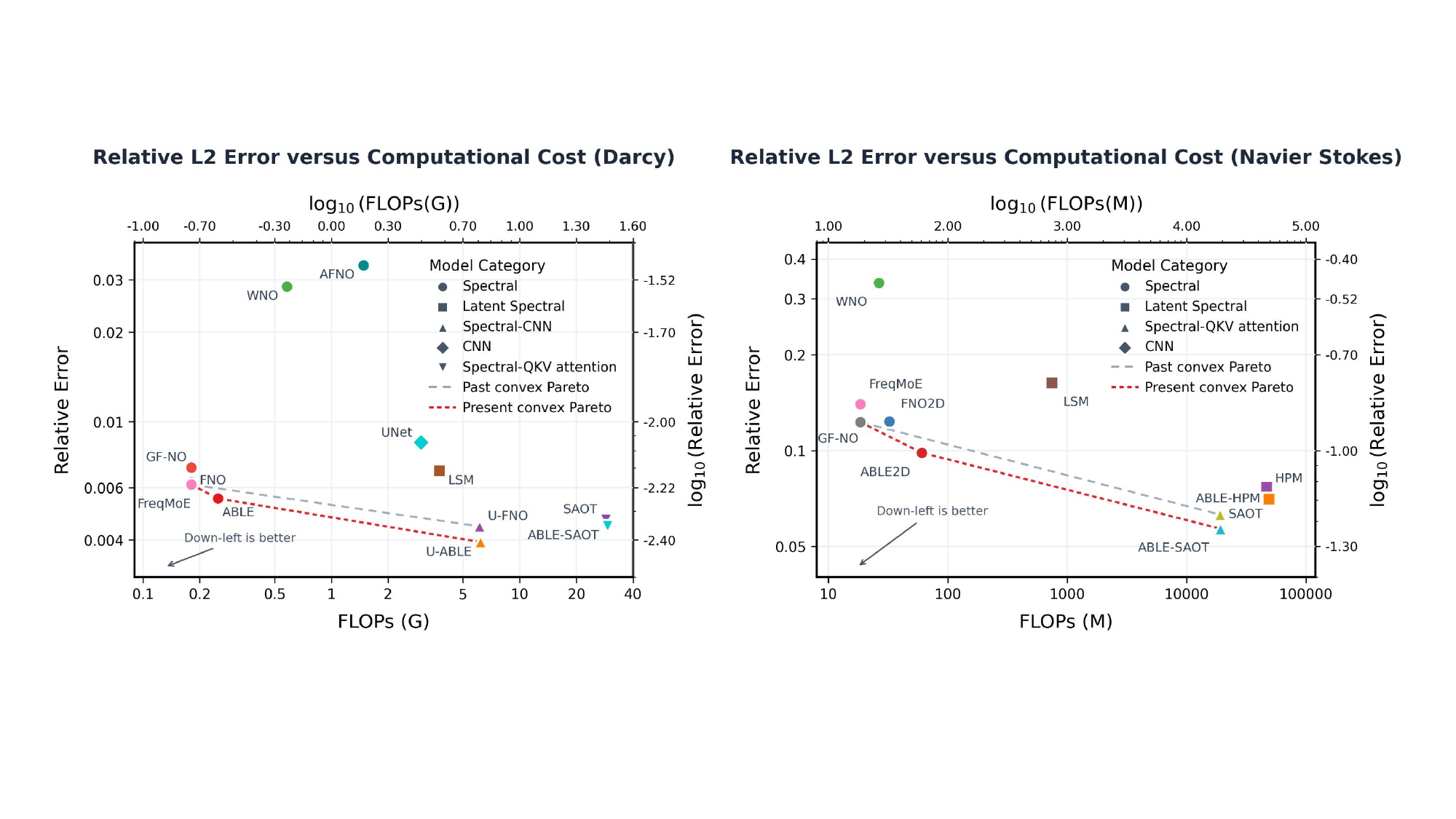}\\
  \caption{
  \textbf{Cost-accuracy analysis.} Left: Darcy flow Right: Navier–Stokes ($\nu=10^{-5}$). Standalone ABLE model achieves a robust improvement for spectral neural operators. And for high-weighted hybrid structure models like U-FNO, HPM and SAOT, by extending their fixed bases to adaptive learnable bases, they all achieve better performance in that cost level. The convex cost-accuracy Pareto frontier is largely promoted.
  }
  \label{fig:cost_accuracy}
\end{figure*}

\subsection{Spectrum Visualization: Rank Recovery through ABLE}
Among all the experiments, the ABLE model encodes the same mode truncation $K^d$ as the FNO model. But the $K^d$ truncation is done in the ABLE's representation system instead of Fourier representation. To compare their filtering behavior, we respectively plot the Fourier spectrum of both FNO's and ABLE's prediction for both Darcy and Navier Stokes in Fig.~\ref{fig:visualisation_freq}. 
\begin{figure*}
  \centering
  \includegraphics[trim = 0.7cm 4.2cm 0.7cm 3.4cm,clip,width=1.0\textwidth]{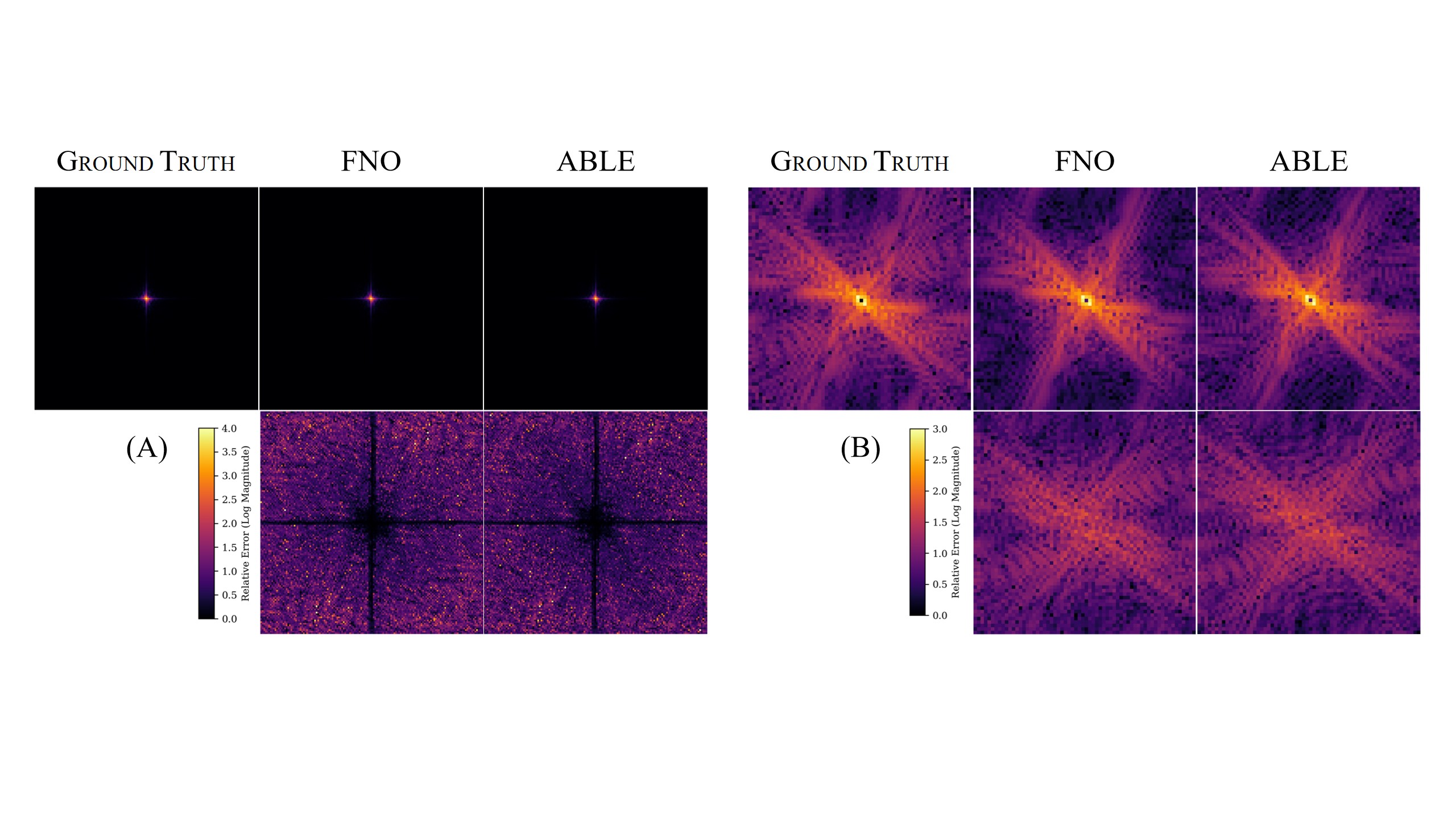}\\
  \caption{
  \textbf{Spectrum comparison.} (A) Darcy: The ground truth mainly locates at low frequency, the relative error plot shows FNO has higher luminosity (error) at the corner, its hard truncation wipe out the high frequency information. With equal truncation $K^d$, ABLE relieves the problem by the learnability of its basis.
  (B) Navier–Stokes ($\nu=10^{-5}$): FNO locates on a stellar shaped low-frequency region due to truncation. ABLE better captures high-frequency corners via learnable basis.
  }
  \label{fig:visualisation_freq}
\end{figure*}
We found that the FNO prediction is with stellar shape and locates in the central low-frequency region, the truncation in the fixed Fourier basis system would impair high-frequency information, which is especially severe in Navier Stokes, and endow FNO with low-pass preference. Compared with FNO, the ABLE prediction now could simulate the corner high-frequency region better, indicating that with the construction of adaptive learnable basis $e_{k,m}$, the first $K^d$ modes in ABLE representation now could recover both the high and low frequency information simultaneously, and this relieves the low-pass preference. 

Moreover, this spectrum visualization result aligns with what we mentioned in theorem~\ref{thm: main_rank_recovering}: the ABLE extension preserves self-adjointness, and the extended basis expands to $K^dM$ components, allowing the rank to be recovered from $K^d$ to $K^d M$, and more high-frequency information to penetrate. 
Notably, $K$ is already chosen as the optimal truncation for FNO in all experiments, and increasing $K^d$ to $K^d M$ in the Fourier representation cannot further improve the result, suggesting that purely increasing rank within a fixed representation still suffers from its data-agnostic nature. In comparison, ABLE provides a robust method to recover the rank. Furthermore, the non-trivial, learnable and data-adaptive rank recovery—unique to ABLE—becomes critical.

% 10. Added by Xuxiang
\begin{figure*} % t!
  \centering
  \includegraphics[width=1\textwidth]{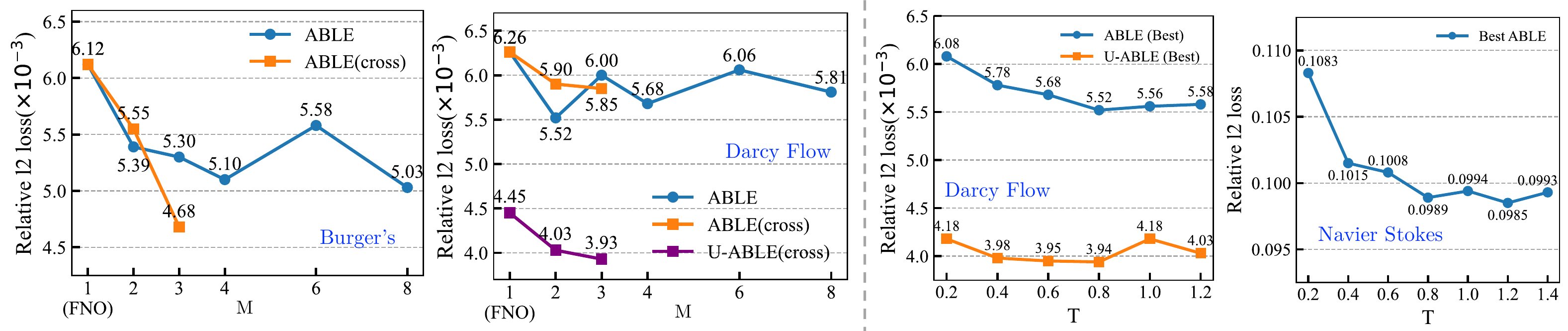}\\
  \caption{
  \textbf{Ablation on basis size $M$ and temperature $T$.}
Performance improves with increasing $M$ up to a moderate value, while optimal accuracy is achieved at intermediate temperatures, highlighting the need for balanced basis capacity and adaptivity.
  }
  \label{fig:ablation}
\end{figure*}

\subsection{Ablation Study: Emergent Scale Diversity and Heterogeneity}
As we mentioned in section \ref{section: temperature controlled heterogeneity}, the basis size $M$ endows the expressivity and heterogeneity of ABLE representation, and it contributes to the diversity of scales. Additionally, the temperature $T$ in the Softmax operation regulates the sharpness of basis functions via $p(x,m)$ and further serves as a control of heterogeneity across ABLE basis components. The optimal basis size $M$ and temperature $T$ are different for each problem, but empirical results point to a collective phenomenon that ABLE models perform best with basis size $M\in [2,8]$ which is small but depends on the ill-posedness of the problem, and a moderate temperature $T\in [0.2,1.2]$. This is evidenced by the following ablation studies.
\begin{figure*} % t!
  \centering
  \includegraphics[trim = 1cm 1.8cm 1cm 0.7cm, clip, width=1\textwidth]{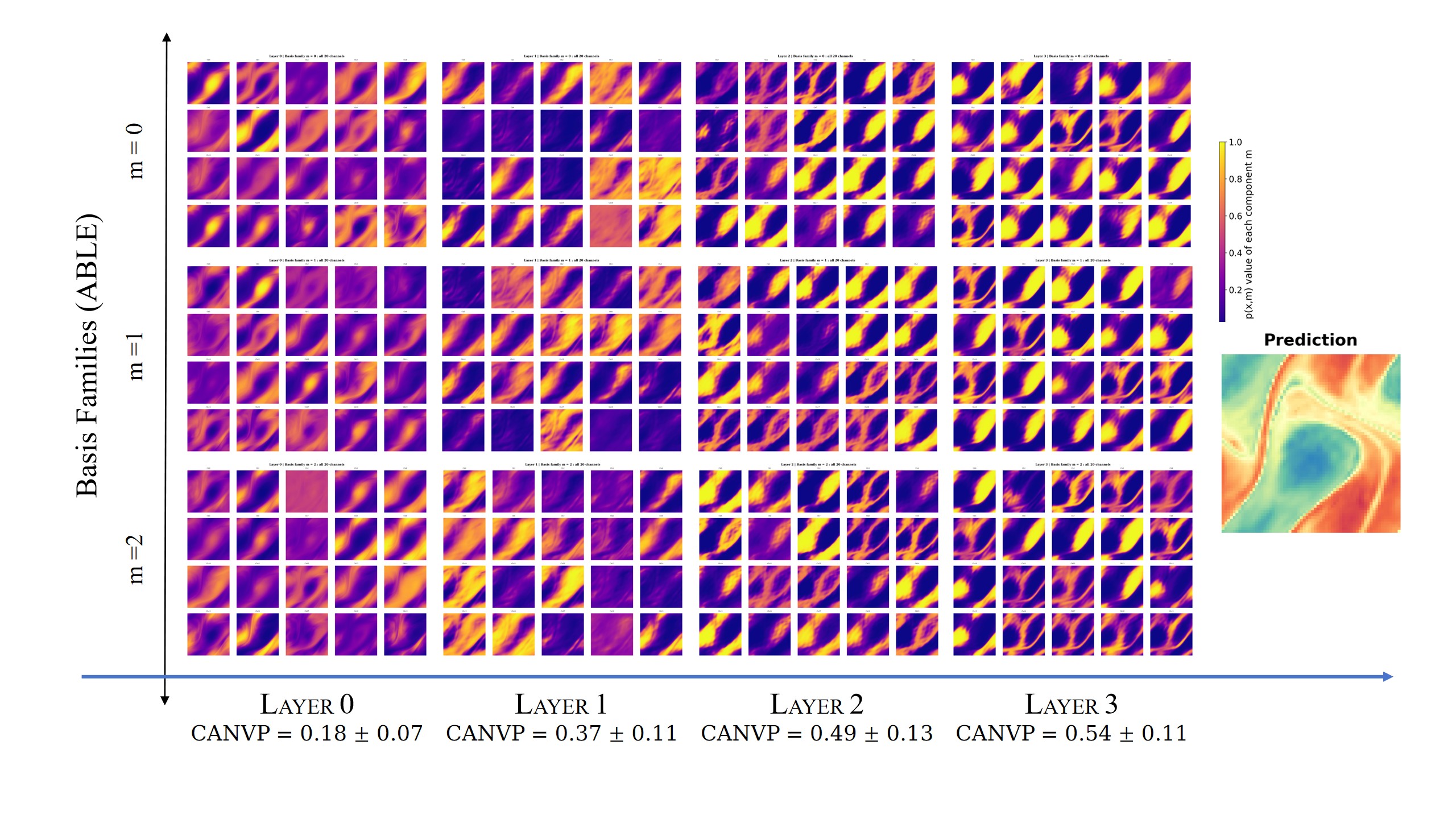}\\
  \caption{
  \textbf{Heterogeneity for each layer (4 in total, 20 channels per layer).} For NS, each channel of each layer in ABLE has its own 3 basis components $p(x,m),m=0,1,2$ and sum up to $1$. Channel Averaged $\text{NVar}p$ (CANVP) is counted per layer, to quantify the heterogeneity of this layer, which takes value from $0$ (homogeneous) to $1$ (totally heterogeneous). Heterogeneity increases from the most inward layer to the most outwards, highlighting that the representation is gradually locally frozen, and heterogeneity gradually emerges.}
  \label{fig:Heterogeneity_visualization}
\end{figure*}
\begin{figure*} % t!
  \centering
  \includegraphics[trim = 1cm 0.9cm 1cm 9.7cm,clip, width=1\textwidth]{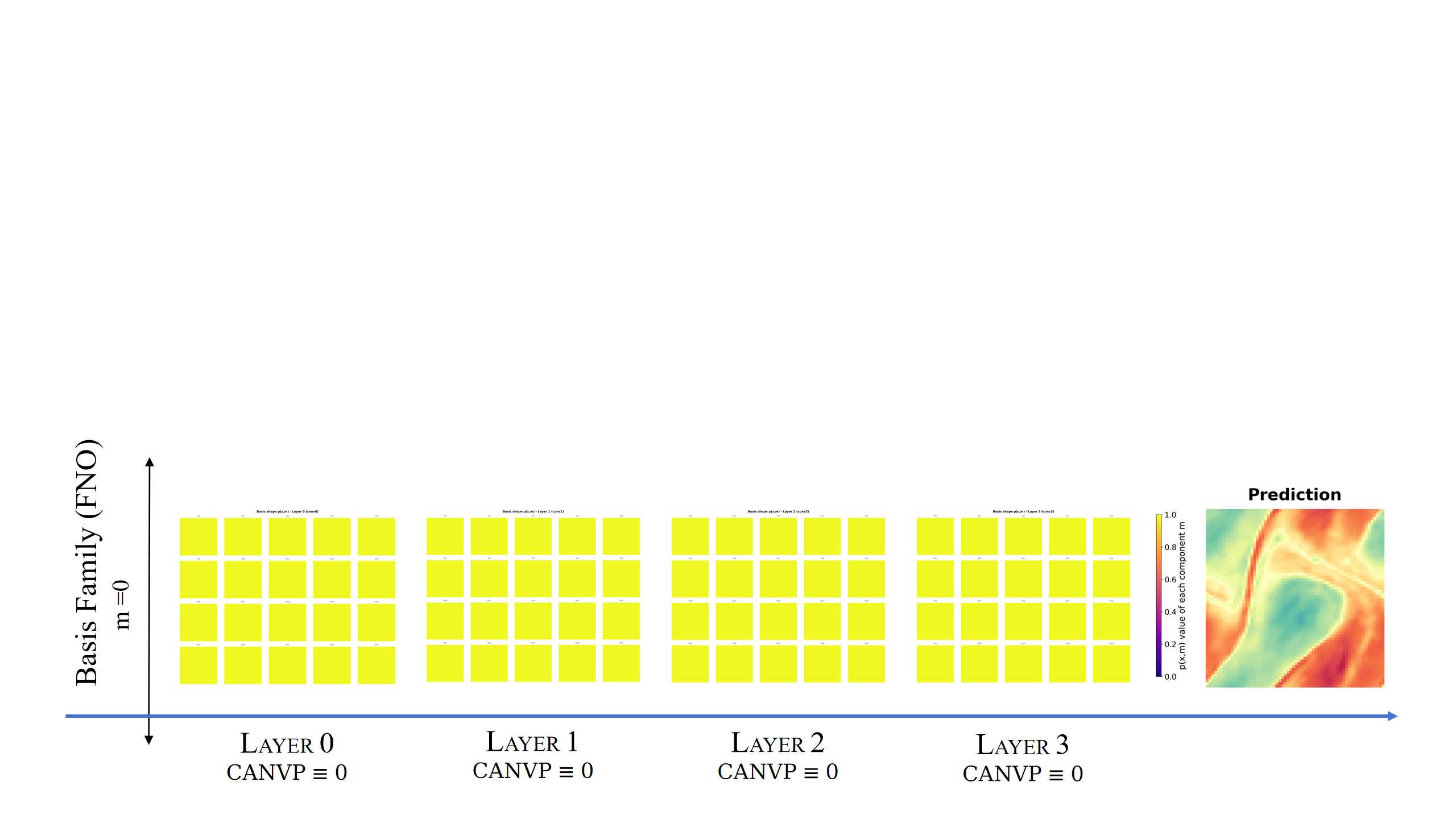}\\
  \caption{
  \textbf{FNO as a trivial homogeneous case of ABLE (4 layers in total, 20 channels per layer).} FNO is exactly ABLE with only $M=1$ basis components, and ancillary function $p(x,m)\equiv p(x,0)\equiv 1$. Thus ABLE basis degenerates to the Fourier basis. The Channel Averaged $\text{NVar}p$ (CANVP) takes value of $0$ (perfectly homogeneous). The spectral representation is prefixed as Fourier representation and cannot capture and manifest extra heterogeneity like ABLE does via learnable $p(x,m)$.}
  \label{fig:FNO_Hetero_visualization}
\end{figure*}

\begin{figure*}
\centering
\includegraphics[width=\linewidth]{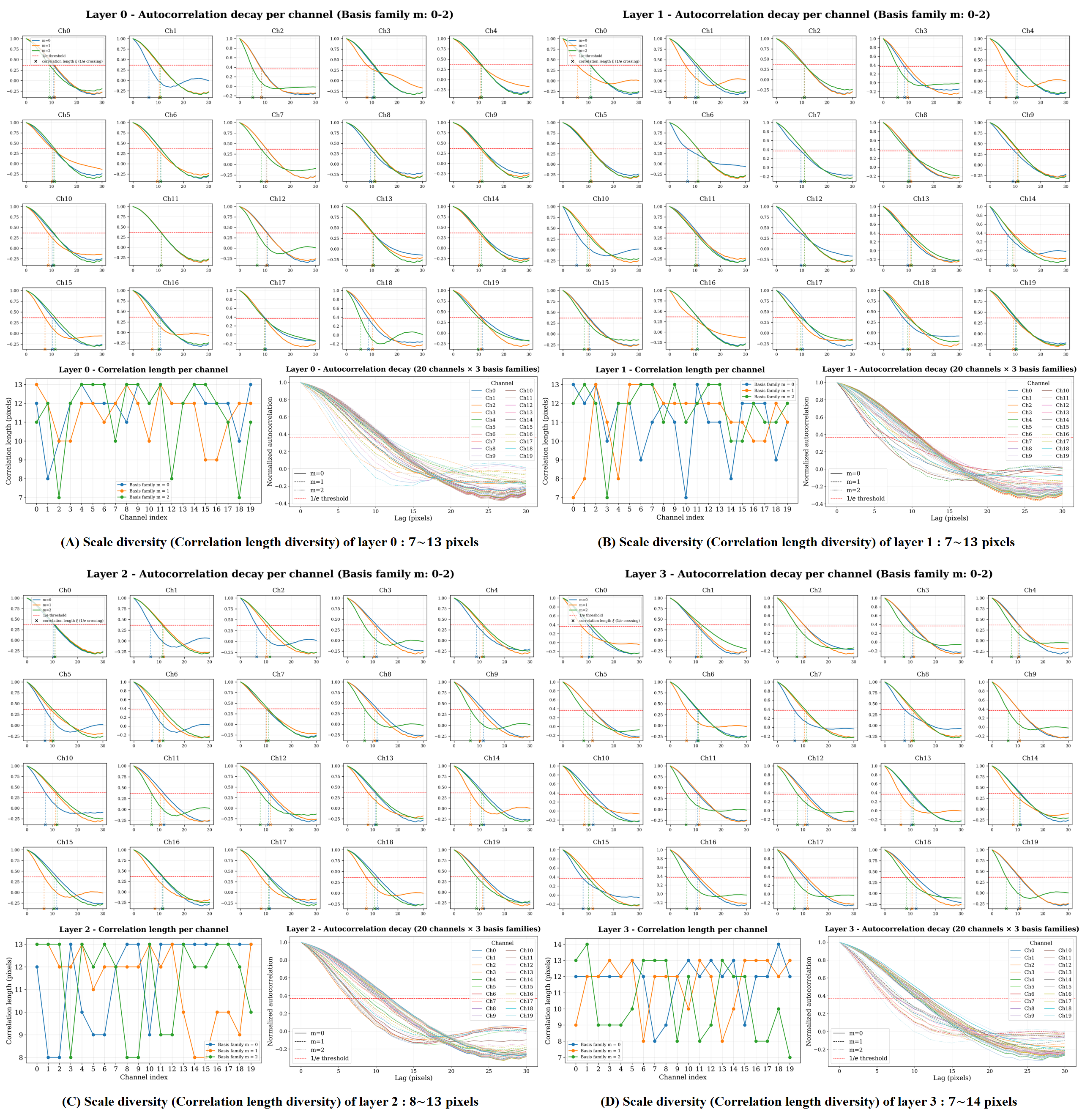}
\\
\caption{
  \textbf{Scale diversity for each layer (4 in total, 20 channels per layer).} For NS, each channel of each layer in ABLE has its own 3 basis components $p(x,m),m=0,1,2$ modeled by MLPs. The correlation of each component MLP $\langle p(x,m)p(0,m)\rangle=G(x,m)\sim e^{-|x|/\xi_m}$ roughly follows the exponential decay but manifests different correlation length $\xi_m$. This implies even without explicit multi-scale regularization, the scale diversity is permitted in ABLE's adaptive basis construction and exactly emerges in its learning procedure.
  }\label{fig:scale_diveristy_visualization}
\end{figure*}

\textbf{Ablation on basis size and temperature.}
Fig.~\ref{fig:ablation} studies the effect of the basis size $M$ and temperature $T$. Increasing $M$ initially improves performance, confirming that a richer adaptive basis enhances expressivity; however, beyond a moderate size the gain saturates or degrades, indicating that excessive basis capacity can lead to overfitting or redundant representations. This behavior is consistent across both Burgers and Darcy, with hybrid models (U-ABLE) benefiting more from larger $M$, suggesting that spatial regularization stabilizes higher-capacity bases.  

Furthermore, the optimal basis capacity depends on two factors, the ill-posedness and multi-scale nature of the problem, and the model choice of whether cross-interactions between the $M$ families are introduced in the kernel $R(k;m,m')$. For Burgers, as the viscosity reduces from $10^{-1}$ to $10^{-3}$, the shock wave becomes more abundant and severer, and the ground truth becomes more ill-posed to simulate, suggesting the demand for larger optimal $M$ to capture local information. This is also enlightened by the BV-class approximation theorem \ref{thm15: approximation} in appendix~A which articulates ABLE's ability to simulate ill-posed BV-class functions increases with basis size $M$. 
While for the Darcy benchmark, it's a steady prediction without shock wave, but still with heterogeneity in the boundary condition, and we observe ABLE with $M=2$ performs the best. 

Second, from the model side, an ABLE model with interactive kernel $R(k,m,m')$ is more expressive than those without it; thus, it achieves the optimal performance with smaller $M$. We observe for Burgers and Navier-Stokes, an ABLE-cross model with $M=3$ performs the best. Moreover, in these two benchmarks, there exists energy shifting between different scales, this inclines their best model identification towards ABLE with cross interaction, since currently different $M$ components could interact with and transfer among each other.

The temperature $T$ controls the sharpness of the basis allocation. As shown in temperature ablations for Darcy and Navier–Stokes in Fig.~\ref{fig:ablation}, intermediate values of $T$ achieve the best performance, while too small or too large values degrade accuracy. 
From the Lipschitz control in Eqn.~\ref{eq:Lipshitz control}, high temperature makes $M$ basis components nearly identical, encourages overfitting to homogeneous and global pattern. In comparison, zero temperature would strictly force each point into exactly one group, which causes overfitting to local phenomenon and impairs generalization. 
This indicates that neither fully deterministic nor fully uniform basis selection is optimal; instead, a balanced, softly adaptive allocation is required. 

Overall, these results highlight that ABLE benefits from moderate basis complexity and controlled adaptivity, supporting its role as a flexible yet stable representation mechanism.

\textbf{Visualization for heterogeneity and scale diversity}
We plot the basis weights $p(x,m)$ of ABLE representation for samples from Navier-Stokes problem, which exhibits the richest multi-scale information (see Fig.~\ref{fig:Heterogeneity_visualization}). We observe that each layer among all $4$ layers manifests the heterogeneity across $M$. More specifically, the heterogeneity gradually emerges from the innermost layer to the outermost, which could be evaluated through CANVP (channel averaged normalized variance $\textbf{Nvar}p$). 

The Navier–Stokes problem is inherently heterogeneous: both initial conditions and solution fields vary sharply across scales, implying strong heterogeneity in the target operator itself. FNO can only address such operator heterogeneity through point-wise nonlinearities (e.g., GeLU), relying on universal approximation rather than structural adaptivity. In comparison, by partially encoding the heterogeneity in the learnable basis itself, ABLE gradually accumulates the heterogeneity across $4$ layers, the representation is gradually frozen to a data-adaptive regime.

More surprisingly, as we plot the spatial correlation function of the basis shape of each channel in Fig.~\ref{fig:scale_diveristy_visualization}, the scale diversity is observed through the diversity of correlation length $\xi_m$. For each channel, the correlation lengths $\xi_m$ are diversified across $M=3$ components, from minimum $7$ pixels to maximum $14$ pixels. 
There is no explicit scale information pre-used in the design of ABLE basis, each components $p(x,m)$ are with same Softmax-MLP structure and parameter volume, but the MLP in different component $m$ still identifies diversified scales. This enlightens us that ABLE's heterogeneously adaptively learned basis further decodes scale information, and this identification is built on the universal approximation of MLP to preferable basis functions. 
Unlike models with explicit scale separation, ABLE alternatively tackles multi-scale problems with learnable MLP which jointly learns a representation adaptive to the underlying problem and the actual scale diversity. And this extension is widely-admissible, which generally opens a possible direction for data-adaptive improvement for spectral representations. 

\section{Conclusion}
We introduced ABLE, a framework that replaces fixed spectral bases with adaptive, learnable ones for neural operator learning. By shifting expressivity from the operator to the representation, ABLE improves both standard spectral models and hybrid architectures while retaining $O(N\log N)$ efficiency. The learned basis forms a Parseval frame, ensuring stable and information-preserving transformations, and strictly generalizes classical Fourier-based operators. 
Empirically, ABLE consistently improves accuracy across PDE regimes, from smooth elliptic problems to highly multiscale turbulent dynamics, and integrates seamlessly with existing architectures. These results suggest that the choice of representation, rather than operator complexity alone, is a key bottleneck in neural operator design. 
Beyond PDEs, the idea of learning the basis opens a new direction for representation learning in structured domains, where adaptive and data-dependent decompositions may offer significant advantages.

\newpage
\appendix

\addtocontents{toc}{\protect\setcounter{tocdepth}{2}}
\begin{center}
\textbf{\Large Don't Fix the Basis -- Learn It: Spectral Representation with Adaptive Basis Learning for PDEs (Appendices)}
\end{center}
\label{sec:appendix}

\noindent\rule{\textwidth}{0.6pt}
\tableofcontents
\noindent\rule{\textwidth}{0.6pt}

\section{Theorem and proof}
\begin{figure} % [htbp]
    	\centering
    	\resizebox{\textwidth}{!}{%
    		\begin{tikzpicture}[
    			module/.style={draw, rounded corners=6pt, thick, fill=blue!5,
    				minimum width=3.4cm, minimum height=0.7cm,
    				font=\small\bfseries},
    			thm/.style={draw, rounded corners=3pt, fill=white,
    				minimum width=3.2cm, font=\footnotesize,
    				align=center},
    			arr/.style={->, thick, >=stealth},
    			cross/.style={->, thick, dashed, >=stealth, red!60!black},
    			label/.style={font=\tiny\itshape, red!60!black, align=center}
    			]
    			
    			% ==================== ROW 1: I & II ====================
    			\node[module] (L1) at (0.5, 0) {I. Fourier basis};
    			\node[thm] (t1) at (-1.3, -1.2) {\footnotesize Thm 4: \\ Orthonormal basis};
    			
    			\node[thm] (p1) at (-1.3, -2.4) {\footnotesize Thm 5 (Plancherel): \\ Isometry};
    			\draw[arr] (t1) -- (p1);
    			
    			\node[thm] (k1) at (2.3, -1.2) {\footnotesize Prop 5.2: \\ Generalized Fourier};
    			\draw[arr] (t1) -- (k1);
    			
    			\node[thm] (e1) at (2.3, -2.4) {\footnotesize Prop 5.1: \\ Comm.\ diagram};
    			\draw[arr] (p1) -- (e1);
    			
    			\node[thm] (d0) at (0.5, -3.5) {\footnotesize Fourier transformation};
    			\draw[arr] (d0) -- (p1);
    			
    			\node[module] (L2) at (10.0, 0) {II. ABLE frame};
    			\node[thm] (t3) at (8.2, -1.2) {\footnotesize Thm 6: \\ Parseval tight frame};
    			\node[thm] (t4) at (8.2, -2.4) {\footnotesize Thm 7 (Thm 1): \\ ABLE isometry \& Parseval};
    			\draw[arr] (t3) -- (t4);
    			
    			\node[thm] (k2) at (11.8, -1.2) {\footnotesize Prop 7.2: \\ Generalized ABLE};
    			\draw[arr] (t3) -- (k2);
    			
    			\node[thm] (e2) at (11.8, -2.4) {\footnotesize Prop 7.1: \\ Comm.\ diagram};
    			\draw[arr] (t4) -- (e2);
    			
    			\node[thm] (d1) at (10.0, -3.5) {\footnotesize Def 2: ABLE transformation};
    			
    			\draw[arr] (d1) -- (t4);
    			\draw[cross] (d0) to node[pos=0.5, above, label] {\footnotesize extension} (d1);

    			\node[thm] (ABLE) at (5.25, 0) {\footnotesize Def 1: \\ Adaptive learnable basis};
    			\draw[cross] (L1) -- (ABLE);
    			\draw[cross] (ABLE) -- (L2);
    			\draw[arr] (ABLE) -- (t3);
    			
    			% ==================== MIDDLE ROW: Model Entities ====================
    			\node[module] (L3) at (5.25, -4.5) {III. Representation ability};
    			\node[draw, thick, fill=gray!15, rounded corners=8pt,
    			minimum width=2.5cm, minimum height=1.2cm,
    			font=\normalsize\bfseries, align=center] (model_fno) at (0.5, -4.8) {FNO};
    			
    			\node[thm] (uat_fno) at (0.5, -6.4)
    			{\footnotesize Universal\\ approximation theorem}; 
    			\draw[arr] (model_fno) -- (uat_fno);
    			
    			\node[thm] (thm8_mid) at (5.25, -5.6)
    			{\footnotesize Thm 8 (Thm 2): Super set\\ ABLE $\supset$ FNO};
    			
    			\node[draw, thick, fill=gray!15, rounded corners=8pt,
    			minimum width=2.5cm, minimum height=1.2cm,
    			font=\normalsize\bfseries, align=center] (model_able) at (10.0, -4.8) {ABLE-NO};
    			
    			\node[thm] (uat_able) at (10.0, -6.4) 
    			{\footnotesize Thm 9:\\ Universal approximation};
    			\draw[arr] (model_able) -- (uat_able);
    			
    			% ---- construction
    			\draw[cross] (d0) to node[pos=0.5, above, label] {\footnotesize} (model_fno);
    			\draw[cross] (d1) to node[pos=0.5, above, label] {\footnotesize} (model_able);

    			\draw[arr] (model_fno.east) -- ++(0.3, 0) |- (thm8_mid.west);
    			\draw[arr] (thm8_mid.east) -- ++(0.6, 0) |- (model_able.west);
    			\draw[arr] (uat_fno.east) -- (uat_able.west);
    			\draw[arr] (thm8_mid.south) -- ++(0, -0.3) |- (uat_able.west);
    			
    			\node[thm, minimum width=1.0] (t52_mid) at (11.0, -7.8) {\footnotesize Thm 10:\\ Inherent\\ non-linearity};
    			\draw[arr] (uat_able) -- (t52_mid);
    			\node[thm, minimum width=1.0] (t50_mid) at (9.0,-7.8)
    			{\footnotesize Def 3: \\ Softmax-MLP\\ model};
    			\draw[arr] (uat_able) -- (t50_mid);
    			
    			\node[thm, minimum width=2.3cm] %, font=\tiny] 
    			(param_M) at (5.25, -7.4) { Basis size: $M$};
    			\node[thm, minimum width=2.3cm] %, font=\tiny] 
    			(param_T) at (5.25, -8.2) { Temperature: $T$};
    			
    			\draw[cross, bend left=30] (param_M) to node[pos=0.2,above,label] {\footnotesize $M=1$} (model_fno);
    			
    			\draw[arr] (t50_mid.west) -- ++(-0.3,0) |- (param_M.east);
    			\draw[arr] (t50_mid.west) -- ++(-0.3,0)|- (param_T.east);
    			
    	    \node[thm] (low) at (10, -9.5) {\footnotesize Thm 11: Low $T$ \\ Learnable indicator / partition};
    	    \node[thm] (high) at (0.5, -9.5) {\footnotesize Thm 12: High $T$\\ Overparametrized FNO};
    	    \draw[cross, bend left=50] (high) to node[pos=0.2,above,label] {\footnotesize $T\rightarrow \infty$} (model_fno);
    	    \node[thm] (temp) at (5.25, -9.5) {\footnotesize Thm 13: Temperature \\ controlled heterogeneity};
    	    \draw[arr] (temp) -- (high);
    	    \draw[arr] (temp) -- (low);
    	    \draw[arr] (param_T) -- (temp);

    		% ==================== ROW 2: IV, V ====================
    		
    		% IV
    		\node[module] (L4) at (0.5, -10.7) {IV. Operator theory};
    		\node[thm, minimum width=1.0] (r1) at (-2.2, -11.7) {\footnotesize $K^d$};
    		\node[thm, minimum width=1.0] (r2) at (3.2, -11.7) {\footnotesize $K^d M$};
    		\node[thm] (t13) at (0.5, -11.7) {\footnotesize Thm 14 (Thm 3):\\ Finite-rank self-adjoint};
    		\node[thm] (t14) at (-1.0, -12.7) {\footnotesize Thm 15:\\ Dense in $\mathcal{K}(\mathcal{H})$};
    		\node[thm] (spec) at (2.5, -12.7) {\footnotesize Thm 16: \\ Spectral theorem};
    		\draw[arr] (t13) -- (t14);
    		\draw[arr] (t13) -- (spec);
    		\draw[cross] (r1) -- (t13);
    		\draw[cross] (t13) -- (r2);
    		
    		% V
    		\node[module] (L5) at (10.0, -10.7) {V. Approximation theory};
    		\node[thm] (t15) at (10.0, -11.7) {\footnotesize Thm 17: \\ BV class approximation};
    		\node[thm] (f2o) at (8.0, -12.8) {\footnotesize Prop 17.1: \\ Function $\to$ operator\\ approximation};
    		\node[thm] (cmp) at (11.5, -12.8) {\footnotesize Prop 17.2: \\ Budget-accuracy \\ comparison};
    		\draw[arr] (t15) -- (f2o);
    		\draw[arr] (t15) -- (cmp);
    			
    		% ==================== CROSS-LAYER ARROWS ====================
    		\draw[cross, bend right=15]
    		(high.west) to node[pos=0.5, above, label] {\footnotesize rank} (r1);
    			
    		\draw[cross, bend right=15]
    		(high.east) to node[pos=0.8, above, label] {\footnotesize trigonometric approx.} (t15.west);
    			
    		\draw[cross, bend left=35]
    		(low.east) to node[pos=0.4, above, label] {\footnotesize piecewise approx.} (t15.east);
    		
    		\draw[cross, bend right=35]
    		(model_fno.west) to node[pos=0.4, above, label] {\footnotesize rank} (r1.north);
    		
    		\draw[cross, bend left=15]
    		(t50_mid.south) to node[pos=0.5, above, label] {\footnotesize rank} (r2);
    			
    		\end{tikzpicture}
    	}
    	\caption{Proof structure of the ABLE theoretical framework. Upper row: Fourier basis foundation (I) and ABLE frame extension (II). Middle: representation ability (III). Bottom row: operator theory (IV), and approximation theory (V). Dashed arrows indicate cross-layer dependencies.}
    	\label{fig:proof_structure_Appendix_A}
    \end{figure}
\subsection*{Notations}
We summarize the key notations used in this paper as follows:
\begin{enumerate}
	\item Discrete set of $M$ points: $[M]=\{1,2,\cdots, M\}$,
	\item Discrete grid: the spatial domain $E\subset \mathbb{R}^d$ can be discritized as grids dependent on the resolution: 
	$E=\times_{i=1}^{d}[N_i]=\{(a_1, \cdots,  a_d)| a_i\in [N_i]\}$ where $N_i$ is the number of points along axis $i\in\{1, 2, \cdots, d\}$, 
	\item 
	Continuous periodic space: $\mathbb{T}^d=[0,1]^d/\sim$, which is a $d$-dimensional unit cube with periodic boundary conditions (i.e., boundaries glued together),
    \item Spectral truncation to first $K^d$ modes: $\mathbf{k}\le\mathbf{K}$,
	\item Different spatial domains $E$ have the Fourier transformation and the inverse Fourier transformation that differ in form but coincide as mathematical objects. When $E$ takes the discrte grid $[N]^d$, the periodic space $\mathbb{T}^d$ or the entire space $\mathbb{R}^d $, the spectral variable $\textbf{k}$ lies in $[N]^d,\, \mathbb{Z}^d
	,\,\mathbb{R}^d$ respectively. The Fourier and inverse Fourier transforms are realized equivalent as:
	\begin{align}
		\mathcal{F}(f)& \Leftrightarrow \frac{1}{\sqrt{N_1\cdots N_d}}\sum_{n_1,\cdots,n_d=1}^{N_1,\cdots N_d} \exp(-i\sum_{j=1}^{d}k_{j}n_j)f(n_1,\cdots ,n_d)\nonumber\\& \Leftrightarrow \frac{1}{(2\pi)^{\frac{d}{2}}}\int_{\mathbb{T}^d} d\textbf{x}\exp(-i\textbf{k}\cdot\textbf{x})f(\textbf{x})\nonumber\Leftrightarrow\frac{1}{(2\pi)^{\frac{d}{2}}}\int_{\mathbb{R}^d} d\textbf{x}\exp(-i\textbf{k}\cdot\textbf{x})f(\textbf{x})\nonumber\\
		\mathcal{F}^{-1}(\hat{f})& \Leftrightarrow \frac{1}{\sqrt{N_1\cdots N_d}}\sum_{n_1,\cdots,n_d=1}^{N_1,\cdots N_d} \exp(i\sum_{j=1}^{d}k_{j}n_j)\hat{f}_{k_1,\cdots ,k_d}\nonumber\\& \Leftrightarrow \frac{1}{(2\pi)^{\frac{d}{2}}}\sum_{\textbf{k}\in\mathbb{Z}^d}\exp(i\textbf{k}\cdot\textbf{x})\hat{f}_{\textbf{k}}\nonumber\Leftrightarrow\frac{1}{(2\pi)^{\frac{d}{2}}}\int_{\mathbb{R}^d} d\textbf{k}\exp(i\textbf{k}\cdot\textbf{x})\hat{f}_{\textbf{k}}\nonumber
	\end{align}
\end{enumerate}
In the main paper, we use $E\subset\mathbb{R}^d$ to present the general formulation. For theoretical analysis of  basis that require a countable spectral index set, we fix $E=\mathbb{T}^d$ as a bounded periodic domain, so that the corresponding spectral indices $\textbf{k}\in\mathbb{Z}^d$ are countable, enabling the formulation of all theorems. In implementation, finite-resolution problems are inherently defined on discrete grid $E= [N]^d$. 
\subsection*{Raodmap of the proofs}
The logical dependencies among the theoretical results are illustrated in Fig.~\ref{fig:proof_structure_Appendix_A}.

\subsection{Fourier basis}
\begin{theorem}\label{theorem_1}
	Fourier series $\mathcal{B}=\{e^{i\textbf{k}\cdot \textbf{x}}\}_{\textbf{k}\in \mathbb{Z}^d}$ is an orthonormal basis of the function space $L^2(\mathbb{T}^d,\mathbb{C})=\{f:\mathbb{T}^d\rightarrow \mathbb{C}|\int_{\mathbb{T}^d}|f(\textbf{x})|^2d\textbf{x}<\infty\}$, which satisfies the following properties:
	\begin{enumerate}
		\item Linear independence: for $\forall \{e_1,\cdots, e_n\}\subset \mathcal{B}$ and $c_1,\cdots,c_n\in \mathbb{C}$,  $\sum_{i=1}^{n}c_ie_i(\textbf{x})=0~ a.e.~\textbf{x}\in \mathbb{T}^d \Rightarrow c_i=0$. 
		\item\label{completeness}  Completeness: $\forall f\in L^2(\mathbb{T}^d,\mathbb{C})$, $\Bigl[\forall\textbf{k}\in \mathbb{Z}^d,\langle f,e_{\textbf{k}}\rangle=\int_{\mathbb{T}^d}f(\textbf{x})e^{-i\textbf{k}\textbf{x}}d\textbf{x}=0 \Bigl]\Rightarrow f=0~ a.e.~\textbf{x}\in \mathbb{T}^d$
		\item Orthogonality and Normality: \begin{equation}\langle e_{\textbf{k}},e_{\textbf{k}'}\rangle=\langle e^{i\textbf{x}\cdot \textbf{k}},e^{i\textbf{x}\cdot \textbf{k}'}\rangle=\frac{1}{(2\pi)^d}\int_{\mathbb{T}^d} e^{i\textbf{k}\textbf{x}}e^{-i\textbf{k}'\textbf{x}}d\textbf{x}=\delta_{\textbf{k},\textbf{k}'}. \label{eq:f-orthn}\end{equation}
	\end{enumerate}
\end{theorem}
\begin{corollary}[Schauder]\label{Corollary1.1} Fourier basis forms a \textbf{Schauder basis} for $L^2(\mathbb{T}^d, \mathbb{C})$, i.e.,
	$\forall f\in L^2(\mathbb{T}^d,\mathbb{C})$, $\exists$  \text{unique} $c_\textbf{k}=\hat{f}_\textbf{k}\in \mathbb{C}$ s.t. 
	$f(\textbf{x})=\sum_{\textbf{k}\in\mathbb{Z}^d} \hat{f}_{\textbf{k}} e^{i\textbf{k}\cdot\textbf{x}}\label{eq:f-complete} $
	where parameters $\hat{f}_{\textbf{n}}=\langle f,e^{-i\textbf{n}\cdot\textbf{x}}\rangle=\frac{1}{(2\pi)^d}\int f(\textbf{x})e^{-i\textbf{n}\textbf{x}}d\textbf{x}$, and this gives the Fourier transformation $\mathcal{F}:L^2(\mathbb{T}^d,\mathbb{C})\rightarrow l^2(\mathbb{Z}^d) , f(\textbf{x})\mapsto \hat{f}_{\textbf{n}} $.
\end{corollary}
\begin{proof}
	First, by the orthonormal property, we have the Bessel inequality $\sum_{\textbf{k}\in\mathbb{Z}^d}|\langle f,e_{\textbf{k}}\rangle|^2\le \|f\|^2_{L^2(\mathbb{T}^d)}$. (Defining $S_N=\sum_{n=1}^{N}\langle f, e_n\rangle e_n$ where $\{e_n\}_{n=1}^{N}\subset \mathcal{B}$, we obtain $0\le\|f-S_N\|^2=\|f\|^2-\sum_{n=1}^{N}|\langle f, e_n\rangle|^2$. Since $\mathcal{B}$ is countable, by taking $N\rightarrow \infty$, $\sum_{\textbf{k}\in\mathbb{Z}^d}|\langle f,e_{\textbf{k}}\rangle|^2\le \|f\|^2_{L^2(\mathbb{T}^d)}$.) The partial sum sequence $S_N = \sum_{n=1}^{N}\langle f, e_n\rangle e_n $ is a Cauchy sequence by $\|S_N-S_M\|^2=\sum _{n=N+1}^{M}|\langle f, e_n\rangle|^2\rightarrow 0$ which follows from the boundedness guaranteed by the Bessel inequality. Then, since $L^2(\mathbb{T}^d)$ is complete, $S_n\xrightarrow{L^2}g=\sum_{\textbf{k}\in\mathbb{Z}^d}\langle f,e_{\textbf{k}}\rangle e_{\textbf{k}}$. By the orthonormal property again, $\langle f-g,e_n\rangle=0$ for $ \forall e_n\in \mathcal{B} $, then by the complete property of the Fourier series, $f=g$, therefore it suffices to take $c_\textbf{k}=\langle f,e_\textbf{k}\rangle$. 
\end{proof}

%
\iffalse
\begin{corollary} Parameters $\hat{f}_{\textbf{n}}=\langle f,e^{-i\textbf{n}\cdot\textbf{x}}\rangle=\frac{1}{(2\pi)^d}\int f(\textbf{x})e^{-i\textbf{n}\textbf{x}}d\textbf{x}$, and this gives the Fourier transformation $\mathcal{F}:L^2(\mathbb{T}^d,\mathbb{C})\rightarrow l^2(\mathbb{Z}^d) , f(\textbf{x})\mapsto \hat{f}_{\textbf{n}} $.
	\begin{proof}
		By acting $\langle \cdot, e^{-i\textbf{k}\cdot\textbf{x}}\rangle $ on (\ref{eq:f-complete}), then using (\ref{eq:f-orthn}), \begin{equation}\langle f(\textbf{x}),e^{-i\textbf{k}\cdot\textbf{x}}\rangle=\langle \sum_{\textbf{k}'\in \mathbb{Z}^d}f_{\textbf{k}'}e^{i\textbf{k}'\cdot\textbf{x}} ,e^{-i\textbf{k}\cdot\textbf{x}}\rangle=\sum_{\textbf{k}'\in \mathbb{Z}^d}f_{\textbf{k}'}\langle e^{i\textbf{k}'\cdot\textbf{x}} ,e^{-i\textbf{k}\cdot\textbf{x}}\rangle=\sum_{\textbf{k}'\in \mathbb{Z}^d}f_{\textbf{k}'}\delta_{\textbf{k}'\textbf{k}}=\hat{f}_{\textbf{k}}\end{equation}
	\end{proof}
\end{corollary}
\fi
%

\begin{remark} For the discrete case, $\textbf{x}\in\mathbb{T}^d$ becomes the discrete grid $\textbf{n}\in\times_{i=1}^{d}[N_i]$, all the above properties are still satisfied up to a change in form: completeness (\ref{eq:f-complete}) becomes $f(\textbf{n})=\sum_{\textbf{k}\in\mathbb{Z}^d} \hat{f}_{\textbf{k}} e^{i\textbf{k}\cdot\textbf{n}}$ with $i\textbf{k}\cdot \textbf{n}:=i 2\pi\sum_{i=1}^{N_i}\frac{k_in_i}{N_i}$, and Eqn.~(\ref{eq:f-orthn}) becomes
	$\langle e_{\textbf{k}},e_{\textbf{k}'}\rangle=\langle e^{i\textbf{x}\cdot \textbf{k}},e^{i\textbf{x}\cdot \textbf{k}'}\rangle=\frac{1}{\Pi_{i=1}^{d}N_i}\sum_{n_1,\cdots n_d=1}^{N_1,\cdots N_d} e^{i\textbf{k}\textbf{n}}e^{-i\textbf{k}'\textbf{n}}=\frac{1}{\Pi_{i=1}^{d}N_i}\sum_{n_1,\cdots n_d=1}^{N_1,\cdots N_d} e^{i\textbf{k}\textbf{n}}e^{-i\textbf{k}'\textbf{n}}=\delta_{\textbf{k},\textbf{k}'} $.
\end{remark}

\begin{theorem}[Plancherel Theorem]
	Fourier transformation $\mathcal{F}:L^2(\mathbb{T}^d,\mathbb{C})\rightarrow l^2(\mathbb{Z}^d) , f(\textbf{x})\mapsto \hat{f}_{\textbf{n}} $ is a linear isometry, i.e., it satisfies:
	\begin{enumerate}
		\item Linearity: $\mathcal{F}(af+bg)=a\mathcal{F}(f)+b\mathcal{F}(g)$
		\item Bijection: $\forall \hat{f}=(\{\hat{f}_{\textbf{k}}\}_{\textbf{k}\in \mathbb{Z}^d})\in l^2(\mathbb{Z}^d)$, $\exists$ unique $f\in L^2(\mathbb{T}^d,\mathbb{C})$ s.t. $\mathcal{F}(f)=\hat{f}$
		\item Parseval's identity / theorem: $\|f\|_{L^2(\mathbb{T}^d)}=\|\hat{f}\|_{l^2(\mathbb{Z}^d)}$ i.e. $\frac{1}{(2\pi)^d}\int_{\mathbb{T}^d}|f(\textbf{x})|^2d\textbf{x}=\sum_{\textbf{k}\in\mathbb{Z}^d}|\hat{f}(\textbf{k})|^2$
	\end{enumerate}\label{thm:plancherel}
	\begin{proof} Linearity is easy to check. For bijection, construct $\tilde{f}(\textbf{x})=\sum_{\textbf{k}\in\mathbb{Z}^d}\hat{f}_{\textbf{k}}e^{i\textbf{k}\textbf{x}}$, then the inner product $\langle \tilde{f}(\textbf{x}),e^{-i\textbf{k}\textbf{x}}\rangle=\langle\sum_{\textbf{k}'\in\mathbb{Z}^d}\hat{f}_{\textbf{k}'}e^{i\textbf{k}'\textbf{x}},e^{-i\textbf{k}\textbf{x}}\rangle =\sum_{\textbf{k}'\in\mathbb{Z}^d}\hat{f}_{\textbf{k}'}\langle e^{i\textbf{k}'\textbf{x}},e^{-i\textbf{k}\textbf{x}}\rangle=\sum_{\textbf{k}'\in\mathbb{Z}^d}\hat{f}_{\textbf{k}'}\delta_{\textbf{k}'\textbf{k}} =\hat{f}_{\textbf{k}}$. By the Schauder property (\ref{Corollary1.1}), $f=\tilde{f}$ is unique. For Parseval's identity / theorem, $ \frac{1}{(2\pi)^d}\int_{\mathbb{T}^d}|f(\textbf{x})|^2d\textbf{x}= \frac{1}{(2\pi)^d}\int_{\mathbb{T}^d}\sum_{\textbf{k}\in \mathbb{Z}^d,\textbf{k}'\in \mathbb{Z}^d}\hat{f}_{\textbf{k}}^*\hat{f}_{\textbf{k}'}e^{i\textbf{k}\textbf{x}-i\textbf{k}'\textbf{x}}d\textbf{x}=\sum_{\textbf{k}\in \mathbb{Z}^d,\textbf{k}'\in \mathbb{Z}^d}\hat{f}_{\textbf{k}}^*\hat{f}_{\textbf{k}'}[\frac{1}{(2\pi)^d}\int_{\mathbb{T}^d}e^{i\textbf{k}\textbf{x}-i\textbf{k}'\textbf{x}}d\textbf{x}]=\sum_{\textbf{k}\in \mathbb{Z}^d,\textbf{k}'\in \mathbb{Z}^d}\hat{f}^*_{\textbf{k}}\hat{f}_{\textbf{k}'}\delta_{\textbf{k}\textbf{k}'}=\sum_{k\in \mathbb{Z}^d}|\hat{f}_{\textbf{k}}|^2$
	\end{proof}
\end{theorem}
\begin{proposition}[Fourier basis commutative diagram]
	For any operator $\mathcal{G}:L^2(\mathbb{T}^d,\mathbb{C})\rightarrow L^2(\mathbb{T}^d,\mathbb{C})$, there exists a unique spectral representation $\hat{G}:l^2(\mathbb{Z}^d)\rightarrow l^2(\mathbb{Z}^d)$, s.t. commutative diagram (\ref{cd: fourier}) $\hat{G}\circ \mathcal{F}=\mathcal{F}\circ \mathcal{G}$.
	\begin{proof}
		By Plancherel theorem (\ref{thm:plancherel}), $\mathcal{F}$ is linear isometry, thus invertible. Then we can construct $ \tilde{G}=\mathcal{F}\circ\mathcal{G}\circ\mathcal{F}^{-1}$, which satisfies $\tilde{G}\circ \mathcal{F}=\mathcal{F}\circ\mathcal{G}$. Then $\hat{G}=\tilde{G}$ by the proof of uniqueness: if $\hat{G}_1\circ\mathcal{F}=\mathcal{F}\circ\mathcal{G}$ and $\hat{G}_2\circ\mathcal{F}=\mathcal{F}\circ\mathcal{G}$, then $(\hat{G}_1-\hat{G}_2)\circ\mathcal{F}=0$. Since $\mathcal{F}$ is a bijection, composing both sides of the equality on the right with $\mathcal{F}^{-1}$ yields $\hat{G}_1=\hat{G}_2$. Hence,  $\hat{G}=\tilde{G}$ and is unique.
	\end{proof}
	\begin{figure} % [htbp]
		\centering
		\begin{subfigure}{0.45\textwidth}
			\centering
			\begin{tikzcd}
				L^2(\mathbb{T}^d,\mathbb C)\arrow[r,"\mathcal{G}"]\arrow[d,"\mathcal{F}"]&L^2(\mathbb{T}^d,\mathbb C)\arrow[d,"\mathcal{F}"]\\
				l^2(\mathbb{Z}^d)\arrow[r,"\hat{G}"] & l^2(\mathbb{Z}^d)
			\end{tikzcd}\caption{Fourier basis commutative diagram}\label{cd: fourier}
		\end{subfigure}
		\hfill
		\begin{subfigure}{0.45\textwidth}
			\centering 
			\begin{tikzcd}
				L^2(E,\mathbb R)\arrow[r,"\mathcal{G}"]\arrow[d,"\mathcal{F}_{gen}"]&L^2(E,\mathbb R)\arrow[d,"\mathcal{F}_{gen}"]\\
				l^2(\mathbb{Z}^d)\arrow[r,"\hat{G}_{gen}"] & l^2(\mathbb{Z}^d)
			\end{tikzcd}\caption{Generalized Fourier basis commutative diagram}\label{cd:generalized basis}
		\end{subfigure}
		\caption{Isometry structure of Fourier and generalized Fourier transformation}
	\end{figure}
\end{proposition}
\begin{proposition}[Generalization to any orthonormal basis] For any generalized orthonormal basis $\mathcal{B}=\{e_{\textbf{k}}(\textbf{x})\}_{\textbf{k}\in\mathbb{Z}^d}$ of $L^2(\mathbb{T}^d,\mathbb{C})$, all items of theorem (\ref{theorem_1}) and (\ref{thm:plancherel}) still holds, by modifying
	\begin{enumerate}
		\item Fourier basis components: $e^{i\textbf{k}\textbf{x}} \rightarrow e_{\textbf{k}}(\textbf{x})$, $e^{-i\textbf{k}\textbf{x}}\rightarrow e_{\textbf{k}}^*(\textbf{x})$,
		\item Fourier transformation $\mathcal{F} $ as Generalized Fourier transformation $ \mathcal{F}_{gen}: L^2(\mathbb{T}^d,\mathbb{C})\rightarrow l^2(\mathbb{Z}^d), f(\textbf{x})\mapsto \hat{f}_k=\langle f,e_{\textbf{k}}\rangle= \int f(\textbf{x}) e^*_{\textbf{k}}(\textbf{x})d\textbf{x}$.
		\item The spectral representation of the operator $\mathcal{G}$ changes from $\hat{G}=\mathcal{F}\circ\mathcal{G}\circ\mathcal{F}^{-1}$ to $\hat{G}_{gen}=\mathcal{F}_{gen}\circ\mathcal{G}\circ\mathcal{F}^{-1}_{gen}$, and  satisfies the new commutation relation $\hat{G}_{gen}\circ\mathcal{F}_{gen}=\mathcal{F}_{gen}\circ\mathcal{G}$. Thus the commutative diagram becomes (\ref{cd:generalized basis}).
	\end{enumerate}	
\end{proposition}
\subsection{Frame theory of ABLE}
\begin{definition}[Adaptive Learnable Basis]
	For any basis $\mathcal{B}=\{e_\textbf{k}\}_{\textbf{k}\in \mathbb{Z}^d}$ of $L^2(\mathbb{T}^d,\mathbb{C})$, one could extend it to \textbf{Adaptive Learnable Basis} $\mathcal{B}_{ABLE}=\{e_{\textbf{k},y}\}_{\textbf{k}\in \mathbb{Z}^d,y\in \chi}=\{e_\textbf{k}(\textbf{x})p(\textbf{x},y)^{\frac{1}{2}}\}_{\textbf{k}\in \mathbb{Z}^d,y\in \chi}$, where 
	\begin{enumerate}
		\item $\chi$ is an ancillary measurable space $(\chi,\mathcal{A}_{\chi},\mu)$.
		\item $p$ is a probability kernel function $p:\mathbb{T}^d \times \chi \rightarrow \mathbb{R}_{\ge0}$, i.e., $\forall \textbf{x}\in \mathbb{T}^d$, $\int_{\chi} p(\textbf{x},y)\mu(dy)=1$.
	\end{enumerate}
	\begin{remark}
		In our paper, we only refer $\chi$ to two simplest special cases: the real space with Lebesgue measure $(\chi,\mathcal{A}_{\chi},\mu)=(\mathbb{R},\mathcal{B}(\mathbb{R}),\lambda)$ (where $\int_{\mathbb{R}} p(\textbf{x},y)dy=1$) or the $M$-point finite set with counting measure $([M],\mathcal{P}([M]),\mu_c)$ (where $\sum_{m=1}^{M}p(\textbf{x},m)=1$).
		Theoretically, the finite case could be extended to a countable set case with $\Sigma_{m=1}^{\infty}p(\textbf{x},m)=1$.
	\end{remark}
\end{definition}
\begin{theorem}[ABLE is Parseval tight frame]
	\label{theorem_3}
	Adaptive learnable basis $\mathcal{B}_{ABLE}=\{e_{\textbf{k},y}\}_{\textbf{k}\in \mathbb{Z}^d,y\in \chi}=\{p(\textbf{x},y)^{\frac{1}{2}}e^{i\textbf{k}\cdot \textbf{x}}\}_{\textbf{k}\in \mathbb{Z}^d,y\in \chi}$ is a \textbf{Parseval tight frame} of function space $L^2(\mathbb{T}^d,\mathbb{C})$, i.e., it satisfies the following properties:
	\begin{enumerate}
		\item Frame inequality: $\forall f\in L^2(\mathbb{T}^d,\mathbb{C})$, $ \exists 0<A\le B<\infty$ s.t. $A\|f\|_{L^2(\mathbb{T}^d)}^2\le\sum_{e_i\in \mathcal{B}}|\langle f,e_i\rangle|^2\le B\|f\|^2_{L^2(\mathbb{T}^d)}$.
		\item Parseval: $A=B=\frac{1}{(2\pi)^d}$, i.e., $\frac{1}{(2\pi)^d}\|f\|_{L^2(\mathbb{T}^d)}^2=\sum_{e_i\in \mathcal{B}_{ABLE}}|\langle f,e_i\rangle|^2=\int_\chi \mu(dy)\sum_{\textbf{k}\in \mathbb{Z}^d} |\hat{f}_{\textbf{k},y}|^2$ with coefficient $\hat{f}_{\textbf{k},y}=\langle f,e_{\textbf{k},y}\rangle$.
	\end{enumerate}
	\begin{proof}
		Since Parsevel $\Rightarrow$ frame inequality, it suffices to prove the Parseval property of ABLE. We assume $p$ s.t. $\int |f(\textbf{x})|^2p(\textbf{x},y)d\textbf{x}<\infty$. (This always holds in our setting, since we take $\chi=[M]$, and thus each $p(x,y)\in [0,1]$ for $y\in [M]$, which implies $\int |f(\textbf{x})|^2p(\textbf{x},y)d\textbf{x}\le \int |f(\textbf{x})|^2d\textbf{x}=\|f\|_{L^2}<\infty$.) Therefore, $p(\textbf{x},y)^{1/2}f(\textbf{x})\in L^2(\mathbb{T}^d)$. Then, expand RHS as 
		\begin{equation}\int \mu(dy)\Sigma_{\textbf{k}\in \mathbb{Z}^d} |\hat{f}_{\textbf{k},y}|^2=\int \mu(dy)\Sigma_{\textbf{k}\in \mathbb{Z}^d}\frac{1}{(2\pi)^d}\int e^{i\textbf{k}\textbf{x}}f^*(\textbf{x})p(\textbf{x},y)^{\frac{1}{2}}d\textbf{x}\frac{1}{(2\pi)^d}\int e^{-i\textbf{k}\textbf{x}'}f(\textbf{x}')p(\textbf{x}',y)^{\frac{1}{2}}d\textbf{x}'.\end{equation}
		Since $p(\textbf{x},y)^{1/2}f(\textbf{x})\in L^2(\mathbb{T}^d)$, we could interchange the integration order as in the Fourier case:
		\begin{align}&\int \mu(dy)|\hat{f}_{\textbf{k},y}|^2=\int \mu(dy)\int d\textbf{x} f^*(\textbf{x})p(\textbf{x},y)^{\frac{1}{2}}\int d\textbf{x}'f(\textbf{x}')p(\textbf{x}',y)^{\frac{1}{2}}\Sigma_{\textbf{k}\in \mathbb{Z}^d}\frac{1}{(2\pi)^{2d}} e^{i\textbf{k}\textbf{x}}e^{-i\textbf{k}\textbf{x}'}=\nonumber\\&\int \mu(dy)\int d\textbf{x}f^*(\textbf{x})p(\textbf{x},y)^{\frac{1}{2}}\int d\textbf{x}' f(\textbf{x}')p(\textbf{x}',y)^{\frac{1}{2}}\frac{1}{(2\pi)^{d}}\delta(\textbf{x}-\textbf{x}')=\frac{1}{(2\pi)^d}\int \mu (dy)\int_{\mathbb{T}^d} d\textbf{x} |f(\textbf{x})|^2p(\textbf{x},y).\end{align} 
		Then since the integration kernel is non-negative, we can continue to interchange the integration by Fubini's theorem.
		\begin{align}
			\int \mu(dy)|\hat{f}_{\textbf{k},y}|^2 =\frac{1}{(2\pi)^d}\int_{\mathbb{T}^d} d\textbf{x}|f(\textbf{x})|^2\bigg[\int \mu (dy)p(\textbf{x},y)\bigg]=\frac{1}{(2\pi)^d}\int_{\mathbb{T}^d} d\textbf{x}|f(\textbf{x})|^2.
		\end{align} The second inequality follows from the probability kernel property of $p$ ($\int_\chi \mu(dy)p(\textbf{x},y)dy=1$).
	\end{proof}
\end{theorem}

\begin{corollary} The frame inequality directly implies the completeness of the tight frame, i.e., for $\forall f\in L^2(\mathbb{T}^d,\mathbb{C})$, $\Bigl[\forall\textbf{k}\in \mathbb{Z}^d , y\in \chi,\langle f,e_{\textbf{k},y}\rangle=\int_{\mathbb{T}^d}f(\textbf{x})e^{-i\textbf{k}\textbf{x}}p(\textbf{x},y)^{\frac{1}{2}}d\textbf{x}=0 \Bigl]~\Rightarrow~ f=0~ a.e.~ \textbf{x}\in \mathbb{T}^d$.
	\begin{proof}
		Since $A\|f\|_{L^2(\mathbb{T}^d)}^2\le\sum_{e_i\in \mathcal{B}}|\langle f,e_i\rangle|^2$, the RHS vanishes. Then because $A>0$, $\|f\|=0$, i.e., $ f=0~a.e.$
	\end{proof}
\end{corollary}
\begin{remark}
	For tight frames, the completeness still holds, but linear independence and orthogonality are not necessarily satisfied.
\end{remark}
\begin{definition}[ABLE transformation]
	The ABLE transformation is defined as $\mathcal{A}:L^2(\mathbb{T}^d,\mathbb{C})\rightarrow L^2(\mathbb{Z}^d\times \chi) , f(\textbf{x})\mapsto \hat{f}_{\textbf{k},y} $ s.t. 
	\begin{equation}\hat{f}_{\textbf{k},y}=\langle f,e^{i\textbf{k}\cdot\textbf{x}}p(\textbf{x},y)^{\frac{1}{2}}\rangle=\frac{1}{(2\pi)^d}\int f(\textbf{x})p(\textbf{x},y)^{\frac{1}{2}}e^{-i\textbf{k}\cdot\textbf{x}}d\textbf{x}\end{equation} 
\end{definition}
\begin{corollary}[Semi-Schauder]\label{corollary3.2}
	ABLE is a frame s.t. $\forall f\in L^2(\mathbb{T}^d,\mathbb{C})$, $\exists$  \text{(not necessarily unique)} $c_{\textbf{k},y}\in \mathbb{C}$ (i.e., $c$ may lie outside $L^2(\mathbb{Z}^d\times \chi)$ ) s.t. $f(\textbf{x})=\int\mu(dy)\sum_{\textbf{k}\in\mathbb{Z}^d }c_{\textbf{k},y}e_{\textbf{k},y}(\textbf{x})$. A natural choice is $f(\textbf{x})=\int_{\chi}\mu(dy)\sum_{\textbf{k}\in\mathbb{Z}^d}\hat{f}_{\textbf{k},y}e_{\textbf{k},y}(\textbf{x})=\int_{\chi}\mu(dy)\sum_{\textbf{k}\in\mathbb{Z}^d} \hat{f}_{\textbf{k},y} e^{i\textbf{k}\cdot\textbf{x}} $
	where parameters $\hat{f}_{\textbf{k},y}=\langle f,e_{\textbf{k},y}\rangle=\frac{1}{(2\pi)^d}\int f(\textbf{x})p(\textbf{x},y)^{\frac{1}{2}}e^{-i\textbf{k}\cdot\textbf{x}}d\textbf{x}$. 
	\begin{proof} 
    By the assumption that $p(\textbf{x},y)^{\frac{1}{2}}f(\textbf{x})\in L^2(\mathbb{T}^d)$, we could interchange the order of integration and sum:
		$\newline  f(\textbf{x})=\int_{\chi}\mu(dy)\sum_{\textbf{k}\in\mathbb{Z}^d}\hat{f}_{\textbf{k},y}e_{\textbf{k},y}(\textbf{x})=\int_{\chi}\mu(dy)\sum_{\textbf{k}\in\mathbb{Z}^d}\frac{1}{(2\pi)^d}\int f(\textbf{x}')p(\textbf{x}',y)^{\frac{1}{2}}e^{-i\textbf{k}\cdot\textbf{x}'}d\textbf{x}' p(\textbf{x},y)^{\frac{1}{2}}e^{i\textbf{k}\cdot\textbf{x}}=\newline\int_{\chi}\mu(dy)\int d\textbf{x}'f(\textbf{x}')p(\textbf{x}',y)^{\frac{1}{2}} p(\textbf{x},y)^{\frac{1}{2}} [\frac{1}{(2\pi)^d}\sum_{\textbf{k}\in\mathbb{Z}^d}e^{-i\textbf{k}\cdot\textbf{x}'}e^{i\textbf{k}\cdot\textbf{x}}]=\int_{\chi}\mu(dy)\int d\textbf{x}'f(\textbf{x}')p(\textbf{x}',y)^{\frac{1}{2}} p(\textbf{x},y)^{\frac{1}{2}}\delta(\textbf{x}-\textbf{x}')=\newline \int_{\chi}\mu(dy)f(\textbf{x})p(\textbf{x},y)=f(\textbf{x})[\int_{\chi}\mu(dy)p(\textbf{x},y)]=f(\textbf{x})
		$.
	\end{proof}
	\begin{remark}
	The uniqueness condition is not satisfied since the orthogonality does not hold in general.
	\end{remark}
\end{corollary}

\begin{theorem}\label{theorem_4}
	ABLE transformation $\mathcal{A}:L^2(\mathbb{T}^d, \mathbb{C})\rightarrow \text{Im}(\mathcal{A})\subset L^2(\mathbb{Z}^d\times \chi) , f(\textbf{x})\mapsto \hat{f}_{\textbf{k},y}$ 
	is an isometry ($p(\textbf{x},m)$ could depend on $f(\textbf{x})$ in practice, which departs from the definition of frame, but the isometric property still holds).
	\begin{enumerate}
		\item Bijection: $\forall\hat{f}=( \hat{f}_{\textbf{x},y})\in \text{Im}(\mathcal{A})$, $\exists \text{unique} f\in L^2(\mathbb{T}^d,\mathbb{C})$ s.t. $\mathcal{A}(f)=\hat{f}$.
		\item Parseval's theorem: $ \|f\|_{L^2(\mathbb{T}^d)}=\|\hat{f}\|_{L^2(\mathbb{T}^d\times\chi)}$ i.e. $\frac{1}{(2\pi)^d}\|f\|_{L^2(\mathbb{T}^d)}^2=\sum_{e_i\in \mathcal{B}_{able}}|\langle f,e_i\rangle|^2=\int \mu(dy)\Sigma_{\textbf{k}\in \mathbb{Z}^d} |\hat{f}_{\textbf{k},y}|^2$.
	\end{enumerate}
	\begin{proof}
		Parseval's theorem was proved in theorem (\ref{theorem_3}).
		The inverse ABLE transformation was given in (\ref{corollary3.2}) as:
		\begin{equation}
			\mathcal{A}^{-1}(\hat{f})=\int \mu(dy)\sum_{\textbf{k}\in \mathbb{Z}^d}\hat{f}_{\textbf{k},y}e^{i\textbf{k}\textbf{x}}p(\textbf{x},y)^\frac{1}{2}.
		\end{equation} The uniqueness holds here because we consider $\text{Im}(\mathcal{A})$ rather than $L^2(\mathbb{Z}^d\times \chi)$. 
	\end{proof}
\end{theorem} 
\begin{proposition}[Commutative diagram of ABLE]
	For any operator $\mathcal{G}:L^2(\mathbb{T}^d,\mathbb{C})\rightarrow L^2(\mathbb{T}^d,\mathbb{C})$, there exists a unique ABLE representation $\hat{G}_{ABLE}:\text{Im}(\mathcal{A})\subset L^2(\mathbb{Z}^d\times \chi)\rightarrow \text{Im}(\mathcal{A})\subset L^2(\mathbb{Z}^d\times \chi)$ s.t. commutative diagram (\ref{cd:able}) $\hat{G}_{ABLE}\circ \mathcal{A}=\mathcal{A}\circ \mathcal{G}$.
	Thus for such $\mathcal{G}$, $\exists$ \text{unique} $\hat{G}_{ABLE}$ \text{s.t.} $\mathcal{G}=\mathcal{A}^{-1}\circ\hat{G}_{ABLE}\circ\mathcal{A} $
	\begin{proof}
		By theorem (\ref{theorem_4}), $\mathcal{A}$ is an isometry, thus invertible on $Im(\mathcal{A})$, then we can construct $ \tilde{G}=\mathcal{A}\circ\mathcal{G}\circ\mathcal{A}^{-1}$ which satisfies $\tilde{G}\circ \mathcal{A}=\mathcal{A}\circ\mathcal{G}$. Thus $\hat{G}_{ABLE}=\tilde{G}$ by the proof of uniqueness: if $\hat{G}_1\circ\mathcal{A}=\mathcal{A}\circ\mathcal{G}$ and $\hat{G}_2\circ\mathcal{A}=\mathcal{A}\circ\mathcal{G}$, then $(\hat{G}_1-\hat{G}_2)\circ\mathcal{A}=0$, since $\mathcal{A}$ is a bijection on $Im(\mathcal{A})$, directly composing it on the right with $\mathcal{A}^{-1}$ yields $\hat{G}_1=\hat{G}_2$. Thus $\hat{G}=\tilde{G}$ which is unique. 
	\end{proof}
	\begin{figure}
		\centering
		\begin{subfigure}[b]{0.45\textwidth}
			\centering
			\begin{tikzcd}
				L^2(\mathbb{T}^d,\mathbb C)\arrow[r,"\mathcal{G}"]\arrow[d,"\mathcal{A}"]&L^2(\mathbb{T}^d,\mathbb C)\arrow[d,"\mathcal{A}"]\\
				\text{Im}(\mathcal{A})\subset L^{2}(\mathbb{T}^d\times\chi,\mathbb C)\arrow[r,"\hat{G}_{ABLE}"] & \text{Im}(\mathcal{A})
			\end{tikzcd}\caption{ABLE commutative diagram}\label{cd:able}
		\end{subfigure}
		\hfill
		\begin{subfigure}{0.45\textwidth}
			\centering 
			\begin{tikzcd}
				L^2(\mathbb{T}^d,\mathbb C)\arrow[r,"\mathcal{G}"]\arrow[d,"\mathcal{A}_{gen}"]&L^2(\mathbb{T}^d,\mathbb C)\arrow[d,"\mathcal{A}_{gen}"]\\
				\text{Im}(\mathcal{A}_{gen})\subset L^{2}(\mathbb{T}^d\times\chi,\mathbb C)\arrow[r,"\hat{G}_{gen}"] & \text{Im}(\mathcal{A}_{gen})
			\end{tikzcd}\caption{Generalized ABLE commutative diagram}\label{cd: generalized able basis}
		\end{subfigure}	
		\caption{ABLE Isometry induced from Fourier and generalized Fourier transformation}
	\end{figure}
	\begin{remark}
		Surprisingly, we found that even if a function $f\in L^2(\mathbb{T}^d,\mathbb{C})$ could have multiple preimages $\hat{c}\in L^2(\mathbb{Z}^d\times \chi)$, however, when finding a preimage in $Im(\mathcal{A})$, the choice is unique. Fortunately, this is exactly the operator learning under ABLE representation. Thus, ABLE still admits a unique representation $\mathcal{G}=\mathcal{A}^{-1}\circ\hat{G}_{ABLE}\circ\mathcal{A}$.
	\end{remark}
\end{proposition}
\begin{proposition}[ABLE generalization for any orthonormal basis] Any generalized orthonormal basis $\mathcal{B}=\{e_{\textbf{k}}(\textbf{x})\}_{\textbf{k}\in\mathbb{Z}^d}$ of $L^2(\mathbb{T}^d,\mathbb{C})$ has its corresponding ABLE extension $\mathcal{B}_{ABLE}=\{e_{\textbf{k}}(\textbf{x})p(\textbf{x},y)^{\frac{1}{2}}\}_{\textbf{k}\in\mathbb{Z}^d, y\in \chi}$. All items of theorems (\ref{theorem_3}) and (\ref{theorem_4}) still hold, by modifying
	\begin{enumerate}
		\item ABLE components: $e^{i\textbf{k}\textbf{x}}p(\textbf{x},y)^{\frac{1}{2}} \rightarrow e_{\textbf{k}}(\textbf{x})p(\textbf{x},y)^{\frac{1}{2}}$, $e^{-i\textbf{k}\textbf{x}}p(\textbf{x},y)^{\frac{1}{2}}\rightarrow e_{\textbf{k}}^*(\textbf{x})p(\textbf{x},y)^{\frac{1}{2}}$,
		\item ABLE transformation $\mathcal{A}$ as generalized ABLE transformation$ \mathcal{A}_{gen}: L^2(\mathbb{T}^d,\mathbb{C})\rightarrow \text{Im}(\mathcal{A})\subset L^2(\mathbb{Z}^d\times \chi), f(\textbf{x})\mapsto \hat{f}_{\textbf{k},y}=\langle f,e_{\textbf{k},y}\rangle= \int f(\textbf{x}) e^*_{\textbf{k}}(\textbf{x})p(\textbf{x},y)^{\frac{1}{2}}d\textbf{x}$.
		\item The ABLE representation of operator $\mathcal{G}$ changes from $\hat{G}_{ABLE}=\mathcal{A}\circ\mathcal{G}\circ\mathcal{A}^{-1}$ to $\hat{G}_{gen}=\mathcal{A}_{gen}\circ\mathcal{G}\circ\mathcal{A}^{-1}_{gen}$ and satisfies the new commutation relation $\hat{G}_{gen}\circ\mathcal{A}_{gen}=\mathcal{A}_{gen}\circ\mathcal{G}$. Thus, the commutative diagram becomes (\ref{cd: generalized able basis})
	\end{enumerate}
\end{proposition}
\subsection{Representation ability of ABLE}
\begin{theorem}[Super-set]\label{theorem_5}
	ABLE neural operator family $\mathcal{T}_{ABLE}=\{\mathcal{A}^{-1}\circ\hat{R}_{ABLE}\circ\mathcal{A}| \hat{R}_{ABLE}:\hat{f}_{k,m}\rightarrow R(k,m)\hat{f}_{k,m}\}$ \textbf{strictly contains} the Fourier neural operator family $\mathcal{T}_{FNO}=\{\mathcal{F}^{-1}\circ\hat{R}\circ\mathcal{F}| \hat{R}: \hat{f}_k\rightarrow R(k)\hat{f}_k\} $, i.e., $\mathcal{T}_{FNO}\subset \mathcal{T}_{ABLE}$.
	\begin{proof}
		First, we show for $\forall \mathcal{G}_{FNO}\in \mathcal{T}_{FNO}$, $ \mathcal{G}_{FNO}\in \mathcal{T}_{ABLE}$. ABLE could be written as
		\begin{equation}
			\mathcal{G}_{ABLE}(f)(\textbf{x})=\int_{\chi}\mu(dy)\sum_{\textbf{k}\in \mathbb{Z}^d}e^{i\textbf{k}\textbf{x}}p(\textbf{k},y)^{\frac{1}{2}}R(\textbf{k},y)\int_{\mathbb{T}^d} f(\textbf{x}')e^{-i\textbf{k}\textbf{x}'}p(\textbf{x}',y)^{\frac{1}{2}}d\textbf{x}'
		\end{equation}
		For the discrete/lattice space case where $\mathbb{T}^d\rightarrow L= \times_{i=1}^{d}[N_i]$ and $\chi=[M]$, this is 
		\begin{equation}
			\mathcal{G}_{ABLE}(f)(\textbf{x})=\sum_{m=1}^{M}\sum_{\textbf{k}\in  L }e^{i\textbf{k}\textbf{x}}p(\textbf{x},m)^{\frac{1}{2}}R(\textbf{k},m)\sum_{\textbf{k}\in L} f(\textbf{x}')e^{-i\textbf{k}\textbf{x}'}p(\textbf{x}',m)^{\frac{1}{2}}
		\end{equation}
		Thus, by setting $\chi=\{1\}$ as a trivial space and $p(\textbf{x},1)=1$, it becomes
		\begin{equation}
			\mathcal{G}_{ABLE}(f)(\textbf{x})=\sum_{\textbf{k}\in L}e^{i\textbf{k}\textbf{x}}R(\textbf{k},1)\sum_{\textbf{x}'\in L }f(\textbf{x}')e^{-i\textbf{k}\textbf{x}'}d\textbf{x}'
		\end{equation}
		Since $R(\textbf{k},1)$ can be set to $R(\textbf{k})$, the ABLE neural operator reduces to FNO. Thus, ABLE family contains FNO.

		Second, we need to show that ABLE strictly contains FNO (i.e., $\exists \mathcal{G}_{ABLE}\in \mathcal{T}_{ABLE}$, $\mathcal{G}_{ABLE}\notin \mathcal{T}_{FNO}$). We only need to consider the discrete case.
		By interchange of the order of integration, ABLE could also be written as 
		\begin{align}
			\mathcal{G}_{ABLE}(f)(\textbf{x})&=\sum_{\textbf{x}'\in L}\sum_{m=1}^{M}[\sum_{\textbf{k}\in  L }e^{i\textbf{k}\textbf{x}-i\textbf{k}\textbf{x}'}R(\textbf{k},m)]p(\textbf{x},m)^{\frac{1}{2}}p(\textbf{x}',m)^{\frac{1}{2}}f(\textbf{x}')\nonumber
			\\&=\sum_{\textbf{x}'\in L}[\sum_{m=1}^{M}K(\textbf{x}-\textbf{x}',m) p(\textbf{x},m)^{\frac{1}{2}}p(\textbf{x}',m)^{\frac{1}{2}}]f(\textbf{x}')
			\nonumber 
		\end{align}
		Thus, the Kernel function $K(\textbf{x},\textbf{x}')=\sum_{m=1}^{M}R(\textbf{x}-\textbf{x}',m) p(\textbf{x},m)^{\frac{1}{2}}p(\textbf{x}',m)^{\frac{1}{2}}$ does not need to satisfy the translation invariant form $K(\textbf{x}-\textbf{x}')$, by which FNO is strictly restricted. Hence, ABLE is a strictly larger family than FNO.
	\end{proof}
\end{theorem}
\begin{definition}[Neural network modeled ABLE-basis]
    The \textbf{neural network modeled ABLE basis} is defined by setting the auxiliary space as $\chi=[M]$, and the auxiliary function as 
    \begin{equation}
		p(\textbf{x},m)=\text{Softmax}(\frac{\text{MLP}_{\theta}(f(\textbf{x}))}{T})
	\end{equation} with $\text{MLP}_{\theta}:F(\mathbb{T}^d,\mathbb{R}^{d_{hid}})\rightarrow F(\mathbb{T}^d,\mathbb{R}^m)$, $\text{Softmax}:F(\mathbb{T}^d,\mathbb{R}^m) \rightarrow F(\mathbb{T}^d,\Delta^m)$ and temperature $T\in \mathbb{R}_{>0}$.
\end{definition}
\begin{theorem}[Universal approximation theorem]
	By stacking a sufficiently large number of layers, the ABLE neural operator can approximate any target operator to any desired precision.
	\begin{equation}
		\forall \epsilon >0, \exists N,M,\theta=\{ W_l,R_l,\theta_{able}\}_{l=1}^{N}, s.t. \|\mathcal{G}_{\theta}-G^{\dagger}\|_{\infty} \le\epsilon
	\end{equation}
	where $W_l$ is the spatial-MLP parameter, $R_l$ is the spectral multiplier, $\theta_{able}$ is the parameter for adaptive basis learning.
	\begin{proof}
		Generally ABLE can be written as a kernel integral with $K(\textbf{x},\textbf{x}',f(\textbf{x}),f(\textbf{x})')$ which includes the linear kernel $K(\textbf{x},\textbf{x}')$. Hence, by the universal approximation theorem \citep{kovachki2021universal} of the linear kernel-integral operator, $N,W_l,R_l$ can be taken as in the linear case while $M=1$ and $p(\textbf{x},f(\textbf{x}),1)=1$, then this property is automatically satisfied.
	\end{proof}
\end{theorem}
\begin{theorem}[ABLE's endogenous non-linearity]
	When $p(\textbf{x},m)$ depends on $f(\text{x})$, ABLE neural operator is with endogenous non-linearity, without the help of extra activation functions $\sigma$.\label{thm: inherent nonlinear}
\end{theorem}
\begin{proof}
	By modeling $p(\textbf{x},m)$ as $p(\textbf{x},f(\textbf{x}),m)$, ABLE attains non-linearity, because the kernel function $K(\textbf{x},\textbf{x}')=\sum_{m=1}^{M}R(\textbf{x}-\textbf{x}',m) p(\textbf{x},m)^{\frac{1}{2}}p(\textbf{x}',m)^{\frac{1}{2}}$ now is lifted as
	$K(\textbf{x},\textbf{x}',f(\textbf{x}),f(\textbf{x}'))$ with dependence on input fields $f(\textbf{x})$, and $\mathcal{G}_{ABLE}=\sum_{\textbf{x}\in L}K(\textbf{x},\textbf{x}',f(\textbf{x}),f(\textbf{x}'))f(\textbf{x}')$ becomes a non-linear operator.
\end{proof}
\begin{remark}
		$p(\textbf{x},f(\textbf{x}),m)^{\frac{1}{2}}$ can be viewed as a learnable activation function, and a better property is that 'Gelu' function loses information, but the Parseval and isometric property ensures that ABLE's non-linearity is not harmful for information.
\end{remark}

\begin{theorem}[Low-temperature limit]
\label{thm: low_temp}	
    For the case where $T\rightarrow 0$, the Softmax-MLP modeled basis becomes an indicator function (i.e., a hard truncation) $p(\textbf{x},m)\rightarrow \mathbf{1}_{E_m^\theta}(\textbf{x})$ with support $E_m$ of positive measure, pairwise disjoint and $\cup_{m=1}^{M}E_m=\mathbb{T}^d$.
	The low-temperature-limit ABLE series become $\mathcal{B}_{ABLE}=\{e_{\textbf{k},m}\}_{\textbf{k}\in \mathbb{Z}^d,m=1}^M=\{\mathbf{1}_{E_m}(\textbf{x})e^{i\textbf{k}\textbf{x}}\}_{\textbf{k}\in \mathbb{Z}^d,m\in [M]}$ which satisfy:
	\begin{enumerate}
		\item Provided $K_{\max}<N$, the first $[N]^d$ terms of ABLE series: ${\mathcal{B}^{N}_{ABLE}=\{\mathbf{1}_{E_m}(\textbf{x})e^{i\textbf{k}\textbf{x}}\}_{\textbf{k}\in [N]^d,m\in [M]}}$ form a \textbf{Schauder basis} for the function space $\mathcal{V}_{tri}^{K_{\max}}=\text{Span}\{f\in L^2(\mathbb{T}^d)| f=\sum_{n=1}^{N} c_n e_n, e_n\in \{e^{i\textbf{k}\textbf{x}}\}_{\textbf{k}\in [K_{\max}]^d} \}$ which is the space spanned by all trigonometric polynomials of degree at most $K_{\max}$, i.e., for $\forall f\in \mathcal{V}_{tri}^{K_{\max}}$, $\exists$  \text{unique} $c_{\textbf{k},m}\in \mathbb{C}$ s.t. $f(\textbf{x})=\sum_{m=1}^{M}\sum_{\textbf{k}\in [N]^d}c_{\textbf{k},m}e_{\textbf{k},m}(\textbf{x})$.
		\item The superior limit of these spaces $\mathcal{V}_{tri}^{\infty}=\limsup_{K_{\max}\rightarrow \infty}\mathcal{V}^{K_{\max}}_{tri}$ is dense in $L^2(\mathbb{T}^d,\mathbb{C})$. 
	\end{enumerate}
	\begin{proof} 
		The partition property is because $\forall \textbf{x}\in \mathbb{T}^d,\exists$ only one $ m $ s.t. $p(\textbf{x},m)=1$ and others are zero. Since we have already discretized $\mathbb{T}^d$ as finite lattice $L$, each $E_m$ is also finite set. Thus the square root of the indicator function is itself $\mathbf{1}_{E_m}^{1/2}=\mathbf{1}_{E_m}$. We only consider the case where $E_m$ has positive measure since when $E_m$ is zero-measure set, there is no difference from the representation with $M-1$ partitions.
		
		For the claim of Schauder basis, it suffices to show for $\forall f\in \mathcal{V}_{tri}^{K_{\max}}$, $\exists$  \text{unique} $c_{\textbf{k},m}\in \mathbb{C}$ s.t. $f(\textbf{x})=\sum_{m=1}^{M}\sum_{\textbf{k}\in [N]^d}c_{\textbf{k},m}e_{\textbf{k},m}(\textbf{x})$. Existence is already done in (\ref{corollary3.2}). For uniqueness, because $E_m$ are pairwise disjoint and with positive measure, $\sum_{m=1}^{M}\sum_{\textbf{k}\in [N]^d}c_{\textbf{k},m}e^{i\textbf{k}\textbf{x}}1_{E_m}(\textbf{x})=0 ~ a.e.~\Rightarrow~\sum_{\textbf{k}\in [N]^d}c_{\textbf{k},m}e^{i\textbf{k}\textbf{x}}=0,\textbf{x}\in E_m $. $\sum_{\textbf{k}\in [N]^d}c_{\textbf{k},m}e^{i\textbf{k}\textbf{x}}$ is a finite sum of trigonometric functions, which is analytic on positive measure set $E_m\subset \mathbb{T}^d$. Then by the uniqueness of the analytic continuation of complex analytic functions, $\sum_{\textbf{k}\in [N]^d}c_{\textbf{k},m}e^{i\textbf{k}\textbf{x}}=0~\textbf{x}\in E_m~\Rightarrow~ \sum_{\textbf{k}\in [N]^d}c_{\textbf{k},m}e^{i\textbf{k}\textbf{x}}=0~\textbf{x}\in \mathbb{T}^d$. Finally, by the uniqueness of Fourier series (\ref{Corollary1.1}), $c_{\textbf{k},m}=0$. Then we prove the uniqueness.
	\end{proof}
	\begin{remark} In the low temperature limit, even though ABLE is not a Schauder basis for $L^2(\mathbb{T}^d)$, it yields a basis for partial Fourier sum spaces. By retaining the first $[N]^d$ modes, where $N^d$ exceeds the dimension of the target space, ABLE admits uniqueness again. 
    
    In implementation, when $\mathbb{T}^d$ is modified to a discrete grid (lattice) $L=\times_{i=1}^{d}[N_i]$, the Fourier transformation becomes DFT/FFT. In this case, the uniqueness of series coefficient is apparently not satisfied. But whether this uniqueness of series coefficient is satisfied or not, the uniqueness of ABLE representation of an operator always holds, as shown in theorem (\ref{theorem_4}).
	\end{remark}
\end{theorem}

\begin{theorem}[High-temperature limit]\label{thm: high_temp}
	By setting $T\rightarrow \infty$, the ABLE series become $M$ identical Fourier series, and the ABLE neural operators reduce to  trivial over-parameterized FNO with $R(\textbf{k})=\frac{1}{M}\sum_{m}R(\textbf{k},m)$, and the kernel reduces to homogeneous kernel (dependent only on $\textbf{x}-\textbf{x}'$).
	\begin{proof}
		Since $T\rightarrow \infty$, $p(\textbf{x},m)=\frac{1}{M}$, and the ABLE neural operator becomes
		\begin{align}
			\mathcal{G}^{T\rightarrow \infty}_{ABLE}(f)(\textbf{x})&=\sum_{m=1}^{M}\sum_{\textbf{k}\in  L }e^{i\textbf{k}\textbf{x}}p(\textbf{x},m)^{\frac{1}{2}}R(\textbf{k},m)\sum_{\textbf{x}'} f(\textbf{x}')e^{-i\textbf{k}\textbf{x}'}p(\textbf{x}',m)^{\frac{1}{2}}\nonumber \\ &=\sum_{\textbf{x}'}\bigg[\sum_{\textbf{k}\in  L }e^{i\textbf{k}\textbf{x}}\big[\frac{1}{M}\sum_{m}R(\textbf{k},m)\big]e^{-i\textbf{k}\textbf{x}'}\bigg] f(\textbf{x}')=\sum _{\textbf{x}'}\bigg[\sum_{\textbf{k}\in  L }e^{i\textbf{k}\textbf{x}}\tilde{R}(\textbf{k})e^{-i\textbf{k}\textbf{x}'}\bigg] f(\textbf{x}')
		\end{align} 
	    For ABLE with crossing interactions, the statement is verified in Eqn.~(\ref{eq: high-temp}) of the main text with $\tilde{R}(\textbf{k})=\frac{1}{M}\sum_{m,m'=1}^{M}R(\textbf{k};m,m')$.
     \end{proof} 
\end{theorem}

\begin{theorem}[Implicit Statistical Ensemble and temperature controlled heterogeneity] \label{thm: temperature controlled heterogeneity}
For density in the learnable basis $p(x,m)=\text{Softmax}(\frac{MLP(x,m)}{T})$, 
\begin{enumerate}
\item The Softmax function constructs an implicit canonical statistical ensemble at each point
\item Modifying the temperature $T$ leads to the transition of heterogeneity through the Lipschitz continuity control: $\text{Lip}(p)\le\frac{1}{T}\text{Lip}(\text{MLP})$.
\end{enumerate}
\begin{proof}
First, we show the mathematical equivalence to canonical ensemble. Note that a canonical statistical ensemble $(p_i,\epsilon_i)$ satisfies Boltzmann distribution as $p_i= \frac{1}{Z}\exp(-\frac{1}{T}\epsilon_i)$ with partition function $Z=\sum_{i}\exp(-\frac{1}{T}\epsilon_i)$.
We correspond $p(x,m)=\text{Softmax}(\frac{MLP_\theta(x,m)}{T})$ to this $\frac{1}{Z}\exp(-\frac{1}{T}\epsilon_m(x))$ by $\epsilon_{1:m}:=-MLP_{\theta}: F(\mathbb{T}^d,\mathbb{R}^d) \rightarrow F(\mathbb{T}^d,\mathbb{R}^M)$, which predicts the $M$ energy (real-valued) for each point $\textbf{x}\in\mathbb{T}^d$. 
$\text{Softmax}:F(\mathbb{T}^d,\mathbb{R}^M)\rightarrow F(\mathbb{T}^d,\Delta^M), \epsilon_m(\textbf{x}) \mapsto p_m(\textbf{x})=\frac{1}{Z}\exp(-\frac{\epsilon_m(\textbf{x})}{T})$ with $Z=\Sigma_{m=1}^{M}\exp(-\frac{\epsilon_m(\textbf{x})}{T})$ and thus $\sum_{m=1}^{M}p_m(\textbf{x})=1$.
This procedure constructs a canonical ensemble $\{\epsilon_m(\textbf{x}),p_m(\textbf{x})\}_{m=1}^{M}$
for each spatial point $\textbf{x}$. 

Second, the transition of heterogeneity is from the control of the Lipschitz continuity. The compostion of an $L$-Lipschitz function $f$ and an $M$-Lipschitz function $g$, is an $LM$-Lipschitz function.
By $\text{p}=\text{Softmax}(\frac{MLP_\theta}{T})$, $\text{Lip}(p)\le \text{Lip}(\text{Softmax})\text{Lip}(\frac{MLP_\theta}{T})=\frac{1}{T}\text{Lip}(MLP_\theta)$. $T\rightarrow\infty$ forces the $p$ to be constant and homogeneous, and reducing $T$ improves heterogeneity due to a relaxed control of Lipschitz continuity. Until $T=0$, the Lipschitz constant is unbounded, and this induces the strongest heterogeneity. 
\end{proof}
\begin{remark}
Note that the ABLE ensemble is not regarded as an actual physical ensemble because it appears as field (locally defined) which couples with input fields through ABLE basis interaction, and the temperature is not a physical temperature but a hyperparameter which controls the sharpness of auxiliary function $p(x,m)$ and further the heterogeneity of the kernel.
\begin{enumerate}
\item Modifying temperature leads to crossover from the low-temperature hard-truncation (mode selection) regime to the high-temperature mode cooperation regime, which encourages emergent phenomena.
\item When setting the entire space $\mathbb{T}^d$ as finite grids $L=\times_{i=1}^{d}[N_d]$, this is exactly the lattice with periodic boundary conditions. And ABLE resembles constructing an energy ensemble for each lattice point $\textbf{n}\in L$, this energy ensemble describes that each point can have $M$ micro-states with its own energy $\epsilon_m$, and the probability of the observation of this micro-state is given by the canonical distribution $p_m$.
\item The ensemble could be locally defined, i.e. it is generally different at each lattice point, since each point can have its own dynamic.
\item From macro-observation, ABLE mechanism gives the superposition of contributions from all those micro-state.
\item When it comes to the low-temperature-limit case, the local dynamic of the point is strictly fixed to one of its microstates. The entire lattice becomes heterogeneous and ABLE learns how to select this single microstate. 
\item When it comes to the case with high temperature limit, all $M$ microstates are with the same energy and reduce to an identical single state with degeneracy $M$. The entire lattice becomes homogeneous. Currently, the kernel is exactly an over-parameterized FNO due to the linearity: $K_{ABLE}(x,z)=\sum_{m=1}^{M}\sum_k e^{-ikz}\sqrt{p(z,m)}R(k,m)\sqrt{p(x,m)}e^{ikx}=\sum_{m=1}^{M}\sum_k e^{-ikz}M^{-\frac{1}{2}}R(k,m)M^{-\frac{1}{2}}e^{ikx}=\sum_k[\frac{1}{M}\sum_m R(k,m)]e^{ik(x-z)}=\sum_k\tilde{R}(k)e^{ik(x-z)}=K_{FNO}(x-z)$.
\item  The degeneracy structure changes from $(1,M-1)$ to $(M)$. For a general $T\in \mathbb{R}_{>0}$, all microstates may have distinct energies, corresponding to a non-degenerate case.
\item The smooth transition from the heterogeneous regime to the homogeneous regime,  and the difference in the degeneracy of the $2$ endpoints leave ample room for representation learning and emergent phenomena. 
\end{enumerate}
\end{remark}
\end{theorem}

\subsection{ABLE dynamics from operator theory perspective}
\iffalse Coverage of Operator classification
\fi

\begin{theorem}[Finite-rank self-adjoint property of FNO and ABLE]\label{thm13:finite_rank} Suppose the multipliers $R(k)$ or $R(k;m,m')$ are truncated with only $\textbf{k}\le\textbf{K}$ modes, then both ABLE and FNO kernel integral operators are finite-rank self-adjoint operators. 
\begin{enumerate}
\item The maximum rank of the FNO operator family $\mathcal{T}_{FNO}$ satisfies $\max \text{rank}(\mathcal{T}_{FNO})= K^d$, 
\item The maximum rank of the ABLE operator family $\mathcal{T}_{ABLE}$ satisfies $K^d \le \max\text{rank}(\mathcal{T}_{ABLE})\le K^d M$, i.e., ABLE provides an extension of spectral neural operators that preserves finite-rank self-adjoint property.
\item $\max\text{rank}(\mathcal{T}_{ABLE})= K^d M$ iff $\dim\text{Span}\{e_{k,m}\}_{k=1,m=1}^{K^d,M}=K^dM $. (A trivial example is given by a Softmax modeled basis $e_{k,m}=p(x,m)^{\frac{1}{2}}e^{ikx}=(\frac{e^{\alpha_m x}}{\sum_{m=1}^{m'}e^{\alpha_{m'}x}})^{\frac{1}{2}}e^{ikx}$ with $\alpha_m\in \mathbb{R}^d$ and $|\alpha_m|$ all distinct.)
\end{enumerate} 
\begin{proof}
We first prove the self-adjoint property. For any operator $\mathcal{G}_{K}$ with kernel $K$, it suffices to verify the Hermitian condition $K(\textbf{x},\textbf{y})=K(\textbf{y},\textbf{x})^*$ (i.e, $\langle f, \mathcal{G}_K g\rangle = \langle \mathcal{G}_K f, g\rangle$ by 
\begin{align}
\langle f, \mathcal{G}_{K} g\rangle=\int_{E} d\textbf{z}\bar{f}(\textbf{z})\bigg[\int_E d\textbf{x} K(\textbf{z},\textbf{x})g(\textbf{x})\bigg]=\int_{E}d\textbf{x}\bigg[\overline{\int_{E} d\textbf{z}K(\textbf{x},\textbf{z})f(\textbf{z})}\bigg]g(\textbf{x})=\langle \mathcal{G}_{K} f,g\rangle. \end{align} ) and boundedness $\int_{E\times E}|K(x,z)|^2dxdz<\infty $.
FNO kernel is written as $K_{FNO}(x-y)=\sum_{\textbf{k}<\textbf{K}} R(\textbf{k})e^{i\textbf{k}(\textbf{x}-\textbf{y})}$ with $R(\textbf{k})=R(-\textbf{k})^*$ in practice, thus $K_{FNO}(\textbf{x}-\textbf{y})=K_{FNO}(\textbf{y}-\textbf{x})^*$ and its boundedness is given by
\begin{align}
&\int_{E\times E}|\sum_{k=1}^{K^d}R(k)e^{ik(x-z)}|^2dxdz=\int_{E\times E}\sum_{k,k'}R(k)R(k')^*e^{i(k-k')x-i(k-k')z} dxdz\nonumber\\& =\sum_{k,k'}R(k)R(k')\int_{E} e^{i(k-k')x}dx\int_{E} e^{-i(k-k')z}dz=\sum_{k,k'}R(k)R(k')^*\delta_{k,k'}=\sum_k|R(k)|^2<\infty.\label{eq: FNO bound}
\end{align}
Therefore, the operator is self-adjoint.
For ABLE, whether it is cross- or non-cross-modeled, it can be uniformly written as $K_{ABLE}(x,y)= \sum_{m,m'=1}^{M}\sum_{\textbf{k}\le \textbf{K}}p(\textbf{x},m)^{\frac{1}{2}}R(\textbf{k};m,m')e^{i\textbf{k}(\textbf{x}-\textbf{y})}p(\textbf{y},m')^{\frac{1}{2}}$. Since in practice $R(\textbf{k};m,m')=R(-\textbf{k};m,m')^*$ and $p(\textbf{x},m),p(\textbf{y},m')\in \mathbb{R}_{\ge 0}$,
the Hermitian relation $K_{ABLE}(\textbf{x},\textbf{y})=K_{ABLE}(\textbf{y},\textbf{x})^*$ still holds. And the boundedness is similarly given by the Cauchy-Schwarz inequality and normalization condition $\sum_m p(x,m)=1$:
\begin{align}
&\int_{E\times E}|\sum_{m,m'}^{M}\sum_{k=1}^{K^d}R(k;m,m')e^{ik(x-z)}p(x,m)^{\frac{1}{2}}p(z,m')^{\frac{1}{2}}|^2 dxdz\le \int_{E\times E}\bigg(\sum_{m,m'}^{M}|\sum_{k=1}^{K^d}R(k;m,m')e^{ik(x-z)}|^2\bigg)\nonumber\\&\times\bigg(\sum_{m,m'}^{M}p(x,m)p(z,m')\bigg)dxdz= \sum_{m,m'}^{M}\int_{E\times E}\big|\sum_{k=1}^{K^d}R(k;m,m')e^{ik(x-z)}\big|^2 dxdz=\sum_{m,m'}^{M}\sum_{k=1}^{K^d}\big|R(k;m,m')\big|^2 < \infty.
\end{align}
( The last step is analogous to Eqn.~(\ref{eq: FNO bound}), as the integral is only performed over $x$ and $z$. ) Combining the Hermitian and boundedness yields the self-adjoint property again.
\par
Second, for FNO, given $R(\textbf{k})$ truncated with the first $K^d$ modes, the kernel can be written as $K(x,y)=\sum_{k=1}^{K^d} u_k a_k(x)b_k(y)$ with $a_k(x)=e^{ikx}\in L^2(\mathbb{T}^d),b_k(x)=e^{-ikx}\in L^2(\mathbb{T}^d)$ and $u_k=R(k)\in \mathbb{C}$ Then, on one hand, $\forall h\in L^2(\mathbb{T}^d),\int K(x,y)h(y)dy=\sum_{k}(\int u_k b_k(y)h(y)dy)a_k(x)\in \text{span}\{a_k\}_{k=1}^{K^d}, i.e., \text{Im}(\mathcal{T}_{FNO})\subset \text{span}\{a_k\}_{k=1}^{K^d}$ whose dimension equals to $K^d$. And $\max\text{rank}(\mathcal{T}_{FNO})\le \dim\text{Im}(\mathcal{T}_{FNO})\le K^d$. On the other hand, there exists $\mathcal{G}_{FNO}\in\mathcal{T}_{FNO}$ whose image is the entire space $\text{span}\{a_k\}_{k=1}^{K^d}$ if $u_k\neq 0, \forall~k$. The Gram matrix $G_{ij}=\langle b_i,b_j\rangle=\delta_{ij}$, which is invertible for Fourier basis functions. Since $\forall~\textbf{c}=(c_1,\cdots c_{K^d}),  \exists~ h = \sum_{i=1}^{K^d}u_i^{-1}(G^{-1}\textbf{c})_i a_i=\sum_{i=1}^{K^d}u_i^{-1}c_ia_i=\sum_{k}u_i^{-1}c_i e^{ikx} $ satisfies $ c_k=\int u_kb_k(y)h(y)dy$, then $\text{Im}(\mathcal{G}_{FNO})=\text{Span}\{a_1,\cdots a_{K^d}\}$. And by $\dim\text{Im}(\mathcal{G}_{FNO})= \text{rank}(\mathcal{G}_{FNO})\le \max\text{rank}(\mathcal{T}_{FNO})$, $\max\text{rank}(\mathcal{T}_{FNO})=\dim\text{Im}(\mathcal{T}_{FNO})=\dim \text{Span}\{a_1,\cdots a_{K^d}\}=K^d$. 
\par
For ABLE, since $R(\textbf{k};m,m')$ is also truncated with the first $K^d$ modes, $K(x,y)=\sum_{m,m'=1}^{M}\sum_{k=1}^{K^d} u_{k,m,m'}a_{k.m}(x)b_{k,m'}(y)$ with $a_{k,m}(x)=e^{ikx}p(x,m)^{\frac{1}{2}}\in L^2(\mathbb{R})$, $b_{k,m}(x)=e^{-ikx}p(x,m)^{\frac{1}{2}}\in L^2(\mathbb{T}^d)$ and $u_{k,m,m'}\in \mathbb{C}$. Then, on one hand, $\forall h\in L^2(\mathbb{T}^d),\int K(x,y)h(y)dy=\sum_{k,m,m'}(\int u_{k,m,m'}b_{k,m'}(y)h(y)dy)a_{k,m}(x)\in \text{span}\{a_{1,1},\cdots a_{K^d,M}\}$, i.e., $\text{Im}(\mathcal{T}_{ABLE})\subset \text{Span}\{a_{k,m}\}_{k=1,m=1}^{K^d,M}$ whose dimension ranges from $K^d$ to $K^d M$. Thus, $\max\text{rank}(\mathcal{T}_{ABLE})\le K^dM$. The dimension range of $\text{Span}\{a_{k,m}\}_{k=1,m=1}^{K^d,M}$ is shown as follows:
\begin{enumerate}
\item For $\sum_{m} p(x,m)=1$, $\exists~ m_0\in [M]$ s.t. $p(x,m_0)\neq 0~a.e.~x\in B_{\epsilon}$ where $B_\epsilon$ is an arbitrarily small open set in $\mathbb{T}^d$ ($p(\cdot,m)$ is supposed to be continuous in $x$ for $\forall m$; in our Softmax modeling, each $p(x,m)$ attains strictly positive value, which is even stronger), then for series $a_{k,m_0}(x)=e^{ikx}p(x,m_0)^{\frac{1}{2}}, \textbf{k}\le \textbf{K}$, $\dim \text{Span}\{a_{k,m_0}\}_{k=1}^{K^d}=K^d$ because: if $c_k$ s.t. $\sum_k c_ka_{k,m_0}=0~ a.e.~B_\epsilon \Rightarrow \sum_kc_kp(x,m_0)^{\frac{1}{2}}e^{ikx}=0~a.e.~ B_\epsilon$, then by $p(x,m_0)\neq 0 ~a.e.~B_\epsilon$ we have $\sum_k c_k e^{ikx}=0 ~a.e.~B_\epsilon$, and further by the analytic continuation of this trigonometric sum: we have $\sum_k c_k e^{ikx}=0~a.e. \mathbb{T}^d$ and finally it implies $\Rightarrow c_k=0$. Thus $\dim \text{Span}\{a_{k,m}\}_{k=1,m=1}^{K^d,M}\ge\dim \text{Span}\{a_{k,m_0}\}_{k=1}^{K^d}=K^d$.  
\item The upper bound is because there are $K^d M$ components in total, there exist cases where the $K^d M$ components are linearly independent. A trivial case is given by $p(x,m)=\frac{e^{\alpha_m x}}{\sum_{m'=1}^{M}e^{\alpha_{m'}x}}$ where $\alpha_m\in \mathbb{R^d}$ and $|\alpha_m|$ are all distinct. In this case if $\sum_{k,m} c_{k,m} p(x,m)^{\frac{1}{2}}e^{ikx}=0~a.e.~\forall x\in \mathbb{T}^d $, then by analytic continuation this analytic function should satisfy $\sum_{k,m} c_{k,m} p(x,m)^{\frac{1}{2}}e^{ikx}=0~a.e.~\forall x\in \mathbb{R}^d $ in the entire $\mathbb{R}^d$ $\Rightarrow \sum_{m=1}^{M}e^{\frac{\alpha_m}{2} x}\sum_{k}c_{k,m}e^{ikx}=0$. Since $\alpha_m$ leads to different decaying rates, $\exists$ sufficient large $R_{m^*}$, $\sum_{k}c_{k,m^*}e^{ikx}=0$ should hold for the dominant $m^*$ for $ a.e.~|x|\ge R_{m^*}$, thus for $ a.e.~x\in \mathbb{R}^d$ by analytic continuation. Then since trigonometric functions are independent $\sum_{k}c_{k,m^*}e^{ikx}=0$, $c_{k,m^*}=0$. And $\sum_{m\neq m^*}e^{\frac{\alpha_m}{2} x}\sum_{k}c_{k,m}e^{ikx}=0$ holds with the dominant modes already proven to vanish, where we can impose the same operation to the second-dominant modes, and so forth all $c_{k,m}=0$. Therefore, this case strictly attains the upper bound. Actually, this also suggests that a Softmax model of $p(x,m)$ in ABLE encourages the independence between the bases, which also offers a platform for scale diversity. \label{item 2}  
\end{enumerate}
Until now, we have only verified $\max \text{rank}(\mathcal{T}_{ABLE})\le \dim \text{Span}\{a_{k,m}\}_{k=1,m=1}^{K^d,M}\le K^dM$, and it still remains to show $\max \text{rank}(\mathcal{T}_{ABLE})\ge K^d$. Since $\exists~m_0$ s.t. $\{a_{k,m_0}\}_{k=1}^{K^d}$ are linearly independent, the $K\times K$ sub-block of Gram matrix $G_{kl}^{(m_0)}=G_{(k,m_0),(l,m_0)}=\langle a_{(k,m_0)},a_{(l,m_0)}\rangle$ is invertible. There exists $\mathcal{G}_{ABLE}\in \mathcal{T}_{ABLE}$ with $u_{k,m_0,m_0}=R(k,m_0,m_0)\neq 0, \forall k$ and all other entries $(m,m')\neq (m_0,m_0)$ vanish. For any such $\mathcal{G}_{ABLE}$, we can prove that $\text{Span}\{a_{k,m_0}\}_{k=1}^{K^d}\subset \text{Im}(\mathcal{G}_{ABLE})$, which is to show $\forall \{c_{k}\}_{k=1}^{K^d}, \sum_{k}c_{k}a_{k,m_0}\in \text{Im}(\mathcal{G}_{ABLE})$. We denote the scalar inverse of $u_{k,m_0,m_0}$ as $(u_{k,m_0,m_0})^{-1}$ for $\forall~k$, and denote the inverse of the sub-block Gram matrix $G^{(m_0)}$ as $G_{sub}^{-1}$. Then for $\forall~\textbf{c}=(c_{1},\cdots c_{K^d}), \exists h = \sum_{k=1}^{K^d}(u_{k,m_0,m_0})^{-1}(G_{sub}^{-1}\textbf{c})_{k}a_{k,m_0}$ satisfies $ c_{(k,m)}=\sum_{m'}\int u_{k,m,m'}b_{k,m'}(y)h(y)dy$, thus $\text{Span}\{a_{k,m_0}\}_{k=1}^{K^d}\subset\text{Im}(\mathcal{G}_{ABLE})$. And by $\dim \text{Im}(\mathcal{G}_{ABLE})= \text{rank}(\mathcal{G}_{ABLE})\le \max\text{rank}(\mathcal{T}_{ABLE})$, $\max\text{rank}(\mathcal{T}_{ABLE})\ge\dim \text{Span}\{a_{k,m_0}\}_{k=1}^{K^d}=K^d$. 

Finally, for item $3$, we first make $e_{k,m}=a_{k,m}$ consistent with our $a,u,b$ notation. By $\text{Im}(\mathcal{T}_{ABLE})\subset \text{Span}\{a_{k,m}\}_{k=1, m=1}^{K^d,M}$, if $\text{dim}\text{Span}\{a_{k,m}\}_{k=1,m=1}^{K^d,M}<K^dM$, $\max\text{rank}(\mathcal{T}_{ABLE})< K^dM$. For the other direction, if $\{a_{k,m}\}_{k=1, m=1}^{K^d,M}$ are linearly independent and span a space of dimension $K^dM$, the entire $K^dM\times K^dM$ Gram matrix $G_{ij}=\langle a_i,a_j\rangle$, $i,j \in \{(k,m)\}_{k=1,m=1}^{K^d,M}$ is invertible, and its inverse is denoted by ${G}^{-1}$. Then there exists operator $\mathcal{G}_{ABLE} \in \mathcal{T}_{ABLE}$ whose coefficient $u_{k,m,m'}=R(k,m,m')$ is a non-degenerate matrix $u_k$ for $\forall k$ with the inverse matrix $u_k^{-1}$ (for non-crossing modeling, it's just $R(k,m)=R(k;m,m)\neq 0$ $\forall k,m $). Then We can prove similarly that for such operator $\mathcal{G}_{ABLE}$, $\text{Span}\{a_{k,m}\}_{k=1,m=1}^{K^d,M}\subset \text{Im}(\mathcal{G}_{ABLE})$, i.e., $\forall \{c_{k,m}\}_{k,m=1}^{K^d,M}, \sum_{k=1,m=1}^{K^d,M}c_{k,m}a_{k,m}\in \text{Im}(\mathcal{G}_{ABLE})$. For $\forall~\textbf{c}=(c_{(1,1)},\cdots c_{(K^d,m)}), \exists h =\sum_{k,m} \big[G^{-1}(u^{-1}\textbf{c})\big]_{(k,m)}a_{k,m}$ satisfies $ c_{(k,m)}=\sum_{m'}\int u_{k,m,m'}b_{k,m'}(y)h(y)dy$, thus $\text{Span}\{a_{1,1},\cdots a_{K^d,M}\}\subset\text{Im}(\mathcal{G}_{ABLE})$. Since $\dim\text{Im}(\mathcal{G}_{ABLE})=\text{rank}({\mathcal{G}_{ABLE}})\le \max\text{rank}(\mathcal{T}_{ABLE})$, $\max\text{rank}(\mathcal{T}_{ABLE})\ge K^dM$. Therefore, $\max \text{rank}(\mathcal{T}_{ABLE})=\dim\text{Im}(\mathcal{T}_{ABLE})=\dim \text{Span}\{a_{k,m}\}_{k=1,m=1}^{K^d,M}=K^dM$. For the trivial example, we have already proved that $e_{k,m}=(\frac{e^{\alpha_m x}}{\sum_{m'}e^{\alpha_{m'}x}})^{\frac{1}{2}}$ spans a $K^dM$ dimensional function space in the item~\ref{item 2} of the proof, thus under this setting, $\max\text{rank}(\mathcal{T}_{ABLE})=K^dM$.
\end{proof}
\begin{remark} The theorem does not require $R(k;m,m')$ to be different across $m$. Even if $R(k;m,m)=R(k)~\forall m$ and $R(k;m,m')=0~ \forall m\neq m'$, as long as the independence of the ABLE basis functions is sufficient, the ABLE extension can strictly recover the rank under the same truncation $K^d$, as $\text{Span}\{e_{1,1},\cdots e_{K^d,M}\}$ becomes a larger-space and these flexible bases encode extra heterogeneous and local information.
\end{remark}
\end{theorem}

\begin{theorem}\label{thm14:closure}
On a separable Hilbert space $\mathcal{H}$, the finite rank operator class $R(\mathcal{H})$ could densely approximate compact operator $K(\mathcal{H})$, i.e. $\overline{R(\mathcal{H})}=K(\mathcal{H})$.
\begin{remark}
Taking $\mathcal{H}=L^2(\mathbb{T}^d)$, then the kernel integration of ABLE and FNO is finite rank by theorem~\ref{thm13:finite_rank}. While $K\rightarrow \infty$, these kernel integrals could well approximate the compact operator class. 
\end{remark}
\end{theorem}

\begin{theorem}
The spectrum of a finite-rank self-adjoint operator only contains all the non-zero real eigenvalues $\lambda_1,\cdots,\lambda_r$, and $0$.
\end{theorem}
\begin{proof}
For a finite-rank self-adjoint operator $\mathcal{K}$, it is first a compact operator by theorem~(\ref{thm14:closure}), the operator has spectral decomposition by the spectral theorem of compact operators $\mathcal{K}(x,y)=\sum_{n=1}^{\infty}\lambda_n\overline{\phi}_n(x)\phi_n(y)$ with orthonormal basis $\phi_n$. And by definition of finite rank, there would be only finite numbers of non-zero eigenvalues among $\{\lambda_i\}_{i=1}^{\infty}$, thus they could be labeled as $\lambda_1,\cdots, \lambda_r$ and the remained part of the spectrum is just $\{0\}$. 
\end{proof}
\subsection{Approximation property of ABLE operator}
\begin{theorem}[BV-class function approximation theorem]\label{thm15: approximation} For target periodic bounded variation function $u\in BV(\mathbb{T}^d,\mathbb{C})$, with total variation $TV(u)=\sup \{\int f(\textbf{x})\nabla\cdot \textbf{g}(\textbf{x})d\textbf{x}:\textbf{g}\in C_{c}^1(\mathbb{T}^d,\mathbb{R}^d),|\textbf{g}(\textbf{x})|\le 1\}$, given input $\forall f\in L^2(\mathbb{T}^d)$, the approximation error estimates are as follows: 
	\begin{enumerate}
		\item For FNO with truncation $[K]^d$, $\exists ~u\in BV(\mathbb{T}^d), c\in\mathbb{R}_{+} $ s.t. \begin{align}c \frac{ \text{TV}(u)}{\sqrt{K}}\le\|u-\mathcal{G}^{K}_{FNO}(f)\|_{L^2(\mathbb{T}^d)}\end{align}
		\item For $d=1$ case, FNO with truncation $[K]^d$, $\exists~C\in\mathbb{R}_{+}$ s.t. \begin{align}\|u-\mathcal{G}^{K}_{FNO}(f)\|_{L^2(\mathbb{T}^d)}\le C \frac{ \text{TV}(u)}{\sqrt{K}}\end{align}
		\item For FNO with truncation $[K]^d$, $\exists ~c,C\in\mathbb{R}_{+} $ s.t. 
		\begin{align}c \frac{ \text{TV}(u)}{K}\le\|u-\mathcal{G}^{K}_{FNO}(f)\|_{L^1(\mathbb{T}^d)}\le C \frac{ \text{TV}(u)}{K}\end{align} 
		\item For $d=1, 2$ cases, ABLE with $M$ slices (i.e., $\chi=[M]$) and arbitrary truncation, 
		$\exists~ C\in \mathbb{R}_{+}$, s.t. \begin{align} \|u-\mathcal{G}^{M}_{ABLE}(f)\|_{L^2(\mathbb{T}^d)}\le C \frac{ \text{TV}(u)}{M^{\frac{1}{d}}}
		\end{align} 
		\item For $d=1$ case, ABLE with $M$ slices (i.e., $\chi=[M]$) and truncation $[K]^d$, $\exists~C\in\mathbb{R}_{+} $ s.t.  \begin{align}\|u-\mathcal{G}^{K,M}_{ABLE}(f)\|_{L^2(\mathbb{T}^d)}\le C \frac{ \text{TV}(u)}{\sqrt{KM}}\end{align} 
		\item For ABLE with $M$ slices (i.e., $\chi=[M]$) and arbitrary truncation, $\exists~C\in\mathbb{R}_{+} $, s.t.
		\begin{align}
			\|u-\mathcal{G}^{M}_{ABLE}(f)\|_{L^1(\mathbb{T}^d)}\le C \frac{ \text{TV}(u)}{M^{\frac{1}{d}}}
		\end{align}
	\end{enumerate}
	\begin{proof}
        Since $\mathbb{T}^d$ is equivalent to $[0,1]^d$ with periodic boundary conditions (up to a volume scaling), we conduct the proof on $\mathbb{T}^d$ with side length $1$ without loss of generality.

		For 1, to get the target $u$, regardless of the input feature $f$, the truncation restricts FNO to outputs of the form $\mathcal{G}_{FNO}(f)=\sum_{\textbf{k}\in [K]^d}\hat{u}_{\textbf{k}}e^{i\textbf{k}\textbf{x}}$, which is the partial sum of Fourier series $S_{K}(u)$. However, there is an example $u^0(\textbf{x})=\mathbf{1}_{x_1>\frac{1}{2}}\in BV(\mathbb{T}^d,\mathbb{C})$, whose Fourier coefficients are \begin{align}\hat{u}^0_{\textbf{k}}=\Pi_{i=2}^{d}\mathbf{1}_{k_i=0}[(\frac{1-(-1)^{k_1}}{-2\pi i k_1})\mathbf{1}_{k_1\neq0}+\frac{1}{2}\mathbf{1}_{k_1=0}]\sim O(\frac{1}{k_1}),\end{align} and $TV(u^0)=1$, 
		thus by Parseval's theorem $\|u-\mathcal{G}_{FNO}(f)\|^2_{L^2}\ge\|u-S_K(u)\|_{L^2}^2=\frac{1}{(2\pi)^2}\sum_{k_1= K+1}^{\infty}\frac{1}{k^2}\ge \frac{1}{2(2\pi)^2}\frac{1}{K}$,
		yielding 1. $\|u-\mathcal{G}_{FNO}(f)\|_{L^2}\ge c\frac{TV(u)}{\sqrt{K}} $. 

		For 2. FNO recover exactly $S_K(u)=\mathcal{G}^{K}_{FNO}(f)$ provided $f_\textbf{k}\neq 0$ whenever $\hat{u}_\textbf{k}\neq0$ for all $\|k\|_\infty\le K$. And since $|\hat{u}_\textbf{k}|=|\int_{\mathbb{T}^d}u(\textbf{x})\exp(-i\textbf{k}\textbf{x})d\textbf{x}|\le \int_{\mathbb{T}^{d-1}}dx_{\hat{i}}|\int dx_i\exp(-ik_ix_i)u(\textbf{x})| \le \frac{TV(u)}{2\pi\|k\|_{\infty}} $, the $\|u-\mathcal{G}_{FNO}(f)\|^2_{L^2}=\|u-S_K(u)\|_{L^2}^2\le \frac{1}{(2\pi)^2}\sum_{k_1,\cdots k_d= K+1}^{\infty}\frac{TV(u)^2}{\|k\|_{\infty}^2}$. For $d=1$, we directly get the upper bound by summing up $\sum_{k\ge K}^{\infty}\frac{1}{k^2}\sim O(1/K)$. For $d\ge 2$, this estimation diverges. Especially, for $d=2$, by Sobelev embedding theorem $BV(\mathbb{T}^d)\hookrightarrow L^p(\mathbb{T}^d)$, with $p=\frac{d}{d-1}$. Thus $BV(\mathbb{T}^2)\hookrightarrow L^2(\mathbb{T}^2)$, which ensures $S_K(u)\xrightarrow{L^2} u$, but this embedding does not give any convergence order. For $d\ge 3$, since $BV\not\subset L^2$, the $L^2$ norm $\|u-S_K(u)\|_{L^2(\mathbb{T^d})}$ may even diverge as $K\rightarrow\infty$, thus this estimation does not hold either.

		For 3. For $d=1$ case, the upper bound is from $c\frac{TV(u)}{K}\le\|S_K u-u\|_{L^1(\mathbb{T})}\le C\frac{TV(u)}{K}$. And for $d\ge 2$ cases, using the tensor-product structure $S_N=S_N^{(1)}\cdots S_N^{(d)}$, we decompose the error by a telescoping sum $\|S_N u-u\|=\|\sum_{i=1}^{d}(I-S_N^{(i)})u\|\le \sum_{i=1}^{d}\|(I-S_N^{(i)})g_i\|_1$ with $g_i=S_N^{(1)}\cdots S_N^{(i-1)}u$, and each term is bounded by the $d=1$ case.

		For 4. $\mathcal{G}_{ABLE}(f)(\textbf{x})=\sum_{m=1}^{M}\sum_{\textbf{k}\in  L }e^{i\textbf{k}\textbf{x}}p(\textbf{x},m)^{\frac{1}{2}}R(\textbf{k},m)\sum_{\textbf{x}'\in L} f(\textbf{x}')e^{-i\textbf{k}\textbf{x}'}p(\textbf{x}',m)^{\frac{1}{2}}$. When taking low temperature limit (so that $p(\textbf{x},m)=\mathbf{1}_{E_m}$) and restricting to the zero mode $\textbf{k}=\textbf{0}$ only, we obtain $\sum_{m=1}^{M}\mathbf{1}_{E_m}R(0,m)\sum_{\textbf{x}'\in L}f(\textbf{x}')$. From another perspective, since $u$ is a BV function, $TV(u)=V<\infty$, and we could partition the entire $\mathbb{T}^d$ into $m$ disjoint Lipschitz domains $\{E'_m\}_{m=1}^{M}$ such that $\cup_{m=1}^{M} E'_m=\mathbb{T}^d$, and with equal total variation $TV(u,E'_m)=\frac{1}{M}TV(u)$. Then, define $\tilde{u}=\sum_{m=1}^{M}1_{E'_m}\bar{u}_{E'_m}$ with $\bar{u}_{E'_m}=\frac{\int_{E'_m}d\textbf{x}u(\textbf{x})}{|E_m'|}$. For the case $d=1$, it is always possible to choose the $M$ Lipshitz regions to be connected respectively (i.e, intervals $E_m=(a_{m=1},a_{m}]$), because the TV is monotonic on $\mathbb{T}^1$. Then
		\begin{align}
			|u(x)-\bar{u}_{E_m'}|\le \frac{1}{|E_m'|} \int_{E_m'}|u(x)-u(x')|dx'\le TV(u,E_m')=\frac{V}{M}
		\end{align}
		and \begin{align}
			\int_{\mathbb{T}^d}|u(x)-\tilde{u}(x)|^2dx=\sum_{m=1}^{M}\int_{E_m}|u(x)-\bar{u}_{E_m}|^2dx\le \frac{V^2}{M^2}\cdot\sum_{m=1}^{M} |E_M'|=\frac{V^2}{M^2},
		\end{align}
		thus $\|u-\tilde{u}\|_{L^2}=(\int_{\mathbb{T}^d}|u(x)-\tilde{u}(x)|^2dx)^{\frac{1}{2}}\le \frac{V}{M}$. What makes ABLE effective is that the learnable density function in the basis $p(\textbf{x},m)$ renders the partition $E_m$ learnable as well, then even if we preserve the zero mode $k=0$ only, by coinciding $E_m=E_m'$, and letting $R(0,m)=\frac{u_{E_m}'}{\sum_{\textbf{x}\in L}f(\textbf{x}')}$, we obtain precisely the 
		simple function $\tilde{u}=\sum_{m=1}^{M}\bar{u}_{E'_m}\mathbf{1}_{E_m'}$. Then $\|u-\tilde{u}\|_{L^2}\le \frac{V}{M}\Rightarrow \|u-\mathcal{G}_{ABLE}\|_{L^2}\le\frac{V}{M}$. Since in 1-dim, $K$ and $M$ contributes as the same order to the cost, $O(\frac{1}{M})$ is better than FNO's $O(\frac{1}{\sqrt{K}})$.

		Moreover, for $d\ge 1$, by identifying the $M$ Lipschitz regions, the result from Jackson inequality from approximation theory \citep{Wojtaszczyk_2003} is
		\begin{align}
			\|u-\tilde u\|_{L^{\frac{d}{d-1}}}\le C_d TV(u) M^{-\frac{1}{d}}.
		\end{align} 
		Thus for $d=2$, we obtain $\|u-\tilde u\|_{L^{2}}\le C_2 TV(u) M^{-\frac{1}{2}}=O(M^{-\frac{1}{2}})$. This is still better than FNO since FNO did not even have an $L^2$ approximation in this case. For $d\ge3$, for the same reason that $BV\not\subseteq L^2$, we cannot even guarantee a finite $L^2$ error bound.

		For 5, we obtain another estimation for $L^2$ error of ABLE in just $d=1$ case, which is a corollary from 2. For BV function $u$, we still construct the $E_m$-partition with equal total variation $TV(u,E_m)=\frac{V}{M}$. Then $u=\sum_{m=1}^{M} 1_{E_m}(\textbf{x})u(\textbf{x})$. But now we could directly obtain the series from ABLE with truncation $K$, and $\chi=[M]$: \begin{align}\tilde{u}=\sum_{m=1}^{M}1_{E_m}\tilde{u}_m(\textbf{x})=\sum_{m=1}^{M}1_{E_m}\sum_{|\textbf{k}|\le K }\hat{u}_{k,m}\exp(2\pi i \textbf{k}\textbf{x})\end{align}
		where $\hat{u}_{k,m}=\langle u1_{E_m}, e^{i\textbf{k}\textbf{x}}\rangle$.
		We directly compute the $L^2$ approximation error for each segment $E_m$ as 
		$\int_{E_m}|u-\tilde{u}(\textbf{x})|^2d\textbf{x}\le \int|u1_{E_m}-S_K(u1_{E_m})|^2d\textbf{x}=\sum_{|\textbf{k}|\ge K}|\hat{u}_{\textbf{k},m}|^2$. By the property $|\hat{u}_{k,m}|\le \frac{TV(u1_{E_m})}{\|\textbf{k}\|_{\infty}}=\frac{TV(u,E_m)}{\|\textbf{k}\|_{\infty}}=\frac{V}{M\|k\|_{\infty}}$, we obtain only for $d=1$ case, $\sum_{\textbf{k}\ge K}\frac{1}{\|\textbf{k}\|_{\infty}^2}\sim O(\frac{1}{K})$. Hence, $\int_{E_m}|u-\tilde{u}(\textbf{x})|^2d\textbf{x}\le C_1 \frac{V^2}{M^2K} $. Then $\|u-\tilde{u}\|_{L^2}=(\sum_{m=1}^{M}\int_{E_m}|u-\tilde{u}|^2dx)^{\frac{1}{2}}\le C\frac{V}{\sqrt{MK}}=O(\frac{1}{\sqrt{MK}})$

		For 6. Directly use $\|u-\tilde{u}\|_{L^1}\le\|u-\tilde u\|_{L^\frac{d}{d-1}}\le C_d TV(u)M^{-\frac{1}{d}}$ for $d\ge 2$.  For $d=1$, use $L^2$-norm in $4$ to bound.
	\end{proof}
\end{theorem}

\begin{proposition}[Function approximation to Operator approximation]
	If $\|u-\mathcal{G}_\theta(f)\|_{L^p}\sim O(K^{-s})$ then $\|\mathcal{G}^{\dagger}-\mathcal{G}_{\theta}\|_p\sim O(K^{-s})$ 
	\begin{proof}
		Since $\|\mathcal{G}^\dagger -\mathcal{G}_\theta\|_p=\sup_{\|f\|_p=1}\|(\mathcal{G}^\dagger-\mathcal{G}_\theta) f\|_p=\sup_{\|f\|_p=1}\|u-\mathcal{G}_\theta f\|_p $, then just take $\sup$ on both side of  $\|u-\mathcal{G}_\theta(f)\|_{L^p}\sim O(K^{-s})$
	\end{proof}
\end{proposition}
\begin{proposition}[Contrast of FNO's and ABLE's approximation ability]
	The budget–accuracy results for FNO's and ABLE's approximation are summarized in Tab.~\ref{tab: function approximation property}.
	\begin{table}[htbp]
		\centering
		\caption{BV-class approximation property}
		\label{tab: function approximation property}
        \begin{threeparttable}
		\begin{tabular}{lccc}
			\toprule
			Function class &  BV ($L^1$ error) & BV ($L^2 error$) & $H^1$ ($L^2$ error) \\
			\midrule
			FNO~(d=1) & $O(K^{-1})$ &  $O(K^{-\frac{1}{2}})$ &  $O(K^{-1})$
			\\
			ABLE~(d=1) & $O(M^{-1})$ & $O(M^{-1})$ or $O((KM)^{-\frac{1}{2}})$ & $O(K^{-1})$
			\\
			FNO~(d=2) & $O(K^{-\frac{1}{2}})$ & - & $O(K^{-1})$
			\\
			ABLE~(d=2) & $O(M^{-\frac{1}{2}})$ & $O(M^{-\frac{1}{2}})$ & $O(K^{-1})$
			\\
			FNO~(d$\ge$3)\tnote{1} & $O(K^{-1})\sim O\big({(C'/\log C')}^{-\frac{1}{d}}\big)>O(C'^{-\frac{1}{d}})$ & - & $O(K^{-1})$
			\\
			ABLE~(d$\ge$3)\tnote{2} & $O(M^{-\frac{1}{d}})\sim O(C^{-\frac{1}{d}})$ & - & $O(K^{-1})$
			\\
			\bottomrule
		\end{tabular}
        \begin{tablenotes}
        \small
        \item[1] FNO cost $O(C')$ counted as $O(N\log N)\sim O(K^d\log K^d)$.
        \item[2] ABLE cost $O(C)$ counted as $O(M)$. Under same budget-$O(C)\sim O(C')$, then $O(M^{-\frac{1}{d}})< O(K^{-1})$.
        \end{tablenotes}    
    \end{threeparttable}
	\end{table}
    
	\begin{enumerate}
		\item
		Since the ABLE family strictly contains the FNO family, all approximation orders achievable by FNO can also be achieved by ABLE. But due to its adaptive basis learning mechanism, ABLE may attain better approximation property.
		\item 
		All currently known additional approximation orders of ABLE (except for the second variant with $d=1$ in statement 5 of theorem~\ref{thm15: approximation}) are obtained by setting $K=1$, and dependent on $M$ only. Thus their computational cost is just $O(M)$ where ABLE does not require FFT anymore and relies solely on the learnable partition.
		\item 
		Under its low temperature limit, ABLE already achieves better approximation orders than FNO  for BV functions in dimension $d=1$ and $d=2$. (This is purely purely a consequence of the ABLE extension.)
		\item 
		In $d\ge 3$ case for BV class, FNO and ABLE attains same approximation order $O(K^{-1})$ (with $M\sim K^d$ for ABLE). Yet FNO still requires $O(NlogN)$ operations with $N\ge K^d$,  while ABLE-retaining only the zero mode-needs only cost $O(M)$ to attain the same approximation order $O(\frac{1}{M^{\frac{1}{d}}})\sim O(\frac{1}{K})$. Therefore, denote  ABLE's cost $O(C)\sim O(M)\sim O(K^d)$. To match the accuracy, FNO incurs a cost of $O(C')\sim O(N\log N)\ge O(C\log C)$.
		\item $H^1\subseteq BV$. For $H^1$ case, ABLE inherits the approximation properties of FNO; but due to flexible data-driven allocation of parameters and adaptive basis learning mechanism, it exhibits stronger expressivity and better parameter control.
		\item ABLE may achieve a better approximation order for $d\ge3$ BV class, since there exists functions in $BV\cap (H^{1})^{c}$ with favorable local regularity that admit a better convergence rate $O(M^{-1})<O(M^{-\frac{1}{d}})$.
	\end{enumerate}
\end{proposition}
\begin{remark}
	ABLE achieves better approximation order than FNO under the same computational budget for BV class, which is a larger class than $H^1$. More PDE with $BV$ but not $H^1$ solution can be approximated better by ABLE theoretically, and if the answer to the famous millennium problem is that $3D$-Navier Stokes has no $H^1$ solution, then ABLE would have great potentials and advantages for $3D$ NS simulations.
\end{remark}

\section{Problem setting and datasets}
\subsection{Problem setup and dataset} 
%This section provides additional details on the datasets, PDE benchmarks, and experimental setups used in the main paper. All the tasks are trained with $1000$ samples and tested on $200$ samples. 
This section provides supplementary details regarding the datasets, PDE benchmarks, and experimental configurations described in the main text. Unless otherwise specified, all experiments are conducted with $1000$ training samples and evaluated on $200$ test samples.
\subsubsection{Burgers}
The 1D Burgers' equation is a non-linear PDE, which takes the following form
\begin{align}
	&\partial_t u(x,t)+u(x,t)\partial_xu(x,t)=\nu\partial_{xx} u(x,t), x\in(0,1),t\in(0,1]\nonumber\\&
	u(x,0)=u_0(x), x\in (0,1)
\end{align} where $u_0\in L^2_{periodic}((0,1),\mathbb{R})$ and is sampled from the same Gaussian random field as in FNO \citep{DBLP:journals/corr/abs-2010-08895}.
We have three benchmark settings with  viscosity values $\nu = 0.1, 0.01, 0.001$, where the low viscosity $\nu =0.001$ regime is the most challenging. The ground truths of the $\nu=0.01$ and $\nu=0.001$ regimes are presented in Fig. \ref{fig: Burgers_xt}, demonstrating how the  sharpness and strength of the shock wave vary with $\nu$.
\begin{figure*} % t!
\centering
\includegraphics[trim = 0cm 8.5cm 0cm 0.5cm,clip,width=1\textwidth]{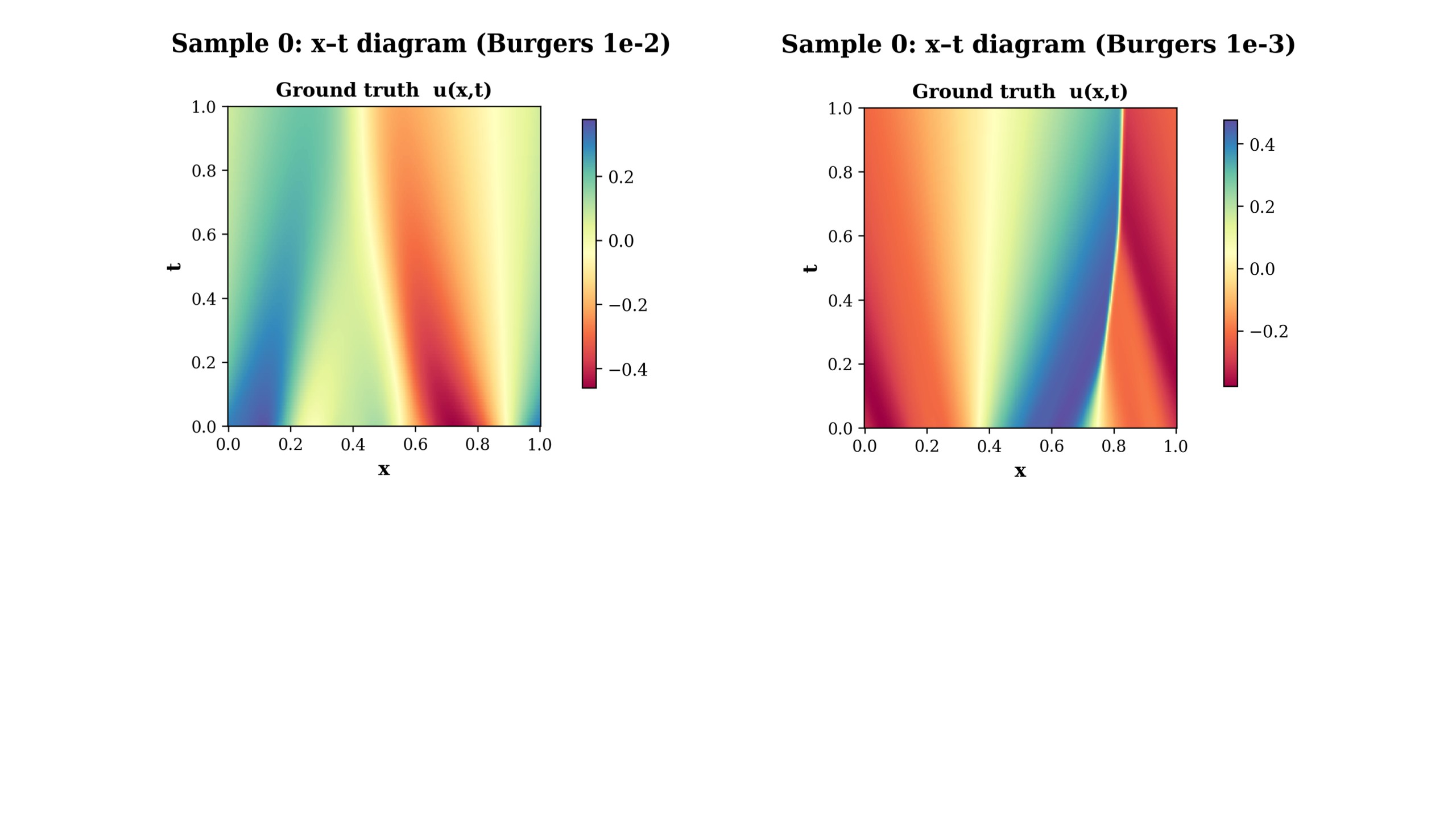}\\
\caption{
\textbf{Burgers dataset of different viscosity.} Left: $\nu=10^{-2}$, right: $\nu=10^{-3}$. The shape of the solution with a lower viscosity ($\nu=10^{-3}$) is much sharper, and the shock wave is stronger. The standard training setting for all baselines is only taking the input value $u(\textbf{x},t=0)$ and directly output the $u(\textbf{x},t=1)$, without the predictions of intermediate steps.}
\label{fig: Burgers_xt}
\end{figure*} 
    
\subsubsection{Darcy}
The 2D Darcy Flow equation is a set of elliptic PDE (\ref{Darcy}) with zero Dirichlet boundary condition and forcing function $f\in L^2((0,1)^2;\mathbb{R})$. 
\begin{align}
	&-\nabla\cdot(a(\textbf{x})\nabla u(\textbf{x}))= f(\textbf{x}), \textbf{x}\in (0,1)^2\nonumber\\&
	u(\textbf{x})=0, \textbf{x} \in \partial (0,1)^2 \label{Darcy}
\end{align}
We focus on the operator learning from the diffusion function $a\in L^\infty((0,1)^2;\mathbb{R}_+)$ to the solution $u\in H_{0,per}^1((0,1)^2,\mathbb{R}_+)$. Datasets are generated from the same distribution as FNO \citep{DBLP:journals/corr/abs-2010-08895} with resolution $421\times 421$, and downsampled to a resolution of $141\times141$ for training and testing.
\subsubsection{Navier Stokes}
The Navier Stokes equation takes the form
\begin{align}
	&\frac{\partial \omega}{\partial t} + (\textbf{u}\cdot \nabla)\omega = \nu \Delta \omega + f(x), \textbf{x}\in (0,1)^2, t\in (0,T]\nonumber\\& \nabla\cdot \textbf{u}=0, t \in [0,T]\nonumber\\&
	\omega(\textbf{x},0)=\omega_0(\textbf{x}), \textbf{x}\in (0,1)^2
\end{align} with velocity field $\textbf{u}:L^2([0,T],H_{per}^r((0,1)^2,\mathbb{R}^2)), t\rightarrow \textbf{u}(\cdot, t)$ and initial vorticity $\omega_0\in L^2_{per}((0,1)^2,\mathbb{R})$. The task is to predict the exact operator $\mathcal{G}^\dagger: C([0,T_{in}],H^r((0,1)^2,\mathbb{R}))\rightarrow C([T_{in},T],H^r((0,1)^2,\mathbb{R})),u(\cdot,[0,T_{in}])\rightarrow u(\cdot,[T_{in},T]) $. The vorticity field is much more difficult to predict than the velocity. Since in NS equation for velocity $\frac{\partial\textbf{u}}{\partial t}+(\textbf{u}\cdot\nabla)\textbf{u}=-\frac{1}{\rho}\nabla p + \nu \nabla^2\textbf{u}$, the pressure term stabilizes the velocity field. However, while acting on $\nabla\times$ to obtain the vorticity NS equation, this term vanishes, and the nonlinear term $(\textbf{u}\cdot\nabla)\omega$ makes the turbulence unstable and challenging to predict.

The dataset is also from \citep{DBLP:journals/corr/abs-2010-08895}, and we conduct experiments on the most challenging regime ($\nu=10^{-5}$) for long-term rollout from $T_{in}\in[0,10]$ to $T\in[11,20]$.

\subsection{Algorithms and implementation details}
\begin{figure}
\centering
\includegraphics[width=1.0\textwidth]{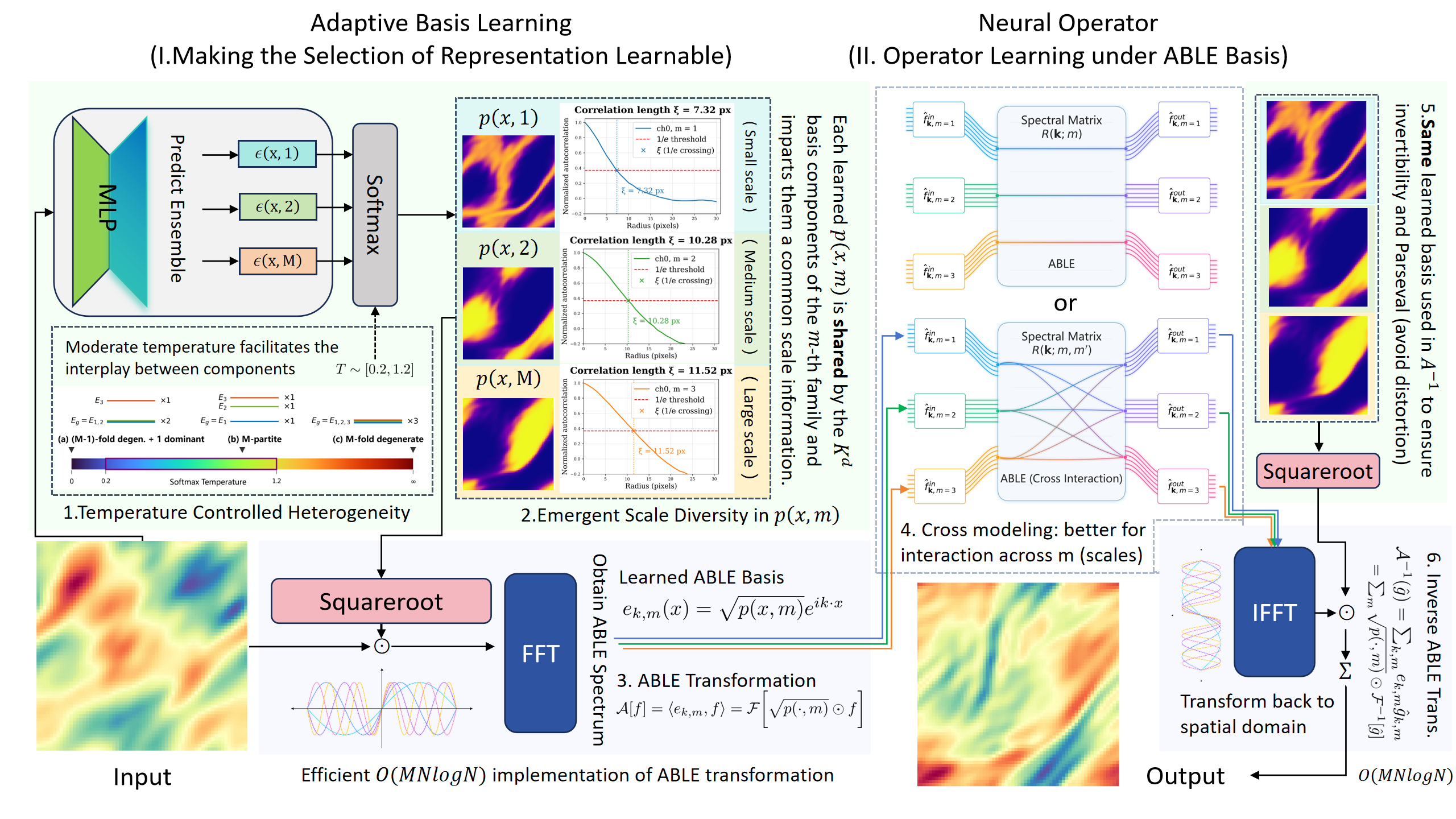}
\caption{Detailed graphical description of ABLE-NO layer in implementation}
\label{fig:FIG_TEASER2}
\end{figure}
\begin{algorithm} % [t]
\caption{Forward Pass of the Fourier Neural Operator (FNO)}
\label{alg:FNO_forward}.
\KwIn{
    $f(\mathbf{x})\in \mathbb{R}^{c_{\mathrm{in}}}$ for all $\mathbf{x}\in\Omega$; \\
    Layers $L$; hidden channels $C$; mode truncation $K$; \\
    Learnable parameters $\{P, b_P, Q, W_\ell, b_\ell, R\}$
}

\KwOut{$g(\mathbf{x}) \in \mathbb{R}^{C_{\mathrm{out}}}$ for all $\mathbf{x}$}

$\mathbf{v}_0 = P f + b_P$ \tcp*{$P: \mathbb{R}^{c_{\mathrm{in}}} \to \mathbb{R}^{C}$ lifting}

\For{$\ell = 1$ \KwTo $L$}{
    \tcp{1. Fourier transformation}
    $\hat{\mathbf{v}}_{\ell-1,k}^{(c)} = \mathrm{FFT}\!\left(\mathbf{v}_{\ell-1}^{(c)}(\mathbf{x})\right)$ \tcp*{Fourier transform}
    
    \tcp{2. Spectrum multiplication of $R(k)$}
    $\hat{\mathbf{v}}^{\mathrm{spec},(c)}_{\ell,k} = \sum_{c'} R_{c,c'}(k)\,\hat{\mathbf{v}}_{\ell-1,k}^{(c')}$ \;
    
    \tcp{3. Inverse Fourier transformation}
    $\mathbf{v}^{\mathrm{spec}}_\ell(\mathbf{x}) = \mathrm{IFFT}(\hat{\mathbf{v}}^{\mathrm{spec},(c)}_{\ell,k})$ \tcp*{back to physical space}
    
    \tcp{4. Spatial branch \& fusion}
    $\mathbf{v}_\ell = \mathrm{GeLU}\!\left(\mathbf{v}^{\mathrm{spec}}_\ell + W_\ell\mathbf{v}_{\ell-1} + b_\ell\right)$ 
}

$g = Q\,\mathbf{v}_L$ \tcp*{$Q: \mathbb{R}^{C} \to \mathbb{R}^{C_{\mathrm{out}}}$ 2-layer MLP projection}
\Return{$g$} \;
\end{algorithm}
The forward pass algorithm of ABLE is already presented in the main text Tab.~\ref{alg:ABLE_forward}. Correspondingly, the algorithm for FNO is presented here in Tab.~\ref{alg:FNO_forward} for comparison. A detailed illustrative figure of the ABLE architecture is also provided here in Fig.~\ref{fig:FIG_TEASER2}.
\subsection{MLP-Softmax design}.\label{Sec: Appendix B3}
\textbf{1D-Burgers}
We first use convolution with fixed kernel $[-0.5,1,-0.5]$ and $[0.5,0,-0.5]$ to extract the first and second order finite differentiation $f',f''$ for each input channel $f$. We then concatenate them with the input channel and an all-one channel, then we get $3C+1$ channels $[f,f',f',1]$. (This embedding extraction design is optional in our ABLE architecture—a flexibility not available to other models, as they lack the entire corresponding Softmax-MLP module unique to ABLE.) Then, the $3C+1$ features are fed into a lightweight 4-layer MLP with structure $[3C+1,h,\frac{h}{2},h,M]$ which outputs $M$ channels. Hidden dimensions are $[h,\frac{h}{2},h]$, with a residual connection from the first hidden layer to the third. Activation functions are selected as the SiLU function. The hidden layer dimension $h$ can be tuned to larger values, but not excessively, to avoid redundant and unstable control. Finally, we apply a Softmax with temperature $T$ to normalize $p_\theta(x,m)$ as a positive distribution, and $T$ is an important hyper-parameter that serves as a control for the heterogeneity of the ABLE kernel. 
\par
The full architecture mirrors that of FNO, as the ABLE layer could serve as a drop-in replacement for the Fourier layer. Both models consist of 4 layers with mainstream width $C=64$, ensuring a fair comparison.
\par
\textbf{2D-Darcy}. We just employ a 2-layer MLP with hidden dimension $h=24$ defined on $C=32$ inputs channels. For Darcy, the finite-difference extraction appears redundant; we therefore remove it and use only the 2-layer MLP. Both the Burgers and Darcy design construct a shared learnable basis across all channels.
\par
\textbf{2D-NS}. The task is to predict $u(x,y,T_{in}+1)$ from multi-step input $\{u(x,y,t)\}_{t=0}^{T_{in}}$, with further autoregressive rollout for $T=10$ additional steps. Recurrent FNO2D lifts all input vorticity field at different time steps via a collective lifting layer, produces channels that mix information across time steps. Thus a shared adaptive learnable basis for all channels is not considered the best design. Fortunately, we could design the MLP with architecture $[C,h,MC]$ instead of $[C,h,M]$, so that for each channel $c\in[C]$, $p_\theta(x,m)$ becomes $p_\theta(x,m,c)$- maintaining the same output size for each channel while allowing flexible, channel-specific parameters. The parameter increase is still negligible compared to the spectral multiplier $R(k,m)$ or $R(k;m,m')$, while expressivity is largely enhanced. ABLE's bipartite and cross-position interaction follow the underlying logic that features would propagate to different spatial locations over time.

\subsection{General enhancement from ABLE}\label{Sec: Appendix B4}
ABLE could directly compete with improvements made in different directions. However, as a general spectral refinement method that only incurs a linear increase in computational cost, it is expected to improve all methods with spectral branches. This enables ABLE to cooperate with-rather than compete against-most improvements in matrix modeling, architecture design, or even methods employing different spectral-spatial interaction mechanisms.

\textbf{U-ABLE.}
U-FNO \citep{Wen2021UFNOA} adds a UNet in parallel to each layer in FNO. And each layer is then comprised of 3 branches: Fourier, MLP and UNet. We directly substitute the Fourier branch with ABLE to obtain U-ABLE. 

\textbf{ABLE-SAOT.}
SAOT \citep{Zhou2025SAOTAE} combines a wavelet attention branch with an FNO branch in each layer, making it also a hybrid architecture model. Keeping all other settings the same as SAOT, we directly replace the FNO branch to with ABLE ($M=2$) for Darcy and ABLE-cross ($M=3$) for NS to obtain ABLE-SAOT.

\textbf{ABLE-HPM.}
HPM \citep{Yue2024HolisticPS} layer is a Laplacian-basis-based representation 
\begin{equation} \mathcal{T}_{HPM}(f)_k=\sum_{i=1}^{N}\text{Softmax}\big(\text{MLP}_\theta(i)\big)_k e_k(i)f(i)=\bigg(\big[\text{Softmax}(\text{MLP})\odot \Phi\big]^Tf\bigg)_k
\end{equation} where Laplacian eigenbasis $e_k(i)$ are precomputed through $\Delta_{ij}e_k(j)=\lambda_ke_k(i)$. However, $e_k(i)$ also admits an adaptive-basis version $e_k(i)\sqrt{p(i,m)}$, and giving rise to ABLE-HPM:
\begin{equation}
\mathcal{T}_{ABLE+HPM}(f)_{k,m}=\sum_{i=1}^{N}\text{Softmax}\big(\text{MLP}_\theta(i)\big)_k e_k(i)\sqrt{p(i,m)}f(i)=\bigg(\big[\text{Softmax}(\text{MLP})\odot \Phi\big]^T\big[\sqrt{p(\cdot,m)} f\big]\bigg)_{k,m}
\end{equation}
While ABLE and HPM appear similar in their use of probability modeling, they follow entirely different improvement directions: HPM focuses on spectrum selection for each spatial point under a basis given a priori, whereas ABLE is designed to make the fixed basis itself refinable and learnable, thereby offering complementary flexibility to HPM.
\par
In implementation, HPM uses $8$ layers with $H=8$ heads and head dimension $d_h=256/8=32$, is already equipped with a flexible temperature in its spectral selection part. With all these settings fixed, we integrate adaptive basis learning into HPM with $p$ modeled as a 2-layer MLP $[H\times d_h,h,M]=[256,16,3]$ and the ABLE basis temperature $T=0.8$. (There are 2 temperatures now, but the ABLE temperature is fixed.) This actually becomes an ABLE-cross ($M=3$) model defined on HPM's latent spectrum.

\section{Full experimental results with hyperparameters}
\textbf{Burgers.} 
The main results together with the hyperparameter settings are shown in Tab\ref{tab: Appendix 1D Burgers}.
\begin{table}[htbp]
	\centering
	\caption{Results on 1D Burgers equation $s=256$}
	\label{tab: Appendix 1D Burgers}
	\begin{threeparttable}
		\resizebox{\textwidth}{!}{%
			\begin{tabular}{lccccc}
				\toprule
				Baseline & Viscosity &Training velocity & Flops & Best learning rate & Relative l2 loss \\
				\midrule
				\multirow{3}{*}{DeepONet\tnote{1}~\citep{lu2019deeponet}}
				& 0.1 & $1.59\pm 0.02s/50ep$  & 17.01M & 8e-4(500000) & 0.0235  \\
				& 0.01 & $1.59\pm 0.02s/50ep$ & 17.01M & 1e-3(500000) & 8.49e-3 \\
				& 0.001 & $1.59\pm 0.02s/50ep$ & 17.01M & 1e-3(500000) & 0.202 \\
				\addlinespace[0.1em]
				\multirow{3}{*}{Galerkin~\citep{Cao2021ChooseAT}} 
				& 0.1 & $2.35\pm0.13$ & 78.78M & 7e-4 & 6.71e-4 \\
				& 0.01 & $2.35\pm0.13$ & 78.78M & 8e-4 & 7.10e-4 \\
				& 0.001 & $2.35\pm 0.13$ & 78.78M & 1e-3 & 1.26e-2 \\
				\addlinespace[0.1em]
				\multirow{3}{*}{FNO~\citep{DBLP:journals/corr/abs-2010-08895}} 
				& 0.1 & $0.34\pm 0.02$& 6.36M & 4e-3 & 6.43e-4 \\
				& 0.01 & $0.34\pm 0.02$ & 6.36M & 3e-3 & 5.96e-4 \\
				& 0.001 & $0.34\pm 0.02$ & 6.36M & 7e-3 & 6.12e-3\\
				\addlinespace[0.1em]
				\multirow{3}{*}{SNO~\citep{Fanaskov2022SpectralNO}} 
				& 0.1 & $0.0875\pm 0.0006$ & 1.48M & 1e-3 & 0.0348 \\
				& 0.01 & $0.0875\pm 0.0006$ & 1.48M & 3e-3 & 0.0196  \\
				& 0.001 & $0.0875\pm 0.0006$ & 1.48M & 2e-3 & 0.196 \\
				\addlinespace[0.1em]
				\multirow{3}{*}{GF-NO\tnote{2}~\citep{QI2024106239}} 
				& 0.1 & $0.64\pm 0.03$ & 6.36M & 2e-3(5000) & 4.8e-4 \\
				& 0.01 & $0.64\pm 0.03$ & 6.36M & 8e-4(1500);2e-3(5000) & 5.62e-4; 3.77e-4 \\
				& 0.001 &$0.64\pm 0.03$ & 6.36M & 2e-3(5000) & 5.07e-3 \\
				\addlinespace[0.1em]
				\multirow{3}{*}{WNO~\citep{Tripura2022WaveletNO}} 
				& 0.1 & $4.04\pm0.21$& 6.36M & 1e-3 & 0.100 \\
				& 0.01 & $4.04\pm0.21$ & 6.36M & 5e-4 & 0.055 \\
				& 0.001 & $4.04\pm0.21$ & 6.36M & 5e-4 & 0.195 \\
				\addlinespace[0.1em]
				\multirow{3}{*}{ABLE~(Ours)} 
				& 0.1 & $0.68\pm 0.05$& 12.14M & 2e-3 & 3.90e-4\\
				& 0.01 & $0.81\pm 0.04$& 12.17M & 3e-3(500);2e-3(1500)  & 3.82e-4; 3.25e-4 \\
				& 0.001 & $0.81\pm 0.04$& 12.17M & 3e-3 & 4.65e-3\\
				\bottomrule
			\end{tabular}
		}
		\begin{tablenotes}
        \small
        \begin{minipage}{0.9\linewidth}
			\item [1] DeepONet uses full-batch optimization with $B=1000$, where $50$ epochs are computationally equivalent to one epoch of FNO's mini-batch training (batch size $B=20$). The training times for these equivalent budgets are compared. 
			\item [2]The original GF-NO implementation sets the training epochs to $5000$, as it converges slowly even with a higher learning rate. In contrast, ABLE with $1500$ epochs already surpass GF-NO's $5000$-epoch result, and ABLE with just $500$ epochs also surpass GF-NO's $1500$ epoch result. To ensure fairness across all other baselines, we report ABLE's results at $500$-epoch and GF-NO's at $1500$ epoch in the main Tab.~\ref{tab: 1D Burgers new} for $\nu=10^{-2}$ regime.
        \end{minipage}
		\end{tablenotes}
	\end{threeparttable}
\end{table}
More precisely, for viscosity setting $\nu=0.1,0.01,0.001$, $M=2,3,8$ achieve the best results among the ABLE (non-crossing) families (kernel modeled as $R(k,m)$), while $M=2,3,3$ achieves the best results among ABLE-cross families (kernel modeled as $R(k;m,m')$). And for $\nu=10^{-1},10^{-2}$, our model performs better without Gelu activation in the outermost $2$ layers. We conduct ablation studies on the basis size $M$ (see Tab.~\ref{tab: Appendix 1D Burgers Ablation--basis size}) and the basis temperature $T$ (see Tab.~\ref{tab: Appendix 1D Burgers Ablation--basis temperature}), the plot for $\nu=10^{-3}$ is also included in main paper (see Fig.~\ref{fig:ablation}). The optimal width of MLP $h$ is also reported in the tables. For most problems, $M=2$ already yields substantial improvements, but increasing $M$ further enhances performance. After fixing the optimal $M$, we can adjust the temperature and performance is generally better at moderate temperatures (also shown in Fig.~\ref{fig:ablation}), which suggests that non-degenerate dynamics among the $M$ basis component families are beneficial. 
\begin{table}[t!]
	\center
	\caption{1D Burgers Ablation--basis size $M$}
	\label{tab: Appendix 1D Burgers Ablation--basis size}
	\begin{tabular}{ccccccc}
		\toprule
		M & Training velocity(s/epoch) & Flops&  h & $\nu=0.1$ &  $\nu=0.01$ & $\nu=0.001$\\
		\midrule
		$M=1$(FNO) & $0.34\pm0.02$ & 6.36M & - & 6.43e-4 & 5.96e-4 & 6.12e-3 
		\\
		$M=2$ & $0.83\pm0.06$ & 10.21M & 16 & \underline{4.86e-4}  & 5.14e-4 & 5.39e-3
		\\
		$M=3$ & $0.85\pm0.04$ & 10.23M & 16 & 5.54e-4 & \underline{4.50e-4} & 5.30e-3
		\\
		$M=4$ & $0.87\pm0.05$& 14.27M & 32 &  4.89e-4 & 4.92e-4 & 5.10e-3
		\\
		$M=6$ & $0.78\pm0.08$ & 14.34M &  32 & 5.06e-4 & 4.87e-4 &  5.58e-3
		\\
		$M=8$ & $0.85\pm0.06$ & 14.41M & 32 &
		5.82e-4 &  5.21e-4 & \underline{5.03e-3}
		\\
		$M=2$(cross) & $0.68\pm 0.05$& 12.14M & 24 & \underline{3.90e-4} &4.39e-4 & 5.55e-3
		\\
		$M=3$(cross) & $0.81\pm 0.04$& 12.17M & 24 & 4.26e-4 & \underline{3.82e-4} & \underline{4.68e-3}
		\\
		\bottomrule
	\end{tabular}
\end{table}

\begin{table}[t!]
	\center
	\caption{1D Burgers Ablation--temperature $T$}
	\label{tab: Appendix 1D Burgers Ablation--basis temperature}
	% \resizebox{\textwidth}{!}{
		\begin{tabular}{clccccccc}
			\toprule
			Problem & Best ABLE & $T=0.2$ & $0.4$ & $0.6$ &  $0.8$ & $1.0$ & Best $T$ & Best loss\\
			\midrule
			\multirow{2}{*}{$\nu=10^{-1}$} & $M=2$
			& 5.73e-4 & 5.42e-4 & 4.99e-4 & 4.86e-4 & 4.93e-4 & 0.8 & 4.86e-4\\
			& $M=2$ cross & 4.09e-4 & 4.09e-4& 3.90e-4 & 4.08e-4 & 4.17e-4
			& 0.6 & \underline{3.90e-4} \\
			\midrule
			\multirow{2}{*}{$\nu=10^{-2}$} &
			$M=3$ & 5.12e-4 & 4.76e-4 & 4.56e-4 & 4.72e-4 & 6.45e-4 & 0.5  & 4.50e-4
			\\
			& $M=3$ cross & 4.08e-4 & 3.82e-4 & 4.02e-4 & 3.94e-4 & 4.06e-4 & 0.4 & \underline{3.82e-4}
			\\
			\midrule
			\multirow{2}{*}{$\nu=10^{-3}$}
			& $M=8$ & 5.14e-3 & 5.15e-3 & 5.03e-3 & 5.07e-3 & 5.17e-3 & 0.6  & 5.03e-3
			\\
			& $M=3$ cross & 5.37e-3 & 4.89e-3 &  5.02e-3 & 4.91e-3 & 4.88e-3 & 1.1 & \underline{4.68e-3}
			\\
			\bottomrule
		\end{tabular}
	% }
\end{table}
\textbf{Darcy}. The main results together with the hyperparameter setting are already shown in the main paper Tab.~\ref{tab: 2D Darcy Flow}. We also conduct the ablation studies on basis size $M$ (see Tab.~\ref{tab: 2D Darcy Flow Ablation--basis size}) and temperature $T$(see Tab.~\ref{tab: 2D Darcy Flow Ablation--basis temperature}). The plots (see Fig.~\ref{fig:ablation}) are also contained in the main paper.
\begin{table}[t!]
	\center
	\caption{2D Darcy Flow Ablation--basis size $M$}
	\label{tab: 2D Darcy Flow Ablation--basis size}
	
	\begin{tabular}{ccccc}
		\toprule
		M & Training velocity(s/epoch) & Flops & h & Relative $l_2$-loss\\
		\midrule
		$M=1$(FNO) &$1.17\pm 0.03$ & 0.18G & - & 6.26e-3
		\\
		$M=2$ & $2.22\pm0.05$ & 0.25G & 24 & \underline{5.52e-3}
		\\
		$M=3$ & $2.83\pm0.02$ & 0.25G & 24 & 
		6.00e-3 \\
		$M=4$ & $3.53\pm0.02$ & 0.28G & 32 & 
		5.68e-3 \\
		$M=6$ & $4.83\pm0.02$ & 0.29G & 32 & 6.06e-3
		\\
		$M=8$ & $6.19\pm0.12$ & 0.30G & 32 & 5.81e-3
		\\
		$M=2$ cross & $2.19\pm0.13$ & 0.25G & 24 & 5.90e-3
		\\
		$M=3$ cross & $2.85\pm0.02$ & 0.25G & 24 & 5.85e-3
		\\
		\midrule
		U-ABLE $M=1$ (U-FNO) & $2.61\pm0.03$ & 6.11G & - & 4.45e-3
		\\
		U-ABLE $M=2$ cross & $3.81\pm 0.02$& 6.16G & 16 & 4.03e-3
		\\
		U-ABLE $M=3$ cross & $4.49\pm 0.02$ & 6.19G & 24 & \underline{3.94e-3}
		\\
		\bottomrule
	\end{tabular}
	
\end{table}

\begin{table}[t!]
	\center
	\caption{2D Darcy Flow Ablation--temperature}
	\label{tab: 2D Darcy Flow Ablation--basis temperature}
	% \resizebox{\textwidth}{!}{
		\begin{tabular}{ccccccccc}
			\toprule
			Temperature($T$) & $0.2$ & $0.4$ & $0.6$ &  $0.8$ & $1.0$ &$1.2 $& Best $T$ & Best loss\\
			\midrule
			ABLE($M=2$) & 0.00608 & 0.00578 & 0.00568 & \underline{0.00552} & 0.00556 & 0.00558 & 0.8 & 0.00552
			\\
			U-ABLE($M=3$ cross) & 0.00418 & 0.00398 & 0.00395 & \underline{0.00394} & 0.00418 & 0.00403 & 0.8 & 0.00394
			\\
			\bottomrule
		\end{tabular}
	% }
\end{table}

\textbf{Navier Stokes.}
The main results together with the hyperparameter settings are shown in the main paper (see Tab.~\ref{tab: 2D Navier Stokes}). The ABLE model in this case adopts the crossing setting with $M=3$ and the hidden-layer-width of the MLP $h=3\times\frac{C}{2}=30$ (backbone-channel width $C=20$). We also conduct the ablation study on basis size $M$ (Tab.~\ref{tab: 2D Navier Stokes Ablation--basis size}), and temperature $T$ (Tab.~\ref{tab: 2D Navier Stokes Ablation--basis temperature}), with plots shown in the main paper (see Fig.~\ref{fig:ablation}).
\begin{table}[t!]
	\center
	\caption{2D Navier Stokes Ablation--basis size $M$}
	\label{tab: 2D Navier Stokes Ablation--basis size}
	
	\begin{tabular}{ccccc}
		\toprule
		M & Training velocity(s/epoch) & Flops & h & Relative l2 loss \\
		\midrule
		$M=1$ (FNO) & $4.74\pm 0.19$ & 32.47M & - & 0.1237
		\\
		$M=2$ cross & $ 5.64\pm 0.27$ & 49.68M & 30 & 0.1091
		\\
		$M=3$ cross & $6.08\pm 0.19$ & 60.49M & 30 & 0.0985
		\\
		\bottomrule
	\end{tabular}
	
\end{table} 
\begin{table}[t!]
	\center
	\caption{2D Navier Stokes Ablation--temperature $T$}
	\label{tab: 2D Navier Stokes Ablation--basis temperature}
	% \resizebox{\textwidth}{!}{
		\begin{tabular}{cccccccccc}
			\toprule
			Best ABLE / T & $0.2$ & $0.4$ & $0.6$ &  $0.8$ & $1.0$ & $1.2$ & $1.4$ & Best $T$ & Best loss\\
			\midrule
			$M=3$ cross & 0.1083 & 0.1015  & 0.1008 & 0.0989  & 0.0994 & \underline{0.0985}  & 0.0993 & 1.2 & 0.0985 
			\\
			\bottomrule
		\end{tabular}
	% }
\end{table} 
Besides, we also found that ABLE performs well without the Gelu functions in the outermost two layers
(see Tab.~\ref{tab: 2D Navier Stokes Ablation--activation}), which suggests the contribution of the inherent non-linearity of adaptive basis learning (see theorem ~\ref{thm: inherent nonlinear}).
\begin{table}[t!]
	\center
	\caption{2D Navier Stokes Ablation--activation function}
	\label{tab: 2D Navier Stokes Ablation--activation}
	\begin{tabular}{ccc}
		\toprule
		Best ABLE & Index of layers with Gelu Activation & Relative l2 loss\\
		\midrule
		\multirow{4}{*}{$M=3$ cross} & 1,2,3 & 0.1057
		\\
		& 1,2 & 0.1029
		\\
		& 1 & \underline{0.0985}          
		\\
		& no Gelu & 0.1012
		\\
		\bottomrule
	\end{tabular}
\end{table}

\textbf{Statistical Testing for Error Reduction}. 
We conduct paired-t tests across six randomly selected seeds to evaluate whether the error reduction from FNO to ABLE is statistically significant on all $5$ benchmarks, with results reported in Tab.~\ref{tab: statistical significance}. All p-values are all below $0.001$, confirming that the improvement is statistically highly significant.
\begin{table}[t!]
	\center
	\caption{Statistical significance on error reduction from FNO to ABLE}
	\label{tab: statistical significance}
	% \resizebox{\textwidth}{!}{
		\begin{tabular}{cccccc}
			\toprule
			  Benchmark & Burgers(1e-1) & Burgers(1e-2) & Burgers(1e-3) &  Darcy & Navier Stokes \\
			\midrule
			FNO Relative $l_2$ Error($\pm$ std) & $6.74 \pm 0.24 $e-4 & $5.89 \pm 0.13 $e-4 & $6.45\pm 0.33$e-3 & $6.36\pm 0.10$e-3 & $0.1249\pm 0.0013 $\\
            ABLE Relative $l_2$ Error($\pm$ std) & $4.28 \pm 0.35 $e-4 & $3.87 \pm 0.21$e-4 &  $4.92\pm 0.24$e-3 & $5.81\pm 0.15$e-3 & $0.1009\pm 0.0013$\\
            Improvement(Error Reduction)  & $36.5\%$ & $34.3\%$ & $26.0\%$ & $8.69\%$ & $19.2\%$\\ 
            T-stat of difference & 14.97 & 24.19 & 7.45 & 9.13 & 31.39 \\
            P-value & <0.001 & <0.001 & <0.001 & <0.001 & <0.001 \\
			\bottomrule
		\end{tabular}
	% }
\end{table} 

Consequently, ABLE achieves significantly better performance than FNO. From a theoretical perspective, FNO corresponds to a trivial degenerate homogeneous case of ABLE with basis size $M=1$ and $p(x)\equiv 1$. The substantial error reduction thus provides direct evidence for the contribution of introducing adaptive basis learning (ABLE).
%
\iffalse
\textbf{More visualizations}.
We provide the visualization for other samples 
We provide the plot for the decaying curve of correlations $G(\textbf{r})=\langle p(\textbf{x},m)p(\textbf{x}+\textbf{r},m)\rangle$ as flows in Fig..
\fi

\section{Limitations and broader impact}
ABLE introduces adaptive basis learning for spectral neural operators while retaining FFT-compatible complexity. However, the current work focuses primarily on PDE benchmarks with regular discretizations and structured domains. Extending adaptive basis learning to irregular geometries, unstructured meshes, and large-scale 3D turbulent simulations remains an important direction for future work.

More broadly, ABLE suggests that representation learning itself may be a central bottleneck in neural operator design. Beyond PDE learning, adaptive spectral representations may provide useful inductive biases for other structured domains involving multiscale or non-stationary dynamics, including scientific computing, climate modeling, and spatiotemporal learning.

%%%%%%%%%%%%%%%%%%%%%%%%%%%%%%%%%%%%%%%%%%%%%%%%%%%%%%%%%%%%

\printcredits
\section*{Declaration of Competing Interest}
The authors declare that they have no known competing financial interests or personal relationships that could have appeared to influence the work reported in this paper. 
\section*{Acknowledgments}
XZ thanks Qiuzhen College, Tsinghua University, for supporting this research. AIAR gratefully acknowledges the support of the Yau Mathematical Sciences Center, Tsinghua University. This work is also supported by the Tsinghua University Dushi Program, and the  Tsinghua University Initiative Scientific Research Program. We thank Jiayi Li for helpful discussions.
\section*{Data Availability}
The dataset used in this work was originally introduced by \citet{DBLP:journals/corr/abs-2010-08895}. 

%\section*{References}
\bibliography{cas-refs.bib}

@article{DBLP:journals/corr/abs-2010-08895,
  author       = {Zongyi Li and
                  Nikola B. Kovachki and
                  Kamyar Azizzadenesheli and
                  Burigede Liu and
                  Kaushik Bhattacharya and
                  Andrew M. Stuart and
                  Anima Anandkumar},
  title        = {Fourier Neural Operator for Parametric Partial Differential Equations},
  journal      = {CoRR},
  volume       = {abs/2010.08895},
  year         = {2020},
  url          = {https://arxiv.org/abs/2010.08895},
  eprinttype   = {arXiv},
  eprint       = {2010.08895},
  bibsource    = {dblp computer science bibliography, https://dblp.org}
}

@article{Guibas:2021,
  author       = {John Guibas and
                  Morteza Mardani and
                  Zongyi Li and
                  Andrew Tao and
                  Anima Anandkumar and
                  Bryan Catanzaro},
  title        = {Adaptive Fourier Neural Operators: Efficient Token Mixers for Transformers},
  journal      = {CoRR},
  volume       = {abs/2111.13587},
  year         = {2021},
  url          = {https://arxiv.org/abs/2111.13587},
  eprinttype   = {arXiv},
  eprint       = {2111.13587},
  bibsource    = {dblp computer science bibliography, https://dblp.org}
}

@article{Wen2021UFNOA,
  title={U-FNO - an enhanced Fourier neural operator based-deep learning model for multiphase flow},
  author={Gege Wen and Zong-Yi Li and Kamyar Azizzadenesheli and Anima Anandkumar and Sally M. Benson},
  journal={ArXiv},
  year={2021},
  volume={abs/2109.03697},
  url={https://api.semanticscholar.org/CorpusID:237442093}
}

@inproceedings{Chen2025FreqMoEDF,
  title={FreqMoE: Dynamic Frequency Enhancement for Neural PDE Solvers},
  author={Tianyu Chen and Haoyi Zhou and Ying Li and Hao Wang and Zhenzhen Zhang and Tianchen Zhu and Shanghang Zhang and Jianxin Li},
  booktitle={International Joint Conference on Artificial Intelligence},
  year={2025},
  url={https://api.semanticscholar.org/CorpusID:278501699}
}

@article{Tripura2022WaveletNO,
  title={Wavelet neural operator: a neural operator for parametric partial differential equations},
  author={Tapas Tripura and Souvik Lal Chakraborty},
  journal={ArXiv},
  year={2022},
  volume={abs/2205.02191},
  url={https://api.semanticscholar.org/CorpusID:248512807}
}

@article{Fanaskov2022SpectralNO,
  title={Spectral Neural Operators},
  author={Vladimir Fanaskov and I. Oseledets},
  journal={Doklady Mathematics},
  year={2022},
  volume={108},
  pages={S226 - S232},
  url={https://api.semanticscholar.org/CorpusID:248986643}
}

@inproceedings{Cao2021ChooseAT,
  title={Choose a Transformer: Fourier or Galerkin},
  author={Shuhao Cao},
  booktitle={Neural Information Processing Systems},
  year={2021},
  url={https://api.semanticscholar.org/CorpusID:235253946}
}

@article{Wu2023SolvingHP,
  title={Solving High-Dimensional PDEs with Latent Spectral Models},
  author={Haixu Wu and Tengge Hu and Huakun Luo and Jianmin Wang and Mingsheng Long},
  journal={ArXiv},
  year={2023},
  volume={abs/2301.12664},
  url={https://api.semanticscholar.org/CorpusID:256389386}
}

@inproceedings{Yue2024HolisticPS,
  title={Holistic Physics Solver: Learning PDEs in a Unified Spectral-Physical Space},
  author={Xihang Yue and Linchao Zhu and Yi Yang},
  booktitle={International Conference on Machine Learning},
  year={2024},
  url={https://api.semanticscholar.org/CorpusID:273350739}
}

@article{QI2024106239,
title = {Gabor-Filtered Fourier Neural Operator for solving Partial Differential Equations},
journal = {Computers \& Fluids},
volume = {274},
pages = {106239},
year = {2024},
issn = {0045-7930},
doi = {https://doi.org/10.1016/j.compfluid.2024.106239},
url = {https://www.sciencedirect.com/science/article/pii/S0045793024000719},
author = {Kai Qi and Jian Sun}
}

@article{Li2022FourierNO,
  title={Fourier Neural Operator with Learned Deformations for PDEs on General Geometries},
  author={Zong-Yi Li and Daniel Zhengyu Huang and Burigede Liu and Anima Anandkumar},
  journal={J. Mach. Learn. Res.},
  year={2022},
  volume={24},
  pages={388:1-388:26},
  url={https://api.semanticscholar.org/CorpusID:250450893}
}

@article{Wang2025LaplacianEN,
  title={Laplacian eigenfunction-based neural operator for learning nonlinear reaction-diffusion dynamics},
  author={Jindong Wang and Wenrui Hao},
  journal={Journal of computational physics},
  year={2025},
  volume={543},
  url={https://api.semanticscholar.org/CorpusID:276249588}
}

@inproceedings{Vaswani2017AttentionIA,
  title={Attention is All you Need},
  author={Ashish Vaswani and Noam Shazeer and Niki Parmar and Jakob Uszkoreit and Llion Jones and Aidan N. Gomez and Lukasz Kaiser and Illia Polosukhin},
  booktitle={Neural Information Processing Systems},
  year={2017},
  url={https://api.semanticscholar.org/CorpusID:13756489}
}

@article{lu2022comprehensive,
  title   = {A comprehensive and fair comparison of two neural operators (with practical extensions) based on FAIR data},
  author  = {Lu, Lu and Meng, Xuhui and Cai, Shengze and Mao, Zhiping and Goswami, Somdatta and Zhang, Zhongqiang and Karniadakis, George Em},
  journal = {Computer Methods in Applied Mechanics and Engineering},
  volume  = {393},
  pages   = {114778},
  year    = {2022}
}

@article{DBLP:journals/corr/RonnebergerFB15,
  author       = {Olaf Ronneberger and
                  Philipp Fischer and
                  Thomas Brox},
  title        = {U-Net: Convolutional Networks for Biomedical Image Segmentation},
  journal      = {CoRR},
  volume       = {abs/1505.04597},
  year         = {2015},
  url          = {http://arxiv.org/abs/1505.04597},
  eprinttype   = {arXiv},
  eprint       = {1505.04597},
  bibsource    = {dblp computer science bibliography, https://dblp.org}
}

@book{evans2022partial,
  title={Partial differential equations},
  author={Evans, Lawrence C},
  volume={19},
  year={2022},
  publisher={American mathematical society}
}

@book{leveque2002finite,
  title={Finite volume methods for hyperbolic problems},
  author={LeVeque, Randall J},
  volume={31},
  year={2002},
  publisher={Cambridge university press}
}

@book{quarteroni1994numerical,
  title={Numerical approximation of partial differential equations},
  author={Quarteroni, Alfio and Valli, Alberto},
  year={1994},
  publisher={Springer}
}

@inproceedings{bryutkin2024hamlet,
  title={HAMLET: Graph Transformer Neural Operator for Partial Differential Equations},
  author={Bryutkin, Andrey and Huang, Jiahao and Deng, Zhongying and Yang, Guang and Sch{\"o}nlieb, Carola-Bibiane and Aviles-Rivero, Angelica I},
  booktitle={International Conference on Machine Learning},
  pages={4624--4641},
  year={2024},
  organization={PMLR}
}

@article{wang2025fourier,
  title={A Fourier neural operator approach for modelling exciton-polariton condensate systems},
  author={Wang, Yuan and Sathujoda, Surya T and Sawicki, Krzysztof and Gandhi, Kanishk and Aviles-Rivero, Angelica I and Lagoudakis, Pavlos G},
  journal={Communications Physics},
  year={2025},
  publisher={Nature Publishing Group UK London}
}

@article{lu2019deeponet,
  title={Deeponet: Learning nonlinear operators for identifying differential equations based on the universal approximation theorem of operators},
  author={Lu, Lu and Jin, Pengzhan and Karniadakis, George Em},
  journal={arXiv preprint arXiv:1910.03193},
  year={2019}
}

@article{wu2024transolver,
  title={Transolver: A fast transformer solver for pdes on general geometries},
  author={Wu, Haixu and Luo, Huakun and Wang, Haowen and Wang, Jianmin and Long, Mingsheng},
  journal={arXiv preprint arXiv:2402.02366},
  year={2024}
}

@article{cheng2025mamba,
  title={Mamba neural operator: Who wins? transformers vs. state-space models for pdes},
  author={Cheng, Chun-Wun and Huang, Jiahao and Zhang, Yi and Yang, Guang and Sch{\"o}nlieb, Carola-Bibiane and Aviles-Rivero, Angelica I},
  journal={Journal of Computational Physics},
  pages={114567},
  year={2025},
  publisher={Elsevier}
}

@article{pathak2022fourcastnet,
  title={Fourcastnet: A global data-driven high-resolution weather model using adaptive fourier neural operators},
  author={Pathak, Jaideep and Subramanian, Shashank and Harrington, Peter and Raja, Sanjeev and Chattopadhyay, Ashesh and Mardani, Morteza and Kurth, Thorsten and Hall, David and Li, Zongyi and Azizzadenesheli, Kamyar and others},
  journal={arXiv preprint arXiv:2202.11214},
  year={2022}
}

@inproceedings{Zhou2025SAOTAE,
  title={SAOT: An Enhanced Locality-Aware Spectral Transformer for Solving PDEs},
  author={Chenhong Zhou and Jie Chen and Zaifeng Yang},
  booktitle={AAAI Conference on Artificial Intelligence},
  year={2025},
  url={https://api.semanticscholar.org/CorpusID:283243811}
}

@article{kovachki2021universal,
  title={On universal approximation and error bounds for Fourier neural operators},
  author={Kovachki, Nikola and Lanthaler, Samuel and Mishra, Siddhartha},
  journal={Journal of Machine Learning Research},
  volume={22},
  number={290},
  pages={1--76},
  year={2021}
}

@article{Wojtaszczyk_2003, title={Projections and Non-Linear Approximation in the Space BV($\mathbb{R}^d$)}, volume={87}, DOI={10.1112/S0024611503014084}, number={2}, journal={Proceedings of the London Mathematical Society}, author={Wojtaszczyk, P.}, year={2003}, pages={471–497}}
\bibliographystyle{cas-model2-names} 
\end{document}